\documentclass[runningheads]{llncs}
\usepackage{hyperref}

\usepackage{graphicx}
\usepackage{xcolor}
\usepackage{amsmath}

\usepackage{amssymb}

\usepackage{paralist}
\usepackage[inline]{enumitem}
\usepackage{url}
\usepackage{verbatim}
\usepackage{siunitx}
\usepackage[capitalise]{cleveref}
\Crefname{figure}{Fig.}{Figs.}
\Crefname{section}{Sect.}{Sects.}
\Crefname{definition}{Def.}{Defs.}

\usepackage{wrapfig}
\crefname{equation}{}{}
\newcommand{\n}[1]{\num[round-mode=places,round-precision=4]{#1}}
\usepackage{multirow}
\usepackage{algorithm}
\usepackage[noend]{algpseudocode}

\makeatletter
\renewcommand{\ALG@beginalgorithmic}{\scriptsize}
\makeatother
\algrenewcommand\alglinenumber[1]{\scriptsize #1:}

\input{vertical_alg_lines}

\usepackage[acronym]{glossaries}
\newacronym[plural=DFA,firstplural=deterministic finite automata (DFA)]{DFA}{DFA}{deterministic finite automaton}
\newacronym[longplural=Markov decision processes]{MDP}{MDP}{Markov decision process}
\newacronym[firstplural=systems under test (SUTs)]{SUT}{SUT}{system under test}
\newacronym{DTMC}{DTMC}{discrete time Markov chain}
\newacronym{SMC}{SMC}{statistical model-checking}
\newacronym[firstplural=systems under learning (SULs)]{SUL}{SUL}{system under learning}
\newacronym{MAT}{MAT}{minimally adequate teacher}
\newacronym{PAC}{PAC}{probably approximately correct}
\newacronym{ioco}{ioco}{input ouput conformance}

\newcommand{\Dist}{\mathit{Dist}}
\newcommand{\Out}{\Sigma^\mathrm{O}}
\newcommand{\In}{\Sigma^\mathrm{I}}
\newcommand{\row}{\mathit{row}}
\newcommand{\lastOut}{\mathit{last}}
\newcommand{\complete}{\mathbf{cq}}
\newcommand{\diff}{\mathit{diff}}
\newcommand{\compatible}{\mathbf{compatible}}
\newcommand{\equivrow}{\mathbf{eqRow}}
\newcommand{\cg}{\mathit{cg}}
\newcommand{\SUL}{\gls*{SUL}}
\newcommand{\Ta}{\widehat{T}}
\newcommand{\LstarMdp}{${L^*_\textsc{mdp}}$}
\newcommand{\LstarMdpE}{$L^*_{\textsc{mdp}^e}$}
\newcommand{\nbatch}{n_\mathrm{resample}}
\newcommand{\ntest}{n_\mathrm{test}}
\newcommand{\nretest}{n_\mathrm{retest}}

\newcommand{\randSel}{\mathit{randSel}}
\newcommand{\coinFlip}{\mathit{coinFlip}}
\newcommand{\prand}{p_\mathrm{rand}}
\newcommand{\pstop}{p_\mathrm{stop}}
\newcommand{\rmin}{r_\mathrm{min}}
\newcommand{\rmax}{r_\mathrm{max}}
\newcommand{\ncomplete}{n_c}
\newcommand{\samples}{\mathcal{S}}
\newcommand{\true}{\mathbf{true}}
\newcommand{\false}{\mathbf{false}}
\newcommand{\semdom}{\mathit{dd}}

\newcommand{\prefixes}{\mathit{prefixes}}

\newcommand{\Traces}{\mathcal{TR}}
\newcommand{\TestSeq}{\mathcal{TS}}
\newcommand{\ContSeq}{\mathcal{CS}}

\newcommand{\fq}{\mathbf{fq}}
\newcommand{\eq}{\mathbf{eq}}
\newcommand{\rsq}{\mathbf{rfq}}
\newcommand{\cq}{\complete}
\newcommand{\odq}{\mathbf{odq}}
\newcommand{\runamb}{r_\mathrm{unamb}}
\newcommand{\tunamb}{t_\mathrm{unamb}}
\newcommand{\LongTr}{Lt}
\newcommand{\can}{\mathit{can}}
\newcommand{\repr}[1]{\mathit{rep}(#1)}

\usepackage{environ}
\usepackage{ifthen}
\newboolean{full}           % Flag: Conference (false) vs. Techn. Rep. (true)
\setboolean{full}{true}

\ifthenelse{\boolean{full}}%
  {%
  \NewEnviron{full}
    {%
    \BODY{}
        }%
}%
    {\NewEnviron{full}
    {%
    }%
}%

\ifthenelse{\boolean{full}}%
  {%
  \NewEnviron{conference}
    {%
        }%
}%
    {\NewEnviron{conference}
    {%
    \BODY{}
    }%
}%
\ifthenelse{\boolean{full}}{\usepackage{todonotes}}{\usepackage[disable]{todonotes}}

\usepackage{tikz}
\usetikzlibrary{automata}

\begin{document}
\begin{full}
\title{$L^*$-Based Learning of Markov Decision Processes
\\ (Extended Version)\thanks{This work is an extended version of the conference paper ``$L^*$-Based 
Learning of Markov Decision Processes'' accepted for presentation at {FM} 2019, the 
23\textsuperscript{rd} International Symposium on Formal Methods in Porto, Portugal.  
}%\footnote{This is an extended version of...}
}
\end{full}
\begin{conference} 
\title{$L^*$-Based Learning of Markov Decision Processes
}
\end{conference}

\author{Martin Tappler\inst{1} \and Bernhard K. Aichernig\inst{1} \and Giovanni Bacci\inst{3} \and
Maria Eichlseder\inst{2} \and Kim G. Larsen\inst{3} 
}
\authorrunning{M.\ Tappler, B.\ K.\ Aichernig, G.\ Bacci, M.\ Eichlseder, and K.\ G.\ Larsen}
\institute{Institute of Software Technology, Graz University of Technology, Graz, Austria 
\email{\{aichernig,martin.tappler\}@ist.tugraz.at}
\and
Institute of Applied Information Processing and Communications, \\ Graz University of Technology, Graz, Austria \\
\email{maria.eichlseder@iaik.tugraz.at}
\and
Dept.\ of Computer Science, Aalborg University, Denmark
\email{{\{giovbacci,kgl\}@cs.aau.dk}}}
\maketitle              % typeset the header of the contribution
\setcounter{footnote}{0}
\begin{abstract}
Automata learning techniques automatically generate system models from test observations. 
These techniques usually fall into two categories: passive and active. Passive learning uses 
a predetermined data set, e.g., system logs. In contrast, active learning actively queries the 
system under learning, which is considered more efficient.
%The ability to influence test generation via queries aims at reducing the number of observations required 
%to construct an accurate model.

An influential active learning technique is Angluin's $L^*$ algorithm for regular languages which 
inspired several generalisations from DFAs to other automata-based modelling formalisms.
In this work, we study $L^*$-based learning of deterministic Markov decision 
processes, first assuming an ideal setting with perfect information. Then, we relax this assumption
and present a novel learning algorithm that collects information by sampling system traces via testing. 
Experiments with the implementation of our sampling-based algorithm suggest that it achieves better accuracy than state-of-the-art passive learning techniques 
with the same amount of test data. 
Unlike existing learning algorithms with predefined states,
our algorithm learns the complete model structure including the states.
\end{abstract}


\setlength{\intextsep}{0pt}%
\section{Introduction}

Automata learning automatically generates models from system observations such as test logs. Hence, 
it enables model-based verification for black-box software systems~\cite{DBLP:conf/dagstuhl/HowarS16,DBLP:conf/dagstuhl/AichernigMMTT16}, e.g. via model checking. Automata learning techniques generally fall into two categories: passive and active learning. Passive algorithms take a given sample of system traces as input and generate models consistent with the sample. The quality and comprehensiveness of learned models therefore largely depend on the given sample. In contrast, active algorithms %on the other hand 
actively query the \gls*{SUL} to sample system traces. This enables to steer the trace generation towards parts of the \SUL{}'s state space that have not been thoroughly covered, potentially finding yet unknown aspects of the \SUL{}.   

Many active automata learning algorithms are based on Angluin's $L^*$ algorithm~\cite{DBLP:journals/iandc/Angluin87}.
It was originally proposed for learning \glspl*{DFA} accepting regular languages and later applied to learn models of reactive systems, by considering system traces to form regular languages~\cite{DBLP:conf/cav/HungarNS03}. $L^*$ has been extended to formalisms better suited for modelling reactive systems such as Mealy machines~\cite{Margaria2004,DBLP:conf/fm/ShahbazG09} and extended finite state-machines~\cite{DBLP:journals/fac/CasselHJS16}. Most $L^*$-based work, however, targets deterministic models, with the exceptions of algorithms for non-deterministic Mealy machines~\cite{DBLP:conf/icgi/KhaliliT14} and non-deterministic input-output transition systems~\cite{DBLP:journals/eceasst/VolpatoT15}. Both techniques are based on testing, but abstract away the observed frequency of events, thus they do not use all available information. 

Here, we present an $L^*$-based approach for learning models of stochastic systems with transitions 
that happen with some probability depending on non-deterministically chosen inputs. More concretely, 
we learn deterministic \glspl*{MDP}, like 
\textsc{IoAlergia}~\cite{DBLP:journals/corr/abs-1212-3873,DBLP:journals/ml/MaoCJNLN16}, a 
state-of-the-art passive learning algorithm. Such models are 
commonly used to model randomised distributed algorithms~\cite{DBLP:books/daglib/0020348}, e.g. in 
protocol verification~\cite{DBLP:journals/sigmetrics/KwiatkowskaNP08a,DBLP:journals/jcs/NormanS06}. 
We present two learning algorithms: the first takes an ideal view assuming perfect knowledge about 
the exact distribution of system traces. The second algorithm relaxes this assumption, by sampling 
system traces to estimate their distribution. We refer to the former as \emph{exact learning 
algorithm} \LstarMdpE{} and to the latter as \emph{sampling-based learning algorithm} \LstarMdp{}. 
We implemented \LstarMdp{} and evaluated it by comparing it to 
\textsc{IoAlergia}~\cite{DBLP:journals/corr/abs-1212-3873,DBLP:journals/ml/MaoCJNLN16}.
Experiments showed favourable 
performance of \LstarMdp{}, i.e. it produced more accurate models than \textsc{IoAlergia} given 
approximately the same amount of data.
\begin{full}
Apart from the empirical evaluation, we show that the model learned by \LstarMdp{} converges in the 
limit to an \gls*{MDP} isomorphic to the canonical \gls*{MDP} representing the \SUL{}. To the best 
of our knowledge, \LstarMdp{} is the first $L^*$-based learning algorithm for \glspl*{MDP} that can 
be implemented via testing. Our contributions span the algorithmic development of learning 
algorithms, their analysis with respect to convergence and the implementation as well as 
the evaluation of learning algorithms. 

This work is an extended version of the conference paper 
``$L^*$-Based Learning of Markov Decision Processes'' accepted for presentation at {FM} 2019, the 
23\textsuperscript{rd} 
International Symposium on Formal Methods in Porto, Portugal. It provides additional 
details on the implementation of \LstarMdp{}, the convergence analysis of both learning algorithms
and an extended evaluation. 
\end{full}
\begin{conference}
 Generally, models learned by \LstarMdp{} converge in the limit to an \gls*{MDP} observationally equivalent to the \SUL{}. To the best of our knowledge, \LstarMdp{} is the first $L^*$-based learning algorithm for \glspl*{MDP} that can be implemented via testing. 
 Our contributions span the algorithmic development of learning algorithms, the implementation and the evaluation of learning algorithms. The full technical report on \LstarMdp{}~\cite{lstar_mdp_tech_report} additionally includes convergence proofs, further experiments and implementation details. 
\end{conference}

The rest of this paper is structured as follows. We introduce notational conventions, preliminaries on \glspl*{MDP} and active automata learning in \cref{sec:prelim}. 
\begin{full} \cref{sec:method_exact} provides a characterisation of \glspl*{MDP} and presents the exact learning algorithm \LstarMdpE{}. \cref{sec:method_non_exact} describes the sampling-based \LstarMdp{} and analyses it with respect to convergence. \end{full}
\begin{conference}
 Section~\ref{sec:method_exact} discusses semantics of \glspl*{MDP} and presents the exact learning algorithm \LstarMdpE{}. Section~\ref{sec:method_non_exact} describes the sampling-based \LstarMdp{}.
\end{conference}
Section~\ref{sec:eval} discusses the evaluation and in \cref{sec:rel_work}, we discuss related work. 
 We provide a summary and concluding remarks in \cref{sec:concl}.

\section{Preliminaries}
\label{sec:prelim}
% \begin{full}
% We introduce background information on \glspl*{MDP} following~\cite{DBLP:conf/sfm/ForejtKNP11,DBLP:books/daglib/0020348,DBLP:conf/rv/AichernigT17}.
% \end{full}
\paragraph{Notation \& Auxiliary Definitions.}
Let $S$ be a set. We denote the concatenation of two sequences $s$ and $s'$ in $S^*$ by $s \cdot s'$, the length of a sequence $s$ by $|s|$ and the empty sequence by $\epsilon$. We implicitly lift elements in $S$ to sequences of length one.  Sequence 
$s$ is a prefix of $s'$ if there exists an $s''$ such that $s\cdot s'' = s'$, denoted by $s \ll s'$. The pairwise concatenation of sets of sequences $A,B \subseteq S^*$ is $A \cdot B = \{a \cdot b \mathrel\vert a \in A, b\in B\}$. A set of sequences $A \subseteq S^*$ is prefix-closed, iff for every $a \in A$, $A$ also contains all prefixes of $A$. Suffixes and suffix-closedness are defined analogously. For a sequence $s$ 
%= s_1 \cdot s_2 \cdots s_n$ 
in $S^*$, $s[i]$ is the element at index $i$, with indexes starting at $1$, $s[\ll i]$ is the prefix of $s$ with length $i$ and $\prefixes(s) = \{s'\mathrel| s' \in S^*:s' \ll s\}$ is the set of all prefixes of $s$. Given a multiset $\samples$, we denote the multiplicity of $x$ in $\samples$ by $\samples(x)$. %and we denote the support of $\samples$ by $\mathrm{supp}(\samples) = \{x \in \samples | \samples(x) > 0\}$.
% \subsection{Auxiliary Definitions and Functions}
$\Dist(S)$ denotes the set of probability distributions over $S$, i.e. for all $\mu: S \rightarrow [0,1]$ in
$\Dist(S)$ we have $\sum_{s\in S} \mu(s) = 1$. In the remainder of this paper, distributions $\mu$ may be partial functions, in which case
we implicitly set $\mu(e) = 0$ if $\mu$ is not defined for $e$. For $A\subseteq S$, $\mathbf{1}_{A}$ denotes the indicator function of $A$, i.e. $\mathbf{1}_{A}(e) = 1$ if $e \in A$ and $\mathbf{1}_{A}(e) = 0$ otherwise. 
\begin{full}
Hence, $\mathbf{1}_{\{e\}}$ for $e \in S$ is the probability distribution assigning probability $1$ to $e$.
In \cref{sec:method_non_exact}, we apply a pseudo-random function $\randSel$ taking 
taking a set $S$ as input and returning a single element of 
the set, whereby the element is chosen according to a uniform distribution, i.e. $\forall e \in S : \mathbb{P}(\randSel(S) = e) = \frac{1}{|S|}$.
In addition to that, we use the function $\coinFlip(p)$ returning $\true$ with probability $p$ 
and $\false$ otherwise. 
\end{full}

\subsubsection{Markov Decision Processes.}% and as learned by \textsc{Alergia}~\cite{DBLP:journals/ml/MaoCJNLN16}.
\begin{definition}[\acrfull*{MDP}] 
 A labelled \acrfull*{MDP} is a tuple $\mathcal{M} = \langle Q,\In, \Out,q_0, \delta, L\rangle$ where
 \begin{conference}
    $Q$ is a finite non-empty set of states, $\In$ and $\Out$ are finite sets of inputs and outputs,
  $q_0 \in Q$ is the initial state, 
  $\delta : Q \times \In \rightarrow \Dist(Q)$ is the probabilistic transition function, and
  $L : Q \rightarrow \Out$ is the labelling function.
 \end{conference}
 \begin{full}
 \begin{compactitem}
  \item $Q$ is a finite non-empty set of states,
  \item $\In$ and $\Out$ are finite sets of input and output symbols respectively,
  \item $q_0 \in Q$ is the initial state, 
  \item $\delta : Q \times \In \rightarrow \Dist(Q)$ is the probabilistic transition function, and
  \item $L : Q \rightarrow \Out$ is the labelling function.
 \end{compactitem}
 \end{full}
%The transition function $\delta$ must be defined for all $q \in Q$ and $i \in \Sigma^\mathrm{in}$. 
An \gls*{MDP} is \emph{deterministic} if $\forall q \in Q, \forall i : \delta(q,i)(q') > 0 \land \delta(q,i)(q'') > 0\rightarrow q' = q'' \lor L(q') \neq L(q'')$. 
%Non-determinism thus results only from the non-deterministic choice of inputs by the environment.
\end{definition}

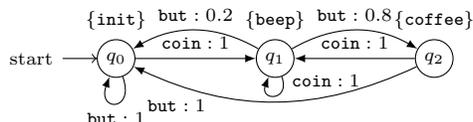
\begin{wrapfigure}[6]{t}{.54\textwidth}
  \begin{tikzpicture}[font=\scriptsize, node distance = 1.6cm]
   \node[initial,state,inner sep=2pt,minimum size = 0.3cm](q_0){$q_0$};
   \node[state,right=of q_0,inner sep=2pt,minimum size = 0.3cm](q_1){$q_1$};
   \node[state, right = of q_1,inner sep=2pt,minimum size = 0.3cm](q_2){$q_2$};
   
   \path[-latex] (q_0)  edge [loop below,min distance=5mm,in=250,out=290] node[] {$\mathtt{but} : 1$} (q_0)
%    edge [bend right = 20] node[below] {$\mathtt{but} : 0.1$} (q_1)
   edge [] node[above] {$\mathtt{coin} : 1$} (q_1)
   (q_1)  edge [loop below,min distance=3.5mm,in=250,out=290] node[below right= -0.4cm and 0.1cm] {$\mathtt{coin} : 1$} (q_1)
   edge [bend right = 30] node[above] {$\mathtt{but} : 0.2$} (q_0)
   edge [bend left = 30] node[above] {$\mathtt{but} : 0.8$} (q_2)
   (q_2) edge node[above] {$\mathtt{coin} : 1$} (q_1)
       edge [bend left = 25] node[below left  = -0.15cm and 0.8cm] {$\mathtt{but} : 1$} (q_0)
   ;
   \node[above = 0cm of q_0]{$\{\mathtt{init}\}$};
   \node[above = 0cm of q_1]{$\{\mathtt{beep}\}$};
   \node[above = 0cm of q_2]{$\{\mathtt{coffee}\}$};
  \end{tikzpicture}
\vspace{-0.5cm}
\caption{\gls*{MDP} model of a faulty coffee machine}
\label{fig:coffee_mdp}
\end{wrapfigure}
We learn deterministic labelled \glspl*{MDP} as learned by passive learning techniques like \textsc{IoAlergia}~\cite{DBLP:journals/ml/MaoCJNLN16}.
Such \glspl*{MDP} define at most one successor state for each source state and input-output pair. 
In the following, we refer to these models uniformly as \glspl*{MDP}. 
We use $\Delta : Q \times \In \times \Out \rightarrow Q\cup\{\bot\}$ to compute successor states. The function is defined
by $\Delta(q,i,o) = q' \in Q$ with $L(q') = o$ and $\delta(q,i)(q') > 0$ if there exists such a $q'$, 
otherwise $\Delta$ returns $\bot$. \cref{fig:coffee_mdp} shows an \gls*{MDP} model of a faulty 
coffee machine~\cite{Aichernig2019}. 
Outputs in curly braces label states and inputs with corresponding probabilities label edges. 
After providing the inputs $\mathtt{coin}$ and $\mathtt{but}$, the coffee machine \gls*{MDP} produces the output $\mathtt{coffee}$
with probability $0.8$, but with probability $0.2$, it resets itself, producing the output $\mathtt{init}$.

% {\color{blue}
% An MDP $\mathcal{M}$ that is in state $q_k \in Q$ at time $k \in \mathbb{N}$, after receiving the input $i \in \In$, it moves at time $k+1$ to the next state $q_{k+1}$ with probability $\delta(q_k,i)(q_{k+1})$. At time $0$, $\mathcal{M}$ is assumed to be in the initial state $q_0$ (with probability $1$). In this sense, a finite prefix of an execution of $\mathcal{M}$ is described by an alternating sequence of states and inputs of the form $\pi = q_0 i_1 q_1 \cdots i_n q_n$. From an execution path $\pi$ we can observe a trace $t = o_0 i_1 o_1 \cdots i_n o_n$ with $o_k = L(q_k)$ describing the sequence of input-output symbols drawn during the execution; we will denote by $L(\pi)$ the trace corresponding to the path $\pi$. We denote by $\mathit{Path}_\mathcal{M}$ and $\mathit{Trace}_\mathcal{M}$ the set of execution paths and traces of $\mathcal{M}$ respectively.}

\emph{Execution.}
A path $\rho$ through an \gls*{MDP} is an alternating sequence of states and inputs starting in the initial state $q_0$, i.e.
$\rho = q_0 \cdot i_1 \cdot q_1\cdot  i_2 \cdot q_2\cdots i_{n-1} \cdot q_{n-1} \cdot i_n \cdot q_n$. %\in Q \times (\Sigma^\mathrm{in} \times Q)^* i_{n-1} q_{n-1} i_n q_n 
In each state $q_k$, the next input $i_{k+1}$ is chosen non-deterministically and based on that, the next state $q_{k+1}$ is chosen probabilistically according to $\delta(q_k,i_{k+1})$. 
% A finite path $\rho$ through an \gls*{MDP} is a finite prefix of an infinite path ending in some state $q_n$, i.e. it is of the form $\rho= q_0 \cdot i_1 \cdot q_1 \cdots i_{n-1} q_{n-1} i_n q_n$. 
\begin{full}
We denote set of all paths of an \gls*{MDP} $\mathcal{M}$ by $Path_\mathcal{M}$.
The execution of an \gls*{MDP} is controlled by a so-called scheduler, resolving the non-deterministic choice of inputs. A scheduler as defined below specifies a distribution over the next input given the current execution path.

\begin{definition}[Scheduler]
 Given an \gls*{MDP} $\mathcal{M} = \langle Q,\In, \Out,q_0, \delta, L\rangle$, a scheduler
 for $\mathcal{M}$ is a function $s : Path_\mathcal{M} \rightarrow Dist(\Sigma^\mathrm{in})$.
\end{definition}

The composition of an \gls*{MDP} $\mathcal{M}$ and a scheduler $s$ induces a deterministic Markov chain, i.e. a fully probabilistic system allowing to define a probability measure over paths. Additionally to $\mathcal{M}$ and $s$, we also need a probability distribution $p_l \in Dist(\mathbb{N}_0)$
over the path lengths.\footnote{Usually in verification, such measures are defined with respect to 
infinite paths. Since our technique is test-based and every test has only finite length, we consider 
finite paths analogously to \cite{Aichernig2019}.} An \gls*{MDP} $\mathcal{M}$, a scheduler $s$, and 
a path length probability distribution $p_l$ induce a probability distribution 
$\mathbb{P}_{\mathcal{M},s}^l$ on finite paths $Path_\mathcal{M}$, defined by: 

\begin{equation}
\label{eq:finite_path_prob}
\mathbb{P}^l_{\mathcal{M},s}(q_0 i_1 q_1 \cdots i_n q_n) = p_l(n) \cdot \left(\prod_{j = 1}^n s(q_0 \cdots i_{j-1} q_{j-1})(i_j) \cdot \delta(q_{j-1},i_j)(q_j)\right)
\end{equation}
\end{full}
\begin{conference}
 The execution of an \gls*{MDP} is controlled by a so-called scheduler, resolving the non-deterministic choice of inputs by specifying a distribution over the next input given the current execution path. The composition of an \gls*{MDP} and a scheduler induces a Markov chain with a corresponding probability measure, see e.g. \cite{DBLP:conf/sfm/ForejtKNP11}.
\end{conference}

\emph{Sequences of Observations.}
During the execution of a finite path $\rho$, we observe a trace $L(\rho) = t$, i.e. an alternating sequence of inputs and outputs starting with an output, with $t = o_0 i_1 o_1 \cdots i_{n-1} o_{n-1} i_n o_n$ and $L(q_i) = o_i$. 
\begin{full} 
Since we consider deterministic \glspl*{MDP}, $L$ is invertible, thus each trace in $\Out \times (\In\times \Out)^*$ corresponds to at most one path and $\mathbb{P}_{\mathcal{M},s}^{l}$ can be adapted to traces $t$ by defining:
\begin{equation*}
\mathbb{P}_{\mathcal{M},s}^{l}(t) = \begin{cases}
            \mathbb{P}_{\mathcal{M},s}^{l}(\rho) &\ldots \text{ if there is a $\rho$ with } L(\rho) = t \\
                                     0 &\ldots \text{ otherwise }
                                    \end{cases} 
\end{equation*}
We say that a trace $t$ is \emph{observable} if there exists a $\rho$ with $L(\rho) = t$, thus there is a scheduler $s$ and a $p_l$ such that $\mathbb{P}_{\mathcal{M},s}^{l}(t) > 0$. 
\end{full}
\begin{conference}
 Since we consider deterministic \glspl*{MDP}, $L$ is invertible, thus each trace in $\Out \times (\In\times \Out)^*$ corresponds to at most one path.
 We say that a trace $t$ is \emph{observable} if there exists a $\rho$ with $L(\rho) = t$.
\end{conference}
 In a deterministic \gls*{MDP} $\mathcal{M}$, each observable trace $t$ uniquely defines a state of $\mathcal{M}$ reached by executing $t$ from the initial state $q_0$. We compute this state by $\delta^*(t) = \delta^*(q_0,t)$ defined by $\delta^*(q,L(q)) = q$ and
\begin{equation*}
%  &\delta^*(q,o) = q \text{ if } L(q) = o  \\
%  &\delta^*(q, o_0 i_1 o_1 \cdots i_{n-1} o_{n-1} i_n o_n) = \bot \text{ if } L(q) \neq o_0\\
%  &\delta^*(q, o_0 i_1 o_1 \cdots i_{n-1} o_{n-1} i_n o_n) = \bot \text{ if } \delta^*(q, o_0 i_1 o_1 \cdots i_{n-1} o_{n-1}) = \bot \\
% \delta^*(q,L(q)) &= q \\
 \delta^*(q, o_0 i_1 o_1 \cdots i_{n-1} o_{n-1} i_n o_n) = \Delta(\delta^*(q, o_0 i_1 o_1 \cdots i_{n-1} o_{n-1}),i_n,o_n). 
\end{equation*}
%  \end{full}
%  \begin{conference}
%   $\delta^*(q,L(q)) = q$ and $\delta^*(q, o_0 i_1 o_1 \cdots i_{n-1} o_{n-1} i_n o_n) = \Delta(\delta^*(q, o_0 i_1 o_1 \cdots i_{n-1} o_{n-1}),i_n,o_n)$.
%  \end{conference}

If $t$ is not observable, then there is no path $\rho$ with $t = L(\rho)$, denoted by $\delta^*(t) = \bot$. We denote the last output $o_n$ of a trace $t = o_0 \cdots i_n o_n$,  by $\lastOut(t)$. 

We use three types of observation sequences with short-hand notations:
\begin{compactitem}
 \item \textsl{Traces:} abbreviated by $\Traces = \Out \times(\In \times \Out)^*$
 \item \textsl{Test sequences:} abbreviated by $\TestSeq = (\Out \times \In)^*$
 \item \textsl{Continuation sequences:} abbreviated by $\ContSeq = \In \times \TestSeq$ % (\Out \times \In)^*$
\end{compactitem}
These sequence types alternate between inputs and outputs, thus they are related among each other.
In slight abuse of notation, we use $A \times B$ and $A \cdot B$ interchangeably for the remainder of this paper. Furthermore, we extend the sequence notations and the notion of prefixes to $\Out$, $\In$, $\Traces$, $\TestSeq$ and $\ContSeq$, e.g., test sequences and traces are related by $\Traces = \TestSeq \cdot \Out$.
% Note that $s \in (\In \times \Out)^*$ and we apply concatenation via $\cdot$ such that e.g. $\Traces = \Traces \cdot (\In \times \Out)^*$.
% Similarly, the sets $\Traces$, $\TestSeq$ and $\ContSeq$ are related by concatenation.

As noted, a trace in $\Traces$ leads to a unique state of an \gls*{MDP} $\mathcal{M}$. A test sequence in $s \in \TestSeq$ of length $n+1$ consists of a trace in $t \in \Traces$ with $n$ outputs and an input $i \in \In$ with $s = t \cdot i$; 
thus executing test sequence $s = t \cdot i$ puts $\mathcal{M}$ into the state reached by $t$ and tests $\mathcal{M}$'s reaction to $i$.  Extending the notion of observability, we say that the test sequence $s$ is observable if $t$ is observable. A continuation sequence $c\in \ContSeq$ begins and ends with an input, i.e. concatenating a trace $t \in \Traces$ and $c$ creates a test sequence $t \cdot c$ in $\TestSeq$. Informally, continuation sequences test $\mathcal{M}$'s reaction in response to multiple consecutive inputs. 

\begin{full}
\begin{lemma}\label{lem:non_observable_prefix}
 If trace $t \in \Traces$ is not observable, then any $t'\in \Traces$ such that $t \ll t'$ is not observable as well.
\end{lemma}

\cref{lem:non_observable_prefix} follows directly from \cref{eq:finite_path_prob}. For a non-observable $t$, we have $\forall s,p_l: \mathbb{P}^l_{\mathcal{M},s}(t) = 0$ and extending $t$ to create $t'$ only adds further factors. The same property holds for test sequences. 
\end{full}

\subsubsection{Active Automata Learning.}
We consider active automata learning in the \gls*{MAT} framework~\cite{DBLP:journals/iandc/Angluin87}, introduced by Angluin for the $L^*$ algorithm.
It assumes the existence of a \gls*{MAT}, which is able to answer queries. $L^*$ 
learns a \gls*{DFA} representing an unknown regular language $L$ over some alphabet $A$ 
and therefore requires two types of queries: \emph{membership} and \emph{equivalence} queries. 
First, $L^*$ repeatedly selects strings in $A^*$
and checks if they are in $L$ via \emph{membership} queries.
Once the algorithm has gained sufficient information, it forms a hypothesis  %, a conjectured
\gls*{DFA} consistent with the membership query results. It then poses an \emph{equivalence} query checking
for equivalence between $L$ and the language accepted by the hypothesis. % \gls*{DFA}. 
The teacher responds either with \emph{yes} signalling equivalence;
or with a counterexample to equivalence, i.e. a string in the symmetric difference between 
$L$ and the language accepted by the hypothesis. 
After processing a counterexample, $L^*$ starts a new round of learning, consisting of membership queries and a concluding equivalence query. 
Once an equivalence query returns \emph{yes}, learning stops with the final hypothesis as output. 

$L^*$ has been extended to learn models of reactive systems such as Mealy machines~\cite{DBLP:conf/fm/ShahbazG09}. 
In practice, queries for learning models of black-box systems are usually implemented via 
testing~\cite{Aichernig2018}. 
Therefore, equivalence queries are generally only approximated as complete testing for black-box systems is impossible unless there is an upper bound on the number of system states.
\begin{full}
We cover the ideal setting in \cref{sec:method_exact} by presenting an $L^*$-based exact learning algorithm for \glspl*{MDP}.
In \cref{sec:method_non_exact}, we discuss an implementation in a sampling-based setting that approximates queries by testing the \acrshort*{SUL}.
\end{full}

\section{Exact Learning of \glspl*{MDP}}
\label{sec:method_exact}
% In the previous section, we viewed an \gls*{MDP} as a generator of some language.
% Now, we move towards practice, viewing an \gls*{MDP} as a reactive systems producing 
% outputs in response to inputs. For the initial discussion, we again place
% strong assumptions on the capabilities of the teacher.
This section presents \LstarMdpE{}, an exact active learning algorithm for \glspl*{MDP},
the basis for the sampling-based algorithm presented in \cref{sec:method_non_exact}. 
In contrast to sampling, \LstarMdpE{} assumes the existence of a teacher with perfect knowledge about the \SUL{} that is able to answer two types 
of queries: \emph{output distribution} queries and \emph{equivalence} queries.
The former asks for the exact distribution of outputs following a test sequence in the \SUL.
The latter takes a hypothesis \gls*{MDP} as input and responds either with \emph{yes} iff 
the hypothesis is observationally equivalent to the \SUL{} or with a counterexample to equivalence. A counterexample 
is a test sequence leading to different output distributions in hypothesis and \SUL{}. 
First, we describe how we capture the semantics of \glspl*{MDP}.

% \begin{full}
% \subsection{Semantics of \glspl*{MDP}}
% \end{full}
% \begin{conference}
\subsubsection{Semantics of \glspl*{MDP}.}
% \end{conference}

We can interpret an \gls*{MDP} as a function $M : \TestSeq \rightarrow \Dist(\Out) \cup \{\bot\}$, mapping 
test sequences $s$ to output distributions or undefined behaviour for non-observable $s$. 
% This 
% follows Steffen et al.~\cite{DBLP:conf/sfm/SteffenHM11}, who interpret Mealy machines as functions from input 
% sequences to outputs. 
This follows the interpretation of Mealy machines as functions from input 
sequences to outputs~\cite{DBLP:conf/sfm/SteffenHM11}. 
\begin{conference}
Viewing \glspl*{MDP} as reactive systems, we consider two \glspl*{MDP} to be equivalent,
if their semantics are equal, i.e. we make the same observations on both. 
\end{conference}
\begin{full}
Likewise, we will define which functions $M$ capture the semantics of \glspl*{MDP}
by adapting the Myhill-Nerode theorem for regular languages~\cite{zbMATH03146331}. 
We denote the set of sequences $s$ where $M(s) \neq \bot$ as defined domain $\semdom(M)$ of $M$.
\end{full}
\begin{definition}[\gls*{MDP} Semantics]
 \label{def:mdp_semantics}
Given an \gls*{MDP} $\langle Q,\In, \Out,q_0, \delta, L\rangle$, its semantics is a 
function $M$, defined
for $i \in \In$, $o \in \Out$, $t \in \Traces$ as follows:
\begin{align*}
% M(s) = \begin{cases}
%                         \mathbf{1}_{\{L(q_0)\}} &\text{ if } s = \epsilon \\
%                         \bot &\text{ if } \exists i \in \In: s = t \cdot i \land \delta^*(t) = \bot \\
%                         \mu \in \Dist(\Out) &\text{ else, where } \begin{array}{l}  
% \forall o \in \Out: \mu(o) = p \text{ with } s = t \cdot i \text{ for $i \in \In$ and }\\
% \delta(\delta^*(t),i)(q) = p > 0 \land L(q) = o
%                                                                       \end{array}
%                \end{cases}
M(\epsilon)(L(q_0)) &= 1 \\
M(t \cdot i) &= \bot \text{ if } \delta^*(t) = \bot \\
M(t \cdot i)(o) &= p \text{ otherwise if } \delta(\delta^*(t),i)(q) = p > 0 \land L(q) = o 
\end{align*}

\begin{conference}
\glspl*{MDP} $\mathcal{M}_1$ and $\mathcal{M}_2$ with
 semantics $M_1$ and $M_2$ are output-distribution equivalent, 
 denoted $\mathcal{M}_1 \equiv_\mathrm{od} \mathcal{M}_2$, iff $M_1 = M_2$.
\end{conference}
\end{definition}

\begin{definition}[$M$-Equivalence of Traces]
\label{def:m_equiv}
 Two traces $t_1,t_2 \in \Traces$ are equivalent with respect to $M : \TestSeq \rightarrow \Dist(\Out) \cup \{\bot\}$, denoted $t_1 \equiv_M t_2$, iff $\lastOut(t_1) = \lastOut(t_2)$ and for all continuations $v \in \ContSeq$ it holds that $M(t_1 \cdot v) = M(t_2 \cdot v)$.
\end{definition}
A function $M$ defines an equivalence relation on traces, like the Myhill-Nerode equivalence for formal languages~\cite{zbMATH03146331}\begin{conference}(see also~\cite{lstar_mdp_tech_report})\end{conference}. 
Two traces are $M$-equivalent if they end in the same output and if their behaviour in response to future inputs is the same. Two traces leading to the same \gls*{MDP} state are in the same equivalence class of $\equiv_M$, as in Mealy machines \cite{DBLP:conf/sfm/SteffenHM11}.
% In the following, we will refer to a function $M$ induced by an \gls*{MDP} $\mathcal{M}$ also as the semantics of $\mathcal{M}$. 

\begin{full}
We can now state which functions characterise \glspl*{MDP}, 
as an adaptation of the Myhill-Nerode theorem for regular languages~\cite{zbMATH03146331}, like for Mealy machines~\cite{DBLP:conf/sfm/SteffenHM11}. 
 
\begin{theorem}[Characterisation]
\label{theorem:characterisation}
 A function $M : \TestSeq \rightarrow \Dist(\Out) \cup \{\bot\}$ represents the semantics of an \gls*{MDP}
 iff \begin{compactitem}
      \item $\equiv_M$ has finite index, \hfill finite number of states
      \item $M(\epsilon) = \mathbf{1}_{\{o\}}$ for an $o\in \Out$, \hfill initial output
      \item $\semdom(M)$ is prefix-closed, and 
      \item $\forall t \in \Traces :$
      either $\forall i \in \In : M(t \cdot i) \neq \bot$ or \hfill input enabledness\\
      $\forall i \in \In : M(t \cdot i) = \bot$ 

%       \item $\forall s \in (\Out \times \In)^* : M(s) \neq \bot$ implies $M(s') \neq \bot$ for all prefixes $s'$ of $s$ 
%       \item $\forall s \in (\Out \times \In)^* : M(s) = \bot$ implies $M(s \cdot s') = \bot$ for all continuations $s' \in (\Out \times \In)^*$ 
     \end{compactitem}
\end{theorem} 

\begin{proof}
\emph{Direction $\Rightarrow$:} first we show that the semantics $M$ of an \gls*{MDP} $\mathcal{M} = \langle Q,\In, \Out,q_0, \delta, L\rangle$ fulfils the conditions of Theorem~\ref{theorem:characterisation}. According to \cref{def:mdp_semantics}, $M(\epsilon)(L(q_0)) = 1$, thus the second condition is fulfilled.

Let $t\in \Traces$ be an observable trace, then we have for $i \in \In, o\in \Out$: \linebreak $M(t \cdot i)(o) = \delta(q',i)(q) = p$, where $q'=\delta^*(t)$, if $p > 0$ and $L(q) = o$. Since $\mathcal{M}$ contains finitely many states $q'$, $\delta(q',i)$ and therefore also $M(t \cdot i)$ take only finitely many values. $M$-equivalence of traces $t_i$ depends 
on the outcomes of $M$ and on their last outputs $\lastOut(t_i)$, which are both finite, therefore $M$-equivalence defines finitely many equivalence classes for observable traces.
For non-observable $t\in \Traces$ we have $\delta^*(t) = \bot$ which implies $M(t \cdot i) = \bot$. As a consequence of \cref{lem:non_observable_prefix}, we also have $M(t \cdot c) = \bot$ for any $c \in \ContSeq$. Hence, non-observable traces are equivalent with respect to $M$ if they end in the same output, therefore $M$ defines finitely many equivalence classes for non-observable traces. In summary, $\equiv_M$ has finite index, i.e. the first condition is fulfilled. 
Prefix-closedness of the defined domain $\semdom(M)$ of $M$ follows from \cref{lem:non_observable_prefix}. Any extension of a non-observable test sequence is also non-observable, thus $M$ fulfils the third condition.

For the fourth condition, we again distinguish two cases. If $t$ is a non-observable trace, i.e. $\delta^*(t) = \bot$, then $M(t \cdot i) = \bot$ for all $i \in \In$ according to \cref{def:mdp_semantics}, which fulfils the second sub-condition.  For observable $t$, the distribution $M(t \cdot i)$ depends on $\delta(\delta^*(t),i)$, which is defined for all $i$ due to input-enabledness of $\mathcal{M}$, satisfying the first subcondition. 

\emph{Direction $\Leftarrow$:}  
from an $M$ satisfying the conditions given in Theorem~\ref{theorem:characterisation}, we can construct an \gls*{MDP} $\mathcal{M}_c = \langle Q,\In, \Out,q_0, \delta, L\rangle$ by: 
\begin{itemize}
 \item $Q = (\Traces / \equiv_M) \setminus \{[t] |t \in \Traces, \exists i \in I: M(t\cdot i) = \bot\}$
 \item $q_0 = [o_0]$, where $o_0 \in \Out$ such that $M(\epsilon) = \mathbf{1}_{\{o_0\}}$
 \item $L([s \cdot o]) = o$ (by Def.~\ref{def:m_equiv} all traces in the same equivalence class end with the same output)
 \item for $[t] \in Q$: \\ 
 $\delta([t],i)([t \cdot i \cdot o]) = M(t \cdot i)(o)$ (defined by fourth condition of Theorem~\ref{theorem:characterisation})
\end{itemize}
Each equivalence class of $\equiv_M$ gives rise to exactly one state in $Q$, except for the equivalence classes of 
non-observable traces $Q$. 
\end{proof}

The \gls*{MDP} $\mathcal{M}_c$ in the above construction is minimal with respect to the number of states and unique, up to isomorphism. Therefore, we refer such an \gls*{MDP} as canonical \gls*{MDP} $\can(M)$ for \glspl*{MDP} semantics $M$. 
\end{full}

% Commented out as it can be found Sect. 3 
% {\color{blue}
% In the following we fix the MDP $\mathcal{M} = \langle Q,\Sigma^\mathrm{in}, \Sigma^\mathrm{out},q_0, \delta, L\rangle$ and refer implicitly to its elements. 
% \begin{definition}[Probabilistic Bisimulation]
% The equivalence relation $R \subseteq Q \times Q$ is a probabilistic bisimulation over $\mathcal{M}$ iff $q \mathrel{R} q'$ implies that 
% \begin{enumerate*}[label=(\roman*)]
% \item $L(q) = L(q')$, and
% \item for all $i \in \In$ and all $C \in Q/_{R}$, $\sum_{r \in C} \delta(q,i)(r) = \sum_{r \in C} \delta(q',i)(r)$.
% \end{enumerate*}
% \end{definition}
% Two MPDs are probabilistic bisimilar if the their respective initial states are related by a bisimulation over the disjoint union of the two MDPs.
% }
 \begin{full}
Viewing \glspl*{MDP} as reactive systems, we consider two \glspl*{MDP} to be equivalent,
if we make the same observations on both. 
\begin{definition}[Output-Distribution Equivalence] %of \acrfullpl*{MDP}]
\label{def:eq_mdps}
\glspl*{MDP} 
 $\mathcal{M}_1$ and $\mathcal{M}_2$ over $\In$ and $\Out$ with
 semantics $M_1$ and $M_2$ are output-distribution equivalent, 
 denoted $\mathcal{M}_1 \equiv_\mathrm{od} \mathcal{M}_2$, iff
 \begin{equation*}
 \forall s \in \TestSeq : M_1(s) = M_2(s)
 \end{equation*}
%  \begin{conference}
%   \ $\forall s \in \TestSeq : M_1(s) = M_2(s)$
%  \end{conference}

%  The function $\mathrm{od}$ computes output distributions in a state $q$ for an input $i$. 
%  Hence, two \glspl*{MDP} are equivalent if all possible sequence in $(\Out \times \In)^*$
%  induce the same output distributions in both \glspl*{MDP}. In genernal, not all traces are 
%  observable for a given \gls*{MDP}, therefore $\mathrm{od}(q,i)$ may return $\bot$, if $q = \bot$.
%  \mt{I am not sure whether this is the smartest choice of equivalence, given the large
%  number of alternative definitions. I assume that for the case of deterministic \glspl*{MDP} 
%  many definitions should be equivalent. 
%  }
\end{definition}
 \end{full}

% \begin{full}
% \subsection{Queries}
% \end{full}
% \begin{conference}
\subsubsection{Queries.}
% \end{conference}

We are now able to define queries focusing on the observable behaviour of \glspl*{MDP}.
Assume that we want to learn a model of a black-box deterministic
\gls*{MDP} $\mathcal{M}$ % = \langle Q,\Sigma^\mathrm{in}, \Sigma^\mathrm{out},q_0, \delta, L\rangle$
, with semantics $M$.
Output distribution queries ($\odq$) and equivalence queries ($\eq$) are then defined as follows: 
\begin{compactitem}
\begin{full}
 \item {\sl output distribution ($\odq$)}: an $\odq$ query takes a sequence $s \in\TestSeq$ as input and returns $M(s)$.
\end{full}
\begin{conference}
 \item {\sl output distribution ($\odq$)}: an $\odq(s)$ returns $M(s)$ for input $s \in\TestSeq$.
\end{conference}

 \item {\sl equivalence ($\eq$)}: an $\eq$ query takes a hypothesis Mealy machine $\mathcal{H}$ with semantics $H$ 
 as input and  returns \emph{yes} if $\mathcal{H} \equiv_\mathrm{od} \mathcal{M}$; %if for all $s \in (\Out \times \In)^* : H(s) = M(s)$, 
 otherwise it returns an $s \in \TestSeq$ such that $H(s) \neq M(s)$ and $M(s) \neq \bot$.
\end{compactitem}

\begin{full}
\begin{remark}
\label{rem:minimal_cex}
 For any counterexample $s$ to $\mathcal{H} \equiv_\mathrm{od} \mathcal{M}$ such that $M(s) = \bot$, there is a prefix $s'$ of $s$ with $H(s') \neq M(s')$ and $M(s) \neq \bot$, i.e. $s'$ is also a counterexample but observable on the \SUL{} with semantics $M$. 
Hence, we can restrict potential counterexamples to be observable test sequences. 
\end{remark}

\begin{proof}
%  We can show this by contradiction.
Since $s$ is a counterexample and  $M(s) = \bot$, we have $H(s) \neq \bot$. Let $s''$ be the the longest prefix of $s$ such that $M(s'') = \bot$, thus $s''$ is of the form $s'' = s' \cdot o \cdot i$ with $M(s')(o) = 0$. Due to prefix-closedness of $\semdom(H)$, $H(s) \neq \bot$ implies $H(s'') \neq \bot$, therefore $H(s')(o) > 0$. Hence, $s'$ with $M(s')\neq \bot$ is also a counterexample because $H(s') \neq M(s')$. 
\end{proof}
\end{full}

\subsubsection{Observation Tables.}
Like $L^*$, we store information in observation table triples $\langle S,E,T\rangle$, where:
\begin{compactitem}
\begin{full}
 \item $S \subset \Traces$ is a prefix-closed set of traces, initialised to
 $\{o_0\}$, a singleton set containing the trace consisting of the initial output $o_0$ of the \SUL{},
given by $\odq(\epsilon)(o_0) = 1$.
\end{full}
\begin{conference}
  \item $S \subset \Traces$ is a prefix-closed set of traces, 
 initialised to $\{o_0\}$, with $o_0$ being the initial \SUL{} output,
% retrieved by $\odq(\epsilon)(o_0) = 1$.
\end{conference}
 \item $E \subset \ContSeq$ is a suffix-closed set of continuation sequences, initialised
 to $\In$, 
%  \item $T : (\Out \times \In)^* \rightarrow \Dist(\Out)$ is a mapping from sequences stored in the table to output distributions. 
 \item $T : (S \cup \LongTr(S)) \cdot E \rightarrow \Dist(\Out) \cup \{\bot\}$ is a mapping from test sequences to output distributions 
 or $\bot$ denoting undefined behaviour. This mapping basically stores a finite subset of $M$.
%  More concretely, 
%  the domain of $T$ is $S \cdot E \cup L \cdot E$, with 
 The set $\LongTr(S) \subseteq S \cdot \In \cdot \Out$ is given by $\LongTr(S) = \{s\cdot i \cdot o | s \in S, i \in \In, o \in \Out, \odq(s \cdot i)(o) > 0 \}$.
%  Hence, actually we have $T : (S \cup L) \cdot E \rightarrow  \Dist(\Out) \cup \{\bot\}$.
%  The mapping $T$ stores $\bot$ for a sequence $s \cdot i$ if the trace $s$ is not observable for the \SUL. 
\end{compactitem}
We can view an observation table as a two-dimensional array with rows labelled by traces in $S\cup \LongTr(S)$ and columns
labelled by $E$. We refer to traces in $S$ as short traces and to their extensions in $\LongTr(S)$ as long traces. 
An extension $s\cdot i \cdot o$ of a short trace $s$ is in $\LongTr(S)$ if $s \cdot i \cdot o$ is observable. 
\begin{full}
Analogously to traces, we refer to rows labelled by $S$ as short rows and we refer to rows labelled 
by $\LongTr(S)$ as long rows.  
\end{full}
\begin{conference}
 Analogously to traces, we refer to rows labelled by $S$ as short rows.
\end{conference}
The table cells store the mapping defined by $T$. To represent rows labelled by traces $s$
we use functions $\row(s) : E \rightarrow \Dist(\Out) \cup \{\bot\}$ for $s \in S \cup \LongTr(S)$ with $\row(s)(e) = T(s\cdot e)$.
Equivalence of rows labelled by traces $s_1,s_2$, denoted $\equivrow_E(s_1,s_2)$, holds iff $\row(s_1) = \row(s_2) \land \lastOut(s_1) = \lastOut(s_2)$ and approximates $M$-equivalence $s_1 \equiv_M s_2$,
by considering only continuations in $E$, i.e. $s_1 \equiv_M s_2$ implies $\equivrow_E(s_1,s_2)$.
The observation table content defines the structure of hypothesis \glspl*{MDP} based on the following principle: 
we create one state per equivalence class of $S/\equivrow_E$, thus we identify states with traces in $S$ reaching 
them and we distinguish states by their future behaviour in response to sequences in $E$ 
(as is common in active automata learning~\cite{DBLP:conf/sfm/SteffenHM11}).
The long traces $\LongTr(S)$ serve to define transitions. 
Transition probabilities are given by the distributions in the mapping $T$. 

\cref{tab:obs_table_coffee}
shows a part of the observation table created during learning of the coffee machine shown in \cref{fig:coffee_mdp}.
The set $S$ has a trace for each state of the \gls*{MDP}. Note that these traces are pairwise inequivalent with respect to $\equivrow_E$, where
$E = \In = \{\mathtt{but},\mathtt{coin}\}$. We only show one element of $\LongTr(S)$, which gives rise
to the self-loop in the initial state with the input $\mathtt{but}$ and probability $1$.

\begin{table}[t]
\centering
\caption{Parts of observation table for learning the faulty coffee machine (\cref{fig:coffee_mdp}).}
\label{tab:obs_table_coffee}
 \begin{tabular}{r|l|c|c}
  &              & $\mathtt{but}$ & $\mathtt{coin}$ \\\hline
\multirow{3}{*}{$S$}&$\mathtt{init}$  & $\{\mathtt{init}\mapsto 1\}$ & $\{\mathtt{beep}\mapsto 1\}$ \\\cline{2-4}
&$\mathtt{init} \cdot \mathtt{coin} \cdot \mathtt{beep}$  
                 & $\{\mathtt{coffee}\mapsto 0.8,\mathtt{init}\mapsto 0.2\}$ & $\{\mathtt{beep}\mapsto 1\}$ \\\cline{2-4}
&$\mathtt{init} \cdot \mathtt{coin} \cdot \mathtt{beep}\cdot \mathtt{but} \cdot \mathtt{coffee}$  
                 & $\{\mathtt{init}\mapsto 1\}$ & $\{\mathtt{beep}\mapsto 1\}$ \\\hline
\multirow{2}{*}{$\LongTr(S)$} & $\mathtt{init} \cdot \mathtt{but} \cdot \mathtt{init}$ & $\{\mathtt{init}\mapsto 1\}$ & $\{\mathtt{beep}\mapsto 1\}$ \\\cline{2-4}
                   &  \ldots & \ldots & \ldots 
 \end{tabular}

\end{table}

\begin{definition}[Closedness]
 An observation table $\langle S,E,T\rangle$ is closed if for all $l \in \LongTr(S)$
 there is an $s \in S$ such that $\equivrow_E(l,s)$.
\end{definition}

\begin{definition}[Consistency]
 An observation table $\langle S,E,T\rangle$ is consistent if for all $s_1,s_2 \in S, i \in \In, o \in \Out$ such that $\equivrow_E(s_1,s_2)$
 it holds either that (1) $T(s_1 \cdot i)(o) = 0 \land T(s_2 \cdot i)(o) = 0$\begin{full}
\footnote{
 Note that $s_1 \in S$ implies that $T(s_1 \cdot i) \neq \bot$ such that $T(s_2 \cdot i)(o) = 0$ follows from $\equivrow_E(s_1,s_2)$ and $T(s_1 \cdot i)(o) = 0$.} 
\end{full} 
\ or (2) $\equivrow_E(s_1 \cdot i \cdot o,s_2 \cdot i \cdot o)$.
                                         
\end{definition}
Closedness and consistency are required to derive well-formed hypotheses, analogously
to $L^*$~\cite{DBLP:journals/iandc/Angluin87}. 
We require closedness to create transitions for all inputs in all states and we require consistency
to be able to derive deterministic hypotheses.
During learning, we apply Algorithm~\ref{alg:make_closed_and_cons} repeatedly to
establish closedness and consistency of observation tables. The algorithm adds a new short trace if the table is not closed and 
adds a new column if the table is not consistent.
\begin{algorithm}[t]
\begin{algorithmic}[1]
 \Function{MakeClosedAndConsistent}{$\langle S, E, T\rangle$}
    \If{$\langle S, E, T\rangle$ is not closed}
    \State $l \gets l' \in \LongTr(S)$ such that $\forall s \in S: \row(s) \neq \row(l') \lor \lastOut(s) \neq \lastOut(l')$
    \State $S \gets S \cup\{ l \}$
%         \State \Return $\langle S, E, T\rangle$ % \Call{MakeClosedAndConsistent}{$\langle S, E, T\rangle$}
    \ElsIf{$\langle S, E, T\rangle$ is not consistent}  %\Comment{$\langle S, E, T\rangle$ is not consistent}
    \ForAll{$s_1,s_2 \in S$ such that $\equivrow_E(s_1,s_2)$}
%     \If{$\exists i \in \In, o\in \Out: T(s \cdot i)(o) > 0 \land \row(s\cdot i \cdot o) \neq \row(s'\cdot i \cdot o)$}
    \ForAll{$i \in \In, o\in \Out$}
       \If{$T(s_1 \cdot i)(o) > 0$ and $\lnot \equivrow_E(s_1\cdot i \cdot o,s_2\cdot i \cdot o)$}
       \State $e \gets e' \in E$ such that $T(s_1 \cdot i \cdot o \cdot e') \neq T(s_2 \cdot i \cdot o \cdot e')$
       \State $E \gets E \cup \{i \cdot o \cdot e\}$
%         \State \Return $\langle S, E, T\rangle$ %\Call{MakeClosedAndConsistent}{$\langle S, E, T\rangle$}
    \EndIf
    \EndFor
    \EndFor
%     \Else
    \EndIf
    \State \Return$\langle S, E, T\rangle$
 \EndFunction
\end{algorithmic}
\caption{Making an observation table closed and consistent}
\label{alg:make_closed_and_cons}
\end{algorithm}

We derive a hypothesis $\mathcal{H} = \langle Q_\mathrm{h},\In, \Out,{q_0}_\mathrm{h}, \delta_\mathrm{h}, L_\mathrm{h}\rangle$ from a closed and consistent observation table 
$\langle S,E,T\rangle$, denoted $\mathcal{H} = \mathrm{hyp}(S,E,T)$, as follows:
\begin{compactitem}
 \item $Q_\mathrm{h} = \{\langle\lastOut(s),\row(s)\rangle|s \in S\}$
 \item ${q_0}_\mathrm{h} = \langle o_0,\row(o_0)\rangle$, $o_0 \in S$ is the trace consisting of the 
  initial \SUL{} output
  
  \item for $s \in S$, $i \in \In$ and $o\in \Out$ : \\ $\delta_\mathrm{h}(\langle\lastOut(s),\row(s)\rangle,i)(\langle o,\row(s \cdot i \cdot o)\rangle) = p$ if  $T(s \cdot i)(o) = p > 0$
 \item for $s \in S$: 
 $L_\mathrm{h}(\langle\lastOut(s),\row(s)\rangle)  = \lastOut(s)$
\end{compactitem}

\subsubsection{Learning Algorithm.}

Algorithm \ref{alg:exact_l_star_mdp} implements \LstarMdpE{} using queries $\odq$ and $\eq$. 
%, which learns \glspl*{MDP} based on the queries and the observation table described above.
% can be learned via Algorithm \ref{alg:exact_l_star_mdp}, which is abbreviated .
First, the algorithm initialises the observation tables and fills the table cells with output distribution queries (Lines~\ref{alg:lstar_exact:init_1} to \ref{alg:lstar_exact:init_2}).
The main loop in Lines \ref{alg:lstar_exact:main_1} to \ref{alg:lstar_exact:main_2} makes
the observation table closed and consistent, derives a hypothesis $\mathcal{H}$ and performs an equivalence query $\eq(\mathcal{H})$.
If a counterexample $\textit{cex}$ is found, all its prefix traces are added as short traces to $S$, otherwise the final hypothesis
is returned, as it is output-distribution equivalent to the \SUL. Whenever the table contains empty cells, the \textsc{Fill}
procedure assigns values to these cells via $\odq$.

\begin{algorithm}[t]
 \begin{algorithmic}[1]
 \Require $\In$, exact teacher capable of answering $\odq$ and $\eq$
 \Ensure learned model $\mathcal{H}$ (final hypothesis)
  \State $\mathit{o_0} \gets o$ such that $\odq(\epsilon)(o) = 1$ \label{alg:lstar_exact:init_1}
  \State $S \gets \{o_0\}$, $E \gets \In$ %\Comment{Initialise observation table}
%   \State 
  \State \Call{fill}{$S,E,T$} \label{alg:lstar_exact:init_2}
%   \ForAll{$s \in S, e \in E$}
%   \State $T(s \cdot e) \gets \textsc{odq}(s \cdot e)$
%   \EndFor
  \Repeat \label{alg:lstar_exact:main_1}
    \While{$\langle S, E, T\rangle$ not closed or not consistent}
    \State $\langle S,E,T \rangle \gets \Call{MakeClosedAndConsistent}{\langle S,E,T \rangle}$
    \State \Call{fill}{$S,E,T$}
    \EndWhile
    \State $\mathcal{H} \gets \mathrm{hyp}(S,E,T)$
    \State $\mathit{eqResult} \gets \eq(\mathcal{H})$
    \If{$\mathit{eqResult} \neq \mathit{yes}$}
      \State $cex \gets \mathit{eqResult}$ 
      \ForAll{$(t \cdot i) \in \prefixes(cex)$ with $i \in \In$} 
	\State $S \gets S \cup \{t\}$ \label{algline:exact_l_star_mdp:add_cex_prefix}
      \EndFor
      \State \Call{fill}{$S,E,T$}
    \EndIf
  \Until{$\mathit{eqResult} = \mathit{yes}$} \label{alg:lstar_exact:main_2}
  \State \Return $\mathrm{hyp}(S,E,T)$ 
%   \item[]
  \Procedure{fill}{$S,E,T$}
  \ForAll{$s \in S \cup \LongTr(S), e \in E$}
  \If{$T(s \cdot e)$ undefined} \Comment{we have no information about $T(s \cdot e)$ yet}
    \State $T(s \cdot e) \gets \odq(s \cdot e)$
  \EndIf
  \EndFor
  \EndProcedure
 \end{algorithmic}
\caption{The main algorithm implementing \LstarMdpE{}}
\label{alg:exact_l_star_mdp}
\end{algorithm}

\begin{full}
 
\subsubsection{Correctness \& Termination.}
\label{sec:exact_convergence}

In the following, we will show that \LstarMdpE{} terminates and learns correct models, i.e. models that are output-distribution equivalent to the \SUL. 
% We will mostly follow Angluin's reasoning for the original $L^*$ algorithm~\cite{DBLP:journals/iandc/Angluin87}, e.g., we introduce similar lemmas and theorems.
Like Angluin~\cite{DBLP:journals/iandc/Angluin87}, we will show that derived hypotheses are consistent with queried information and that they are minimal with respect to the number of states.
For the remainder of this section, let $M$ be the semantics of the \gls*{MDP} underlying the \SUL{} and let $\mathcal{M} = \can(M)$ be the corresponding canonical \gls*{MDP}
and let $\mathcal{H} = \langle Q,\In, \Out,{q_0}, \delta, L\rangle$ denote hypotheses.
The first two lemmas relate to observability of traces. 
\begin{lemma}
\label{lem:prob_0_semantics_imply_bot}
 For all $s \in \TestSeq, o \in \Out, e \in \ContSeq: M(s)(o) = 0 \Rightarrow M(s \cdot o \cdot e) = \bot$. 
\end{lemma}

\begin{proof}
Let $\delta_M$ be the probabilistic transition relation of $\mathcal{M}$.
 $M(s)(o) = 0$ with $s = t \cdot i$, for $i \in \In$ implies that there is no state labelled $o$ reachable by executing $i$ in the state
 $\delta_M^*(t)$ (\cref{def:mdp_semantics}), thus $\delta_M^*(t \cdot i \cdot o) = \delta_M^*(s \cdot o) = \bot$. 
 By \cref{def:mdp_semantics}, $M(s \cdot o \cdot i') = \bot$ for any $i'$. Due to prefix-closedness of 
 $\semdom(M)$, we have $M(s \cdot o \cdot e) = \bot$ for all $e\in \ContSeq$.  

\end{proof}

\begin{lemma}
\label{lem:pos_prob}
 Let $\langle S,E,T\rangle $ be a closed and consistent observation table. Then for $s\in S$ and $s \cdot i \cdot o \in S\cup \LongTr(S)$
 we have $T(s\cdot i)(o) > 0$.
\end{lemma}

\begin{proof}
 
The lemma states that traces labelling rows are observable.
\cref{alg:exact_l_star_mdp} adds elements to $S$ and consequently $\LongTr(S)$ in two cases: (1) if an equivalence query 
returns a counterexample and (2) to make observation tables closed.

\paragraph{Case 1.}
Counterexamples $c \in \TestSeq$ returned by equivalence queries $\eq(\mathcal{H})$ satisfy $M(c) \neq \bot$ (see also \cref{rem:minimal_cex}).  In \cref{algline:exact_l_star_mdp:add_cex_prefix} of \cref{alg:exact_l_star_mdp}, we add $t_p$ to $S$ for each $t_p \cdot i_p \in \prefixes(c)$. 
Due to prefix-closedness of $\semdom(M)$, $M(t_p \cdot i_p) \neq \bot$ for all $t_p \cdot i_p \in \prefixes(c)$, and therefore $M(s \cdot i)(o) = T(s \cdot i)(o) > 0$ for each added trace $t_p$ of the form $t_p = s\cdot i \cdot o$ with $i \in \In$ and $o\in \Out$. The set $\LongTr(S)$ is implicitly extended by all observable extensions of added $t_p$. By this definition, $\LongTr(S)$ contains only traces $t = s\cdot i \cdot o$ such that $T(s \cdot i)(o) > 0$. 

\paragraph{Case 2.}
If an observation table is not closed, we add traces from $\LongTr(S)$ to $S$. As noted above, all traces $t = s\cdot i \cdot o$ in $\LongTr(S)$ satisfy $T(s \cdot i)(o) > 0$. Consequently, all traces added to $S$ satisfy this property as well. 
\end{proof}

\begin{theorem}[Minimality]
\label{theorem:consistent_and_minimal}
 Let $\langle S,E,T\rangle$ be a closed and consistent observation table and let 
 $\mathcal{H} = \mathrm{hyp}(S,E,T)$ be a hypothesis derived from that table with semantics $H$.
 Then $\mathcal{H}$ is consistent with $T$, that is, $\forall s \in (S \cup \LongTr(s)) \cdot E : T(s) = H(s)$,
 and any other \gls*{MDP} consistent with $T$ but inequivalent to $\mathcal{H}$ must have more states.
\end{theorem}

\begin{lemma}
\label{lem:delta_row}
 Let $\langle S,E,T\rangle$ be a closed and consistent observation table. For $\mathcal{H} = \mathrm{hyp}(S,E,T)$ and
 every $s \in S \cup \LongTr(S)$, we have $\delta^*(q_0,s) = \langle \lastOut(s),\row(s)\rangle$. 
\end{lemma}

\begin{proof}

Similarly to \cite{DBLP:journals/iandc/Angluin87}, we prove this by induction on the trace length $k$, i.e. the number of 
outputs in $s$. For $k = 1$, i.e. $s = o$, where $o$ is the initial output, we have $\delta^*(q_0,o) = 
\delta^*(\langle o,\row(o)\rangle,o) = \langle o,\row(o)\rangle$. 

Assume that for every $s \in S \cup \LongTr(S)$ of length at most $k$, $\delta^*(q_0,s) = \langle \lastOut(s),\row(s)\rangle$.
Let $t \in S \cup \LongTr(S)$ of length $k+1$, i.e. $t = s \cdot i \cdot o_t$, for $s \in \Traces, i\in \In, o_t \in \Out$.
If $t \in \LongTr(S)$ % =  \{s\cdot i \cdot o | s \in S, i \in \In, o \in \Out, T(s \cdot i)(o) > 0 \}$, 
then $s$ must be in $S$, and if $t \in S$, then $s \in S$ because $S$ is prefix-closed. 
\begin{align*}
 \delta^*(q_0,s \cdot i \cdot o_t) &=  \Delta(\delta^*(q_0, s),i,o_t) \\
   &=  \Delta(\langle\lastOut(s),\row(s)\rangle,i,o_t) \tag{by induction hypothesis}\\
 &= \langle o_t, \row(s \cdot i \cdot o_t)\rangle \tag{definition of $\Delta$}
 \\ &\text{ if }
 \delta(\langle\lastOut(s),\row(s)\rangle,i)(\langle o_t, \row(s \cdot i \cdot o_t)\rangle) > 0 \\
 &\text{ and }  L(\langle o_t, \row(s \cdot i \cdot o_t)\rangle) = o_t \\
  \cline{1-2}
  \delta(\langle\lastOut(s),\row(s)\rangle,i)&(\langle o_t, \row(s \cdot i \cdot o_t)\rangle) > 0\\
  &\Leftrightarrow T(s\cdot i)(o_t) > 0 \tag{construction of $\delta$} \\
  &\Leftrightarrow \true \tag{\cref{lem:pos_prob} } \\
   L(\langle o_t, \row(s \cdot i \cdot o_t)\rangle) &= o_t \\ &\Leftrightarrow \true \tag{construction of $L$}
\end{align*}

\end{proof}

\begin{lemma}
\label{lem:consistent}
 Let $(S,E,T)$ be a closed and consistent observation table. Then $\mathrm{hyp}(S,E,T)$ is consistent with $T$, i.e.
 for every $s \in S\cup \LongTr(S)$ and $e \in E$ we have $T(s \cdot e) = H(s \cdot e)$.
\end{lemma}
\begin{proof}

We will prove this by induction on the length of $e$, i.e. the number of inputs of $e$. 
As induction hypothesis, we assume $T(s \cdot e) = H(s \cdot e)$ for all $s \in S\cup \LongTr(S)$ and $e \in E$ of length at most $k$. 
For the base case, we consider $e$ consisting of a single input, i.e. $e \in \In$. 
From \cref{def:mdp_semantics} we can derive that $H(s \cdot i) \neq \bot$ if $\delta^*(s) \neq \bot$, then we have:
\begin{align*}
 H(s \cdot i)(o) &= \delta(\delta^*(s),i)(q) \text{ with } L(q) = o \\
 &= \delta(\langle \lastOut(s), \row(s)\rangle,i)(q) \text{ with } L(q) = o \tag{\cref{lem:delta_row} } \\
 &= \delta(\langle \lastOut(s), \row(s)\rangle,i)(\langle o, \row(s \cdot i \cdot o)\rangle) \tag{hypothesis construction} \\
 &= T(s \cdot i)(o) \tag{hypothesis construction}
\end{align*}

For the induction step, let $e \in E$ be of length $k+1$, thus it is of the form $e = i \cdot o \cdot e_k$ for $i \in \In$, $o \in \Out$, 
and due to suffix-closedness of $E$, $e_k \in E$. We have to show that $T(s \cdot e) = H(s \cdot e)$ for $s \in S\cup \LongTr(S)$.
Let $s' \in S$ such that $\equivrow_E(s,s')$, which exists due to observation table closedness. Traces $s$ and $s'$ 
lead to the same hypothesis state because:
\begin{align*}
  \delta^*(q_0,s) &= \langle \lastOut(s), \row(s)\rangle \tag{\cref{lem:delta_row} } \\
  &=\langle \lastOut(s'), \row(s')\rangle \tag{$\equivrow_E(s,s')$} \\
  &=\delta^*(q_0,s')\tag{\cref{lem:delta_row} }
\end{align*}
Thus, $s$ and $s'$ are $H$-equivalent and therefore $H(s \cdot e) = H(s' \cdot e)$.
Due to $\equivrow_E(s,s')$, $T(s \cdot e) = T(s' \cdot e)$ and in combination:
\begin{equation*}
 T(s \cdot e) = H(s \cdot e) 
\Leftrightarrow T(s' \cdot e) = H(s' \cdot e) 
\Leftrightarrow T(s' \cdot i \cdot o \cdot e_k) = H(s' \cdot i \cdot o \cdot e_k)
\end{equation*}

\paragraph{Case 1.} Suppose that $s' \cdot i \cdot o \in S \cup \LongTr(S)$. Then, $T(s' \cdot i \cdot o \cdot e_k) = H(s' \cdot i \cdot o \cdot e_k)$ holds by the induction hypothesis, as $e_k$ has length $k$. 

\paragraph{Case 2.} Suppose that $s' \cdot i \cdot o \notin S \cup \LongTr(S)$. Since $s' \in S$ and by the definition of $\LongTr(S)$, we have $\odq(s'\cdot i)(o) = M(s'\cdot i)(o) = 0$.
By \cref{lem:prob_0_semantics_imply_bot}, if follows that  $M(s' \cdot i \cdot o'\cdot e) = \bot$ for any continuation $e \in \ContSeq$.
As the observation table is filled via $\odq$ we have $T(s' \cdot i \cdot o \cdot e_k) = \odq(s' \cdot i \cdot o \cdot e_k) = M(s' \cdot i \cdot o \cdot e_k) = \bot$.
By the induction base, we have $H(s' \cdot i) = T(s'\cdot i)$ for $i\in \In$, thus $H(s'\cdot i)(o) = T(s'\cdot i)(o)= 0$. With \cref{lem:prob_0_semantics_imply_bot}, we can conclude $H(s' \cdot i \cdot o  \cdot e_k) = \bot$. 

In both cases, it holds that $H(s' \cdot i \cdot o  \cdot e_k) = T(s' \cdot i \cdot o  \cdot e_k)$ which is equivalent to $H(s \cdot e) = T(s \cdot e)$.
% Now, we introduce the final lemma to prove Theorem~\ref{theorem:consistent}.%, which is again similar to one of Angluin's lemmas.
\end{proof}

With \cref{lem:consistent}, we have shown consistency between derived hypotheses and the queried information. 
Now, we show that hypotheses are minimal with respect to the number of states.

\begin{lemma}
 \label{lem:minimal}
  Let $\langle S,E,T \rangle$ be a closed and consistent observation table and let $n$ be the number of different
  values for $\langle\lastOut(s), \row(s)\rangle$, i.e. the number of states of hypothesis $\mathrm{hyp}(S,E,T)$. 
  Any \gls*{MDP} consistent with $T$ must have at least $n$ states. 
\end{lemma} 
% \todo[inline]{Angluin proved lemmas analogous to \cref{lem:minimal} and \cref{lem:isomorphic}, but she proved \cref{lem:isomorphic} first. She proved minimality twice (once as part of \cref{lem:isomorphic}) and I don't understand that. It seems that this order simplifies things. }

\begin{proof}
Let $\mathcal{M}' = \langle Q',\In, \Out,q_0', \delta', L'\rangle$ with semantics $M'$ be an \gls*{MDP} consistent with $T$. 
Let $s_1,s_2 \in S$ such $\lnot \equivrow_E(s_1,s_2)$, then (1) $\lastOut(s_1) \neq \lastOut(s_2)$ or (2) $\row(s_1) \neq \row(s_2)$.
If $\lastOut(s_1) \neq \lastOut(s_2)$, then $s_1$ and $s_2$ cannot reach the same state in $\mathcal{M}'$, because the states reached by $s_1$ and $s_2$ need to be labelled differently. 
If $\row(s_1) \neq \row(s_2)$, then there exists an $e\in E$ such that $M'(s_1 \cdot e) \neq M'(s_2 \cdot e)$, because $\mathcal{M}'$ is consistent with $T$. 
In this case $s_1$ and $s_2$ cannot reach the same state, as the observed future behaviour is different. Consequently, $\mathcal{M}'$ has at least $n$ states.
\end{proof}

% Like Angluin~\cite{DBLP:journals/iandc/Angluin87}, we show that other \glspl*{MDP} consistent with queried information and at most $n$ states, are isomorphic to our hypotheses. 

\begin{lemma}
\label{lem:isomorphic}
 Let $\langle S,E,T \rangle$ be a closed and consistent observation table and $\mathcal{H} = \mathrm{hyp}(S,E,T)$ be a hypothesis with $n$ states
 derived from it. Any other \gls*{MDP} $\mathcal{M}' = \langle Q',\In, \Out,q_0', \delta', L'\rangle$ with semantics $M'$ consistent with $T$,
 initial output $L'(q_0)$, and %initial output of the \SUL (i.e. $L'(q_0') = o$ where $o$ is the \SUL's initial output), 
 with $n$ or fewer states is isomorphic to $\mathrm{hyp}(S,E,T)$.
\end{lemma}

\begin{proof}
From \cref{lem:minimal}, it follows that $\mathcal{M}'$ has at least $n$ states, therefore we examine $\mathcal{M}'$ with exactly $n$ states. 
For each state of $\mathcal{H}$, i.e. each unique row labelled by $s\in S$, exists a unique state in $Q'$.
We will now define a mapping $\phi$ from short traces to $Q'$ given by $\phi(\langle\lastOut(s),\row(s)\rangle) = {\delta'}^*(q_0',s)$
for $s$ in $S$. It is bijective and we will now show that it maps $q_0$ to $q_0'$, that it preserve the probabilistic transition
relation and that it preserves labelling.
% We show that labelling is preserved by showing that $\phi$ preserves the output distribution function $\mathrm{od}$.

\noindent First, we start with the initial state and show $\phi(q_0) = q_0'$:
\begin{align*}
 \phi(q_0) &= \phi(\langle o,\row(o)\rangle) \text{ where $o$ is the initial output of the \SUL} \\
 &= {\delta'}^*(q_0',o)  \\
 &= q_0' \tag{definition of ${\delta'}^*$}%\text{ since $L'(q_0) = o$ }
\end{align*}

% For each $s$ in $S$ and $(i,o)$ in $\In \times \Out$ such that $s \cdot i \cdot o$ in $L$, let $s_1$ be in $S$ with 
% $\row(s_1) = \row(s \cdot i \cdot o)$ and $\lastOut(s_1) = o$.  

For each $s$ in $S$, $i$ in $\In$ and $o \in \Out$. We have:
\begin{align*}
\delta(\langle \lastOut(s),\row(s)\rangle),i)&(\langle \lastOut(s\cdot i \cdot o),\row(s\cdot i \cdot o)\rangle) \\
&= T (s \cdot i)(o) \tag{hypothesis construction}\\
&\text{and} \\
 \delta'(\phi(\langle \lastOut(s),\row(s)\rangle),i)&(\phi(\langle \lastOut(s\cdot i\cdot o),\row(s \cdot i \cdot o)\rangle))) \\
 &= \delta'({\delta'}^*(q_0',s),i)({\delta'}^*(q_0',s \cdot i \cdot o)) \\
 &= M'(s\cdot i)(o) \tag{\cref{def:mdp_semantics}} \\
 &= T(s \cdot i)(o) \tag{$\mathcal{M}'$ is consistent with $T$} \\
 &\text{Transition probabilities are preserved.}
\end{align*}

% \begin{align*}
%  \phi(\delta(\langle \lastOut(s),\row(s)\rangle,i)) &= \phi(\bigl\{\langle o,\row(s \cdot i \cdot o)\rangle \mapsto p | o \in \Out : T(s \cdot i)(o) = p > 0 \bigr\})\\
%  &= \bigl\{\phi(\langle o,\row(s \cdot i \cdot o)\rangle) \mapsto p | o \in \Out : T(s \cdot i)(o) = p > 0 \bigr\}  \\
%  &\qquad \mu \in \Dist(Q) \Rightarrow \phi(\mu) \in \Dist(Q') \text{ because $\phi$ is one-to-one} \\
%  &= \bigl\{ {\delta'}^*(q_0',s \cdot i \cdot o) \mapsto p | o \in \Out : T(s \cdot i)(o) = p > 0 \bigr\}
% \end{align*}
% and 
% \begin{align*}
%  \delta'(\phi(\langle \lastOut(s),\row(s)\rangle),i) &= \delta'({\delta'}^*(q_0',s),i) \\
%  &=\bigl\{{\delta'}^*(q_0',s \cdot i \cdot o) \mapsto p | o \in \Out : T(s \cdot i)(o) = p > 0 \bigr\} \\
%  &\qquad \text{ since } \mathcal{M}' \text{ is consistent with } T \\
%  &=  \phi(\delta(\langle \lastOut(s),\row(s)\rangle,i)) \text{ for all } s \in S, i \in \In
% \end{align*}

Finally, we show that labelling is preserved. For all $s$ in $S$:
\begin{align*}
 L'(\phi(\langle \lastOut(s), \row(s)\rangle)) &= L'({\delta'}^*(q_0',s)) \\
 &= \lastOut(s) \tag{definition of ${\delta'}^*$ }\\
 &\text{and }\\
 L(\langle\lastOut(s),\row(s)\rangle) &= \lastOut(s) \tag{definition of $L$, thus } \\
 L'(\phi(\langle \lastOut(s), \row(s)\rangle)) &= L(\langle\lastOut(s),\row(s)\rangle) \\
 &\text{Labelling is preserved by the mapping $\phi$.}
\end{align*}
\end{proof}

\begin{theorem}
% Let $\mathcal{M}$ be an \gls*{MDP} modelling the \SUL.
The algorithm \LstarMdpE{}\ terminates and returns an \gls*{MDP} $\mathcal{H}$ isomorphic to $\mathcal{M}$, thus, it is minimal and also satisfies
$\mathcal{M} \equiv_\mathrm{od} \mathcal{H}$.
\end{theorem}

\end{full}

\begin{conference}
 \begin{theorem}
% Let $\mathcal{M}$ be an \gls*{MDP} modelling the \SUL.
\LstarMdpE{}\ terminates and learns an \gls*{MDP} $\mathcal{H}$ that is output-distribution equivalent to the \SUL{} and minimal
in the number of states (proof~\cite{lstar_mdp_tech_report}).
\end{theorem}

\end{conference}

\begin{full}
 
\paragraph{Termination.}
Let $\langle S,E,T\rangle$ be a closed and consistent observation table and let $c \in \TestSeq$ be a counterexample to equivalence between 
$\mathcal{M}$ and hypothesis $\mathrm{hyp}(S,E,T)$ with semantics $H$. Since $c$ is a counterexample, $M(c) \neq H(c)$.
% We have $T(c) \neq \mathrm{od}(\delta^*(q_0,c_1),i)$ for $c = c_1\cdot i$.
Now let $\langle S',E',T'\rangle$ be an observation table extended by adding all prefix traces of $c$ to $S$ and (re-)establishing closedness and consistency. 
For $\mathrm{hyp}(S',E',T') = \mathcal{H}'$ with semantics $H'$,
we have $T'(c) = M(c)$ due to output distribution queries. Since $\mathcal{H}'$ is consistent with $T'$, we have $T'(c) = H'(c) = M(c)$.
% since $\mathcal{M}'$ is consistent with $T'$ (Lemma \ref{lem:consistent}) 
Hence, $H'(c) \neq H(c)$, which shows that $\mathcal{H}'$ is not equivalent to $\mathcal{H}$, with $c$ being a counterexample to equivalence.  
We do not remove elements from $S$, $E$, or $T$, thus $\mathcal{M}'$ is also consistent with $T$. Therefore, $\mathcal{M}'$
must have at least one state more than $\mathcal{M}$ according to Theorem~\ref{theorem:consistent_and_minimal}. 
It follows that each round of learning, which finds a counterexample, adds at least one state.
Since Algorithm \ref{alg:exact_l_star_mdp} derives minimal hypotheses
and $\mathcal{M}$ can be modelled with finitely many states, there can only be finitely many rounds that find counterexamples.
Hence, we terminate after a finite number of rounds, because Algorithm \ref{alg:exact_l_star_mdp} returns the final hypothesis as soon as no counterexample 
can be found via equivalence queries $\eq$.

\paragraph{Correctness.}
The algorithm terminates when the equivalence query $\eq(\mathcal{H})$ does not find any new counterexample between the final hypothesis $\mathcal{H}$ and $\mathcal{M}$.
Since there is no counterexample, we have $\mathcal{H} \equiv_\mathrm{od}\mathcal{M}$. \cref{theorem:consistent_and_minimal} states that $\mathcal{H}$ is minimal and $\mathcal{M} = \can(M)$ is consistent with $T$, therefore it follows from \cref{lem:isomorphic} that $\mathcal{H}$ is isomorphic to $\mathcal{M}$ the canonical \gls*{MDP} modelling the \SUL{}.

\end{full}

\section{Learning MDPs by Sampling}
\label{sec:method_non_exact}

\begin{full}
In this section, we introduce \LstarMdp{}, an approximate sampling-based learning method for \glspl*{MDP} based
on \LstarMdpE{}. In contrast to \LstarMdpE{}, which requires exact information, we place 
weaker assumptions on the teacher. Here, we do not require exact \emph{output distribution} queries and \emph{equivalence} queries,
but we approximate these queries via sampling, i.e. testing. Since large amounts of data are required to produce accurate models, we also alter the learning 
algorithm structure in contrast to the previous section. The sampling-based 
\LstarMdp{} allows to derive an approximate model at any time, unlike most other $L^*$-based algorithms.
Therefore, this section is split into three parts: first, we present a sampling-based interface between teacher and learner, as well as the 
interface between teacher and \gls*{SUL}. The second and third part describe the adapted learner and the implementation of the teacher, respectively.
\end{full}

\begin{conference}
 The sampling-based \LstarMdp{} is based on \LstarMdpE{}, but samples \SUL{} traces instead of posing exact queries. Distribution
 comparisons are consequently approximated through statistical tests.
 While using similar data structures, \LstarMdp{} has a slightly different algorithm structure allowing to stop before reaching exact equivalence.
\end{conference}

\subsubsection{Queries.}

The sampling-based teacher maintains a multiset of traces $\mathcal{S}$ for the estimation of output distributions that grows during learning.
It offers an equivalence query and three queries relating to output distributions and samples $\mathcal{S}$.

\begin{compactitem}

\item \textsl{frequency ($\mathbf{fq}$)}:
given a test sequence $s \in \TestSeq$, $\fq(s) : \Out \to \mathbb{N}_0$ are output frequencies observed after $s$, where $\fq(s)(o) = \samples(s \cdot o)$ for $o\in\Out$.

\item \textsl{complete ($\mathbf{cq}$)}: given a test sequence $s \in \TestSeq$, $\cq(s)$ returns $\true$ if sufficient information is available
to estimate an output distribution from $\fq(s)$; returns $\false$ otherwise. %\todo{motivate complete query}

 \item \textsl{refine ($\rsq$)}:
instructs the teacher to refine its knowledge of the \SUL{} by testing it directed towards rarely observed samples.
Traces sampled by $\rsq$ are added to $\mathcal{S}$, increasing the accuracy of subsequent probability estimations.
 
 \item \textsl{equivalence ($\mathbf{eq}$)}: given a hypothesis $\mathcal{H}$, $\eq$ tests for output-distribution equivalence
 between the \gls*{SUL} and $\mathcal{H}$;
 returns a counterexample from $\TestSeq$ showing non-equivalence, or returns \emph{none} if no counterexample was found. 
\end{compactitem}
\begin{conference}
To implement these queries, we require the ability to reset the \SUL{}, to perform a single input on the \SUL{}
and to observe the \SUL{} output.
\end{conference}

\begin{full}
The sampling-based teacher thus needs to implement two different testing strategies, one for increasing accuracy of probability estimations
along observed traces (\emph{refine}) and one for finding discrepancies between a hypothesis and the \gls*{SUL} (\emph{equivalence}).
The \emph{frequency} query and the \emph{complete} query are used for hypothesis construction by the learner. 

To test the \SUL, we require the ability to (1) reset it and to (2) perform an input action and observe the produced 
output. 
For the remainder of this section, let $\mathcal{M} = \langle  Q,\In, \Out,q_0, \delta, L \rangle$ be the \gls*{MDP} 
underlying the \SUL{} with semantics $M$. 
Based on $q \in Q$, the current execution state of $\mathcal{M}$, we define two operations available to the teacher:
\begin{description}
 \item [reset] resets $\mathcal{M}$ to the initial state, i.e. $q = q_0$, and returns $L(q_0)$.
 \item [step] takes an input $i \in \In$ and selects a new state $q'$ according to $\delta(q,i)(q')$. 
 The \emph{step} operation then updates the execution state to $q'$ and returns $L(q')$. 
\end{description}
Note that we consider $\mathcal{M}$ to be a black box, i.e. its structure and transition probabilities are assumed to be unknown.
We are only able to perform inputs and observe output labels, e.g., we observe the initial \SUL{} output $L(q_0)$ after performing 
a \textbf{reset}.
\end{full}
% we have access to $L(q_0)$ the initial output of the \SUL through the reset operation. 

\subsection{Learner Implementation}
\label{sec:learner_impl}

\subsubsection{Observation Table.}
\begin{full}
 The sampling-based learner is also based on observation tables, therefore we use the 
same terminology as in \cref{sec:method_exact}.
\end{full}

\begin{conference}
\LstarMdp{} also uses observation tables. They carry similar information as in \cref{sec:method_exact}, 
but instead of output distributions in $\Dist(\Out)$,
 we store integral output frequencies  $(\Out \rightarrow \mathbb{N}_0)$, from which we estimate distributions.
\end{conference}

\begin{definition}[Sampling-based Observation Table]
 An observation table is a tuple $\langle S,E,\Ta\rangle$, consisting of a prefix-closed set of traces $S \subset \Traces$,
 a suffix-closed set of continuation sequences $E \subset \ContSeq$, and a mapping 
 $\Ta : (S \cup \LongTr(S)) \cdot E \to (\Out \rightarrow \mathbb{N}_0)$, where 
 $\LongTr(S) = \{s \cdot i \cdot o\, |\, s \in S, i \in \In, o \in \Out: \fq(s \cdot i)(o) > 0\}$. 
\end{definition}
\begin{full}
An observation table can be represented by a two-dimensional array, 
containing rows labelled with elements of $S$ and $\LongTr(S)$ and columns labelled by $E$.
Each table cell corresponds 
to a sequence $c = s \cdot e$, where $s \in S \cup \LongTr(S)$ is the row label of the cell and $e \in E$ is the
column label. It stores queried output frequency counts $\Ta(c)= \fq(c)$.
To represent the content of rows, we define the function $\row$ on $S \cup \LongTr(S)$ by $\row(s)(e) = \Ta(s \cdot e)$. 
The traces in $\LongTr(S)$ are input-output-extensions of $S$ which have been observed so far. 
We refer to traces in $S$/$\LongTr(S)$ as short/long traces. Analogously,
we refer to rows labelled by corresponding traces as short and long rows.

As in \cref{sec:method_exact}, we identify states with traces reaching these states.
These traces are stored in the prefix-closed set $S$. 
We distinguish states by their future behaviour in response to sequences %, as is common in active automata learning~\cite{DBLP:conf/sfm/SteffenHM11}. 
in $E$.
We initially set $S = \{L(q_0)\}$, where $L(q_0)$ is the initial output of the \gls*{SUL},  
and $E = \In$. 
Long traces, as extensions of access sequences in $S$, serve to define transitions of
hypotheses.
\end{full}

\subsubsection{Hypothesis Construction.}
\begin{full}
 As in \cref{sec:method_exact}, observation tables need to be closed and consistent for a hypothesis to be constructed. 
Unlike before, we do not have exact information to determine equivalence of rows. We need to  statistically test if rows are different. 
First, we give a condition determining whether two sequences lead to statistically different observations, i.e. the corresponding output frequency samples come from different distributions. This condition is based on Hoeffding bounds which are also used by Carrasco and Oncina~\cite{DBLP:journals/ita/CarrascoO99}. 
We further apply this condition in a check for approximate equivalence between cells and extend this check to rows. 
Using similar terminology to \cite{DBLP:journals/ita/CarrascoO99}, we refer to such checks as compatibility checks and 
we say that two cells/rows are compatible if we determine that they are not statistically different.
These notions of compatibility serve as the basis for slightly adapted definitions of closedness and consistency.
\end{full}

\begin{conference}
  As in \cref{sec:method_exact}, observation tables need to be closed and consistent for a hypothesis to be constructed. 
Here, we test statistically if cells and rows are approximately equal, referred to as compatible. 
The statistical tests applied in \cref{def:different} are based on Hoeffding bounds, as in~\cite{DBLP:journals/ita/CarrascoO99}. 
\cref{def:compatibility} serves as basis for adapted notions of closedness and consistency.
\end{conference}

\begin{definition}[Different]
\label{def:different}
  Two sequences $s$ and $s'$ in $\TestSeq$ produce statistically different output distributions with respect to $f : \TestSeq \rightarrow (\Out \to \mathbb{N}_0)$, denoted $\diff_\textit{f}(s,s')$, iff (1) $\complete(s) \land \complete(s') \land n_1 > 0 \land n_2>0$
  %\footnote{A similar condition is also used in the compatibility checks in \cite{DBLP:journals/ita/CarrascoO99}, which assumes no difference if there are no observations for  either $s$ or $s'$.},  
  where $n_1 = \sum_{o \in \Out} f(s)(o)$, $n_2 = \sum_{o \in \Out} f(s')(o)$, and (2) one of the following conditions holds: 
  \begin{compactenum}[2a.]%[label=2\alph*.]
  \item $\exists o \in \Out : \lnot (f(s)(o) > 0 \Leftrightarrow f(s')(o) > 0)$, or
  \item $\exists o \in \Out: \left| \frac{f(s)(o)}{n_1} - \frac{f(s')(o)}{n_2} \right| > \left(\sqrt{\frac{1}{n_1}}+\sqrt{\frac{1}{n_2}}\right) \cdot \sqrt{\frac{1}{2} \ln \frac{2}{\alpha}}$, where $\alpha$ specifies the confidence level $(1-\alpha)^2$ for testing each $o$ separately based on a Hoeffding bound~\cite{DBLP:journals/ita/CarrascoO99,10.2307/2282952}. % for each $o$ separately
 \end{compactenum}
 
\end{definition}

\begin{definition}[Compatible]
\label{def:compatibility}
 Two cells labelled by $c = s \cdot e$ and $c' = s' \cdot e'$ are compatible, denoted $\compatible(c,c')$, iff $\lnot \diff_{\Ta}(c,c')$.
Two rows labelled by $s$ and $s'$ are compatible, denoted $\compatible_E(s,s')$
iff $\lastOut(s) = \lastOut(s')$ and the cells corresponding to all $e \in E$ are compatible, i.e. $\compatible(s \cdot e,s' \cdot e)$.
\end{definition}

\subsubsection{Compatibility Classes.}

In Sect.~\ref{sec:method_exact}, we formed equivalence classes of traces with respect to $\equivrow_E$ creating one hypothesis state per equivalence class. 
Now we partition rows labelled by $S$ based on compatibility. Compatibility given by \cref{def:compatibility}, however, is not an equivalence relation, as it is not transitive in general. 
% For three traces $s$, $t$ and $u$, we may have $\compatible_E(s,u) = \true$ and $\compatible_E(t,u)= \true$, but $\compatible_E(s,t) = \false$. This may happen if we have too few samples for $u$, because $\diff_{\Ta}(e,e')$ returns $\false$ if $\lnot (\complete(e) \land \complete(e'))$. 
As a result, we cannot simply create equivalence classes. We apply the heuristic implemented by \cref{alg:create_comp_classes} to partition $S$. 
%   \begin{wrapfigure}[11]{L}{0.62\textwidth}
%     \begin{minipage}{0.62\textwidth}
\begin{algorithm}[t]
 \begin{algorithmic}[1]
  \ForAll{$s \in S$} %\Comment{rank by number of observations and then by inverse length}
  \State $\mathrm{rank}(s)\gets \sum_{i\in \In} \sum_{o\in \Out} \Ta(s\cdot i)(o)$ %\langle\sum_{i\in \In} \sum_{o\in \Out} \Ta(s\cdot i)(o), -|s|\rangle$
  \EndFor
 \State $\textit{unpartitioned} \gets S$, %$S^o \gets sort(S)$ \Comment{sort by descending $\mathrm{rank}$}
  $R \gets \emptyset$
%  \Repeat
\While{$\textit{unpartitioned} \neq \emptyset$}
 \State $r \gets m$ where $m \in \textit{unpartitioned}$ with largest $\mathrm{rank}(m)$
 \State $R \gets R \cup \{r\}$
 \State $\cg(r) \gets \{s \in \textit{unpartitioned}\, |\, \compatible_E(s,r)\}$
 \ForAll{$s \in \cg(r)$}
 \State $\repr{s} \gets r$ \label{algline:create_comp_classes:rep_notation}
 \EndFor
 \State $\textit{unpartitioned} \gets \textit{unpartitioned}~ \setminus \cg(r)$
 \EndWhile
%  \Until{$\textit{unpartitioned} = \emptyset$}
%   \ForAll{$l \in \LongTr(S)$}
%  \State $[l] \gets r$ where $r \in R: \compatible_E(r,l)$ with largest $\mathrm{rank}(r)$ \label{algline:create_comp_classes:long_rep} 
%  \EndFor
 \end{algorithmic}
\caption{Creating compatibility classes}
\label{alg:create_comp_classes}
\end{algorithm}
%     \end{minipage}
% \end{wrapfigure}

First, we assign a rank to each trace in $S$. % we sort $S$ according to the number of available observations for each $s$ in $S$. 
Then, we partition $S$ by iteratively selecting the trace $r$ with the largest rank  
and computing a \emph{compatibility class} $\cg(r)$ for $r$. The trace $r$ is the (canonical) representative for $s$ in $\cg(r)$, which we
denote by $\repr{s}$ (Line~\ref{algline:create_comp_classes:rep_notation}). Each $r$ is stored in the set of representative traces
$R$. 
In contrast to equivalence classes, elements in a compatibility class need not be pairwise compatible
and an $s$ may be compatible to multiple representatives, where the unique representative $\repr{s}$ of $s$ has the largest rank. 
\begin{full}
However, in the limit $\compatible_E$ based on Hoeffding bounds converges to an equivalence relation~\cite{DBLP:journals/ita/CarrascoO99}
and therefore compatibility classes are equivalence classes in the limit (see Sect.~\ref{sec:convergence}).
\end{full}

% Now we are able to define closedness and consistency of observation tables.
% \todo{say somewhere that we add extensions
% for $s \in S$ depending on the support of $\row(s)(e)$ for all $e \in \In$}
\begin{definition}[Sampling Closedness]
  An observation table $\langle S,E,\Ta\rangle$ is\linebreak closed if for all $l \in \LongTr(S)$ there is a representative $s \in R$ with
 $\compatible_E(l,s)$.
\end{definition}

\begin{definition}[Sampling Consistency]
 An observation table $\langle S,E,\Ta\rangle$ is consistent if for all compatible pairs of short traces $s,s'$ in $S$ 
 and all input-output pairs $i\cdot o \in \In \cdot \Out$, 
 we have that (1) at least one of their extensions has not been observed yet, i.e. $\Ta(s \cdot i)(o) = 0$ or $\Ta(s' \cdot i)(o) = 0$, or
 (2) both extensions are compatible, i.e. $\compatible_E(s \cdot i \cdot o,s' \cdot i \cdot o)$. 
\end{definition}
\begin{full}
Note that the first condition of consistency may be satisfied because of incomplete information. 
\end{full}
Given a closed and consistent observation table $\langle S,E,\Ta
\rangle$, we derive hypothesis \gls*{MDP} 
$\mathcal{H} = \mathrm{hyp}( S,E,\Ta )$
through the steps below. 
Note that extensions $s\cdot i \cdot o$ of $s$ in $S$ define transitions.
Some extensions may have few observations, i.e. $\Ta(s\cdot i)$ is low and $\complete(s \cdot i) = \false$. In case 
of such uncertainties, we add transitions to a special sink state labelled by $\mathrm{chaos}$, an output not in the original alphabet\footnote{This is inspired by 
the introduction of chaos states in $\mathbf{ioco}$-based learning~\cite{DBLP:journals/eceasst/VolpatoT15}.}.
A hypothesis is a tuple $\mathcal{H} = \langle Q_h,\In, \Out \cup \{\mathrm{chaos}\},{q_0}_h, \delta_h, L_h\rangle$ 
where:

\begin{compactitem}
 \item representatives for long traces $l \in \LongTr(S)$ are given by (see \cref{alg:create_comp_classes}): \\ 
 $\repr{l} = r$ where $r \in \{r' \in R \mathrel\vert \compatible_E(l,r')\}$ with largest $\mathrm{rank}(r)$
 \item $Q_h = \{\langle\lastOut(s),\row(s)\rangle \mathrel| s \in R \} \cup \{q_\mathrm{chaos}\}$, 
 \begin{compactitem}
\item for $q = \langle o,\row(s)\rangle \in Q_h \setminus \{q_\mathrm{chaos}\}$: $L_h(q) = o$
\item for $q_\mathrm{chaos}$: $L_h(q_\mathrm{chaos}) = \mathrm{chaos}$ and for all $i \in \In$:
  $\delta_h(q_\mathrm{chaos},i)(q_\mathrm{chaos}) = 1$
 \end{compactitem}
 \item ${q_0}_h = \langle L(q_0),\row(L(q_0))\rangle$
 \item for $q = \langle o, \row(s) \rangle \in Q_h \setminus \{q_\mathrm{chaos}\}$ and $i \in \In$ (note that $\In \subseteq E$):
 \begin{compactenum}
  \item If $\lnot \complete(s \cdot i)$:  
  %   \begin{enumerate}
    $\delta(q,i)(q_\mathrm{chaos}) = 1$, i.e. move to chaos 
%   \end{enumerate}
  \item Otherwise estimate a distribution $\mu = \delta_h(q,i)$ over the successor states: \\
%   \begin{enumerate}
    for $o \in \Out$ with $\Ta(s \cdot i)(o) > 0$: 
   $\mu(\langle o,\row(\repr{s \cdot i \cdot o})\rangle) = \frac{\Ta(s \cdot i)(o)}{\sum_{o'\in \Out} \Ta(s \cdot i)(o')}$
%   \end{enumerate}
 \end{compactenum}
\end{compactitem}

%  i.e. we create one state for each representative and an additional chaos 
%  state to denote that we have insufficient information for parts of the state space.

\subsubsection{Updating the Observation Table.}

% \paragraph{Establishing Closedness and Consistency.}
Analogously to \cref{sec:method_exact}, we make observation tables closed by adding new short 
rows and we establish consistency by adding new columns. 
% \begin{full}
% We again refer to the combined 
% operation as $\textsc{MakeClosedAndConsistent}$. 
% While \cref{alg:exact_l_star_mdp} needs to fill the observation 
% table after each call of $\textsc{MakeClosedAndConsistent}$, this is not required in the sampling-based setting due to 
% the adapted notions of closedness and consistency.
% \end{full}
% \begin{conference}
While \cref{alg:exact_l_star_mdp} needs to fill the observation 
table after executing $\textsc{MakeClosedAndConsistent}$, this is not required in the sampling-based setting due to 
the adapted notions of closedness and consistency.
% \end{conference}

\emph{Trimming the Observation Table.}
 Observation table size greatly affects learning performance, therefore it is common to avoid adding redundant information~\cite{DBLP:journals/iandc/RivestS93,DBLP:conf/rv/IsbernerHS14}.
Due to inexact information, this is hard to apply in a stochastic setting.
We instead remove rows via a function \textsc{Trim}, once we are certain that this does not change the hypothesis. 
\begin{full}
Given an observation table $\langle S, E, \Ta\rangle$, we remove $s$ and 
all $s'$  such that $s \ll s'$ from $S$ if:
\begin{enumerate}
 \item there is exactly one $r\in R$ such that $\compatible_E(s,r)$ %the row of $s$ is compatible to exactly one representative in $R$
 \item $s \notin R$ and $\forall r \in R: \lnot (s \ll r)$ %, i.e. $s$ and none of its extensions are short row representatives
 \item and $\forall s' \in S, i \in \In$, with $s\ll s'$: $\diff_\fq(s' \cdot i,r\cdot i) = \false$, where $r \in R$ such that $\langle \lastOut(r), \row(r)\rangle = \delta_h^*(r) = \delta_h^*(s')$, and $\delta_h$ is the transition relation of $\mathrm{hyp}(S,E,\Ta)$.
\end{enumerate}
The first condition is motivated by the observation that if $s$ is compatible to exactly one $r$, 
then all extensions of $s$ can be assumed to reach the same states as the extensions of $r$, i.e. we do not need to store $s$ in the observation table. 
The other conditions make sure that we do not remove required rows, because of a spurious compatibility check
in the first condition. The third condition is related to the implementation of equivalence queries and basically checks if an extension $s'$ reveals a difference
between observed frequencies (queried via $\fq$) and frequencies used for hypothesis construction.
Note that removed rows do not affect hypothesis construction. 
\end{full}
\begin{conference}
We remove rows that are (1) not prefixes of representatives $r \in R$, (2) that are compatible to exactly one $r \in R$, and (3) that are not prefixes of counterexamples to equivalence between \SUL{} and hypothesis. 
\end{conference}

\subsubsection{Learning Algorithm.}
Algorithm \ref{alg:overview} implements \LstarMdp{}. 
% The main data structure is the , storing information about observed system traces. 
It first initialises an observation table $\langle S,E,\Ta\rangle$ with the initial \SUL{} output as first row 
and with the inputs $\In$ as columns (Line \ref{algline:non_exact:init}). 
Lines \ref{algline:non_exact:initmq_1} to \ref{algline:non_exact:initmq_2} perform a refine query and then update $\langle S,E,\Ta\rangle$,
which corresponds to output distribution queries in \LstarMdpE{}. Here, the teacher resamples the only known trace $L(q_0)$.
\begin{full}
Resampling that trace consists of observing $L(q_0)$, performing some input and observing another output. 
\end{full}

After that, we perform Lines \ref{algline:non_exact:repeat_1} to \ref{algline:non_exact:repeat_2} until a stopping criterion is reached. We establish 
closedness and consistency of $\langle S,E,\Ta\rangle$ in Line \ref{algline:non_exact:make_closed_and_consistent} to 
build a hypothesis $\mathcal{H}$ in Line~\ref{algline:non_exact:hyp}. 
After that, we remove redundant rows of the observation table via \textsc{Trim} in Line \ref{algline:non_exact:trim}.
Then, we perform an equivalence query, testing for equivalence between \SUL{} and $\mathcal{H}$.
If we find a counterexample, we add all its prefix traces as rows to the observation table like in \LstarMdpE{}.
Finally, we sample new system traces via $\rsq$ to gain more accurate information about the \SUL{} (Lines \ref{algline:non_exact:mq_1} to \ref{algline:non_exact:mq_2}).
Once we stop, we output the final hypothesis.

\begin{algorithm}[t]
 \begin{algorithmic}[1]
  \Require sampling-based teacher capable of answering $\fq, \rsq, \eq$ and $\cq$
  \State $S \gets \{L(q_0)\}$, $E \gets \In$, $\Ta \gets \{\}$ \Comment{initialise observation table} \label{algline:non_exact:init}
  \State perform $\rsq(\langle S,E,\Ta\rangle)$ \label{algline:non_exact:initmq_1} \Comment{sample traces for initial observation table}
  \ForAll{$s \in S \cup \LongTr(S), e \in E$}
  \State $\Ta(s \cdot e) \gets \fq(s \cdot e)$ \Comment{update observation table with frequency information}
  \EndFor \label{algline:non_exact:initmq_2}
  \State $\mathit{round} \gets 0$
  \Repeat \label{algline:non_exact:repeat_1}
  \State $\mathit{round} \gets \mathit{round} + 1$
%   \State $\langle S, E,\Ta \rangle \gets \Call{MakeClosedAndConsistent}{\langle S, E,\Ta \rangle}$ \label{algline:non_exact:make_closed_and_consistent}
    \While{$\langle S, E, \Ta\rangle$ not closed or not consistent}
    \State $\langle S,E,\Ta \rangle \gets \Call{MakeClosedAndConsistent}{\langle S,E,\Ta \rangle}$\label{algline:non_exact:make_closed_and_consistent}
    \EndWhile
  \State $\mathcal{H} \gets \mathrm{hyp}( S,E,\Ta )$ \label{algline:non_exact:hyp} \Comment{create hypothesis}
  \State $\langle S, E,\Ta \rangle \gets \Call{trim}{\langle S, E,\Ta \rangle, \mathcal{H}}$ \label{algline:non_exact:trim} \Comment{remove rows that are not needed}
  %\Comment{TODO: third trimming condition requires hypothesis}
  \State $\mathit{cex} \gets \eq({\mathcal{H}})$
   \If{$\mathit{cex} \neq \textit{none}$} \Comment{we found a counterexample}
      \ForAll{$t \cdot i \in \mathit{prefixes}(\mathit{cex})$ with $i \in \In$}
        \State $S \gets S \cup \{t\}$ \Comment{add all prefixes of the counterexample}
      \EndFor
    \EndIf
    \State perform $\rsq(\langle S,E,\Ta\rangle)$ \label{algline:non_exact:mq_1} \Comment{sample traces to refine knowledge about \SUL{}}
  \ForAll{$s \in S \cup \LongTr(S), e \in E$}
  \State $\Ta(s \cdot e) \gets \fq(s \cdot e)$ \Comment{update observation table with frequency information}
  \EndFor \label{algline:non_exact:mq_2}
  \Until{\Call{stop}{$\langle S, E,\Ta \rangle$, $\mathcal{H}$, $\mathit{round}$}} \label{algline:non_exact:repeat_2} 
  \State \Return $\mathrm{hyp}(S,E,\Ta)$ \Comment{output final hypothesis}
 \end{algorithmic}
\caption{The main algorithm implementing \LstarMdp{}}
\label{alg:overview}
\end{algorithm}

% \subsection{Stopping}

% Currently, there is simply an upper bound on the number of rounds to determine whether learning should stop. 
% Generally, we could use a similar approach as in deterministic active automata learning. Once a hypothesis
% has been thoroughly tested without revealing any counterexample, learning is stopped. 
% Thus, we could require that (1) testing must not reveal new counterexamples, (2) there must not be any 
% counterexample to compatibility between observation tree and hypothesis, and (3) additionally there must not be a transition
% to the chaos state. The third condition could also be that the probability of reaching the chaos state must be below some threshhold.  

\begin{full}
 
\paragraph{Stopping.}
\LstarMdpE{} and deterministic automata learning usually
stop learning once equivalence between the learned hypothesis and the \SUL{} is achieved, i.e. no counterexample can be found. 
Here, we employ a different stopping criterion, because equivalence can hardly be achieved via sampling. 
Furthermore, we may wish to carry on resampling via $\rsq$ although we did not find a counterexample. Resampling may improve accuracy of a hypothesis 
which is beneficial for the test-case generation in subsequent equivalence queries. 

Our stopping criterion takes uncertainty in compatibility checks into account. 
As previously noted, rows may be compatible to multiple other rows. In particular, 
a row labelled by $s$ may be compatible to multiple representatives, i.e. we are not certain 
which state is reached by the trace $s$. We address 
this issue by stopping based on the ratio $\runamb$ of unambiguous traces to all traces, which we compute by:
\begin{align*}
\runamb &= \frac{|\{s \in S \cup \LongTr(S): \mathrm{compRep}(s) = 1\}|}{|S \cup \LongTr(S)|} \text{ where } \\
\mathrm{compRep}(s) &= |\{r \in R : \compatible_E(s, r)\}|
\end{align*}

More concretely, we stop if:
\begin{enumerate}
 \item[1.a.] at least $\rmin$ rounds have been executed and 
 \item[1.b.] the chaos state $q_\mathrm{chaos}$ is unreachable and 
 \item[1.c.] and $\runamb \geq \tunamb$, where $\tunamb$ is a user-defined threshold, 
 \item[or]
 \item[2.a.] alternatively we stop after a maximum number of rounds $\rmax$.
\end{enumerate}

\end{full}

\begin{conference}
\emph{Stopping.}
 The exact learner \LstarMdpE{} stops upon reaching equivalence to the \SUL{}, i.e. once there is no counterexample. In the sampling-based setting, we may not find a counterexample due to inaccurate hypotheses. Our stopping criterion therefore takes uncertainty into account, which we quantify with $\runamb$, the relative number of (unambiguous) traces in $S\cup \LongTr(S)$ compatible to exactly one representative in $R$. Additionally, we check if the $\mathrm{chaos}$ state is reachable.
 
 Consequently, we stop when (1) $\runamb \geq \tunamb$ where $\tunamb$ is a user-defined threshold, (2) the chaos state is unreachable, and (3) at least $\rmin$ rounds have been executed. We also stop after a maximum of $\rmax$ rounds. 
\end{conference}

\subsection{Teacher Implementation}

\begin{conference}
 Due to space constraints, we discuss each query only briefly. An accurate description can be found in the full technical report~\cite{lstar_mdp_tech_report}.
 \begin{compactitem}
\item \textsl{frequency ($\mathbf{fq}$)}: returns output frequencies observed in the sampled traces $\samples$.
\item \textsl{complete ($\mathbf{cq}$)}:
  complete queries are based on threshold $\ncomplete$. We consider test sequences complete that have been sampled at least $\ncomplete$ times.
  \item \textsl{refine ($\rsq$)}:
  refine queries take an observation table $\langle S,E,\Ta\rangle$ and resample incomplete sequences in $(S \cup \LongTr(S)) \cdot E$.
  The parameter $\nbatch$ defines how often we resample. 
  \item \textsl{equivalence ($\mathbf{eq}$)}:
  we apply two strategies for equivalence queries. \emph{First}, we test for structural equivalence between
  hypothesis $\mathcal{H}$ and \SUL{}. The testing strategy inspired by~\cite{Aichernig2018} 
performs random walks on $\mathcal{H}$ and has three parameters:
  $\ntest$, the maximum number of tests, $\pstop$, the stop probability, and $\prand$, the probability of choosing inputs uniformly at random.
  \emph{Second}, we check for conformance between the collected samples $\samples$ and $\mathcal{H}$ via $\diff_{\fq}$.
 \end{compactitem}
  Note that we return no counterexample if trivial counterexamples containing $\mathrm{chaos}$ are observable in the hypothesis. This prompts \LstarMdp{} to issue further \emph{refine} queries, causing the chaos state to be unreachable eventually. Otherwise, the observation table might grow unnecessarily which is detrimental to performance. 
\end{conference}

\begin{full}
 
In the following, we describe the implementation of each of the four queries provided by the teacher. 
Recall that we interact with the \SUL{} $\mathcal{M}$ with semantics $M$ (see Sect.~\ref{sec:method_exact}). 
% In Sect.~\ref{sec:method_exact}, we assumed the existence of a teacher, which is
% capable of answering two types of queries: output distribution queries (ODQs) and equivalence queries (EQs).
% To implement $L^*_{MDP}$, we approximate these queries via directed sampling of 
% traces of the \SUL, i.e. by testing the testing the \SUL.
% To approximate ODQs, our testing goal is to resample already observed traces
% to gain more information about output probabilities along these traces. For EQs, we
% perform randomised conformance testing based on the current hypothesis to detect that the 
% hypothesis is not equivalent to \SUL. 

\subsubsection{Frequency Query.}
The teacher keeps track of a multiset of sampled system traces $\mathcal{S}$. Whenever a new a trace is added,
all its prefixes are added as well, as they have also been observed. Therefore, we have for 
$t \in \Traces, t' \in \prefixes(t): \mathcal{S}(t) \leq \mathcal{S}(t')$. 
The frequency query $\fq(s) : \Out \to \mathbb{N}_0$
for $s\in \TestSeq$ returns output frequencies observed after $s$:
\begin{equation*}
 \forall o \in \Out: \fq(s)(o) = \mathcal{S}(s \cdot o)
\end{equation*}

\subsubsection{Complete Query.}
Trace frequencies retrieved via $\fq$ are generally used to compute empirical output distributions $\mu$
following a sequence $s$ in $\TestSeq$, i.e. the learner computes 
$\mu(o) = \frac{\fq(s)(o)}{\sum_{o'\in \Out} \fq(s)(o')}$ to approximate $M(s)(o)$.
The \emph{complete} query $\complete$ takes a sequence $s$ as input and signals whether $s$
should be used to approximate $M(s)$, e.g. to perform statistical tests\footnote{This query serves a similar role as in \cite{DBLP:journals/eceasst/VolpatoT15}.}. 
% In practice, we cannot draw sufficiently many samples from $M(s)$ to approximate this distribution
% with high confidence for all $s$. 
% Instead, we employ a predicate $\complete(t)$ based on a user-defined threshhold $\ncomplete$, which determines
% whether $f(t)(o)$ is used in statistical tests.
We base $\complete$ on a threshold $\ncomplete > 0$ by defining:
\begin{equation*}
 \complete(s) = \begin{cases}
                 \true & \text{ if } \sum_{o\in \Out} \mathcal{S}(s\cdot o) \geq \ncomplete \\
                 \true & \text{ if } \exists s',o,i: \text{ s.t. } s'\cdot o \cdot i\ll s \land \complete(s') \land \mathcal{S}(s \cdot o)=0 \\
                 \false & \text{ otherwise }
                \end{cases}
\end{equation*}

% We extend $\complete$ to traces $s \cdot o$ with $s\in \TestSeq$ and $o \in \Out$ by $\cq(t) = \cq(s)$. 
Note that for a complete $s$, all prefixes of $s$ are also complete. Additionally,
if $\complete(s)$, we assume that we have seen all extensions of $s$; therefore, we
we set for each $o$ with $\mathcal{S}(s \cdot o) = 0$ all extensions of $s \cdot o$ to be complete (second clause).
The threshold $\ncomplete$ is user-specifiable in our implementation. %\todo{provide recommendation for setting $\ncomplete$}

\subsubsection{Refine Query.}
\emph{Refine} queries serve the purpose of refining our knowledge about output distributions along previously 
observed traces. 
% , i.e.
% our goal is to gain accurate information about the distribution of previously observed traces. 
Therefore, we select rarely observed traces and resample them. We implemented this through the procedure
outlined in Algorithm~\ref{alg:refine}.

First, we build a trie from rarely observed traces (Lines \ref{algline:rs:rare} and \ref{algline:rs:trie}),
where edges are labelled by input-output pairs and nodes are labelled by traces reaching the nodes. 
This trie is then used for directed online-testing of the \SUL{} via \textsc{sampleSul} (Lines \ref{algline:rs:sampleSul_1} to
\ref{algline:rs:sampleSul_2}) with the
goal of reaching a leaf of the trie. 
In this way, we create $\nbatch$ new samples and add them to the 
multiset of samples $\mathcal{S}$.

\begin{algorithm}[t]
\begin{algorithmic}[1]
 \State $\mathit{rare} \gets \{s \mathrel\vert s \in (S \cup \LongTr(S))\cdot E : \lnot \complete(s) \}$ \Comment{select incomplete sequences} \label{algline:rs:rare}
 \State $\mathit{trie} \gets \Call{buildTrie}{\mathit{rare}}$ \label{algline:rs:trie}
 \For{$i \gets 1 \textbf{ to }\nbatch$} \Comment{collect $\nbatch$ new samples}
 \State $\mathit{newTrace} \gets \Call{sampleSul}{\mathit{trie}}$
 \State $\mathcal{S} \gets \mathcal{S} \uplus \{\mathit{newTrace}\}$
 \EndFor
\Function{sampleSul}{$\mathit{trie}$} \label{algline:rs:sampleSul_1}
\State $\mathit{node} \gets \mathit{root}(\mathit{trie})$
\State $\mathit{trace} \gets \mathbf{reset}$ \Comment{initialise \SUL{} and observe initial output}
\Loop
\State $\mathit{input} \gets \randSel(\{i \in \In|\exists o \in \Out, \mathit{n}: \mathit{node}\xrightarrow{i,o} \mathit{n}\})$ \Comment{random input}
\State $\mathit{output} \gets \mathbf{step}(i)$ \Comment{execute \SUL{} and observe output}
\State $\mathit{trace} \gets \mathit{trace} \cdot i \cdot o$
\If{$\mathit{trace} \notin \mathit{trie}$ or $\mathit{trace}$ labels leaf} \Comment{did we leave the trie?}
\State \Return $\mathit{trace}$
\EndIf
\State $\mathit{node'} \gets n$ with $\mathit{node}\xrightarrow{i,o} \mathit{n}$ \Comment{move in trie}
\State $\mathit{node} \gets \mathit{node'}$
\EndLoop
\EndFunction \label{algline:rs:sampleSul_2}
 \end{algorithmic}
\caption{\emph{Refine} query}
\label{alg:refine}
\end{algorithm}

\subsubsection{Equivalence Query.}

Equivalence queries are often implemented via (conformance) testing in active automata~\cite{DBLP:conf/dagstuhl/HowarS16}, 
% \todo{reference may be suboptimal, as they explicitly say that the goal is to find counterexamples rather than to prove conformance}
e.g., via the W-method~\cite{Chow1978} method for deterministic models.
Such testing techniques generally execute some test suite to find counterexamples to conformance 
between a model and the \SUL. In our setup, a counterexample is a test sequence inducing a
different output distribution in the hypothesis $\mathcal{H}$ than in the \SUL. Since we cannot directly observe those
distributions, we apply two strategies to find counterexamples during equivalence queries. 
First, we search for counterexamples with respect to the structure of $\mathcal{H}$ via testing. 
Second, we check for statistical conformance between all traces $\mathcal{S}$ collected so far and $\mathcal{H}$, which 
allows us to detect incorrect output distributions.

Note that all traces to the %\emph{chaos} 
state $q_\mathrm{chaos}$ are guaranteed to be counterexamples, as $\mathrm{chaos}$
is not part of the original output alphabet $\Out$. For this reason, we do not search for other counterexamples
if $q_\mathrm{chaos}$ is reachable in $\mathcal{H}$. In slight abuse of terminology, we implement this by returning $\mathit{none}$ from $\eq(\mathcal{H})$. 
\LstarMdp{} in \cref{alg:overview} will then issue further $\rsq$ queries, lowering uncertainty about state transitions, which
in turn causes $q_\mathrm{chaos}$ to be unreachable eventually.

\paragraph{Testing of Structure.}
Our goal in testing is to sample a trace of the \SUL{} that is not observable on the hypothesis.
For that, we adapted a randomised testing strategy from Mealy machines to 
\glspl*{MDP}, which proved effective in previous work~\cite{Aichernig2018}.
In this work, we generated test cases for active automata learning by interleaving random walks in hypotheses
with paths leading to randomly chosen transitions. 
By generating many of these tests, we aim at covering 
hypotheses adequately, while exploring new parts of \SUL's state space through random testing. 
Here, we aim at covering randomly chosen states and 
apply an online testing procedure, as the \SUL{} is stochastic. This procedure 
is outlined in Algorithm~\ref{alg:state_coverage_testing}.

\begin{algorithm}[t]
 \begin{algorithmic}[1]
 \Require %$\begin{array}{l}
$\mathcal{H} = \langle Q,\In, \Out,q_0, \delta, L\rangle$, schedulers $\mathit{qSched}$
          %\end{array}$
 \Ensure counterexample test sequence $s \in \TestSeq$ or $\mathit{none}$
%  \State 
%  \For{$i \gets 1 \mathbf{\ to\ } \ntest$}
    \State $q_\mathrm{curr} \gets q_0$  \Comment{current state}
    \State $\mathit{trace} \gets \mathbf{reset}$ 
    \State $q_\mathrm{target} \gets \randSel(\mathit{reachable}(Q,q_\mathrm{curr}))$ \label{algline:test:first_target} \Comment{choose a target state} 
    \Loop
    \If{$\coinFlip(\prand)$}
    \State $\mathit{in} \gets \randSel(\In)$  \label{algline:test:rand_in} \Comment{random next input}
    \Else 
    \State $\mathit{in} \gets \mathit{qSched}(q_\mathrm{target})$  \label{algline:test:target_in} \Comment{next input leads towards target}
    \EndIf
    \State $\mathit{out} \gets \mathbf{step}(\mathit{in})$ \Comment{perform input}
    \State $q_\mathrm{curr} \gets \Delta(q_\mathrm{curr}, \mathit{in} \cdot \mathit{out})$ \Comment{move in hypothesis}
    \If{$q_\mathrm{curr} = \bot$} \Comment{output not possible in hypothesis}
    \State \Return $\mathit{trace} \cdot \mathit{in}$ \label{algline:test:return_cex} \Comment{return counterexample}
    \EndIf
    \State $\mathit{trace} \gets \mathit{trace} \cdot \mathit{in} \cdot \mathit{out}$
    \If{$\coinFlip(\pstop)$} \Comment{stop with probability $\pstop$}
    \State \Return $\mathit{none}$ \label{algline:test:return_none}
    \EndIf
    \If{$q_\mathrm{curr} = q_\mathrm{target}$ or $q_\mathrm{target} \notin \mathit{reachable}(Q,q_\mathrm{curr})$}
    \State $q_\mathrm{target} \gets \randSel(\mathit{reachable}(Q,q_\mathrm{curr}))$ \Comment{choose new target} \label{algline:test:choose_new_target}
    \EndIf
    \EndLoop
%  \EndFor
 \end{algorithmic}
 \caption{State-coverage-based testing for counterexample detection}
 \label{alg:state_coverage_testing}
\end{algorithm}

The algorithm takes a hypothesis and \textit{qSched} as input, where \textit{qSched} is a mapping from states to schedulers. 
Given $q \in Q$, $\textit{qSched}(q)$ is a scheduler maximising the probability of reaching $q$, i.e. it selects inputs optimally with respect to reachability of $q$. 
For optimal reachability, there exist schedulers that are memoryless and deterministic~\cite{DBLP:conf/sfm/ForejtKNP11}, which means that they take only the last state in the current execution path into account and that input choices are not probabilistic. Therefore, a scheduler $\textit{qSched}(q)$ is a function $s : Q \to \In$.
In \cref{alg:state_coverage_testing}, we start by randomly choosing a target state $q_\mathrm{target}$ from the states
reachable from the initial state (Line~\ref{algline:test:first_target}), which are given by 
$\mathit{reachable}(Q,q_\mathrm{curr})$. Then, we execute the \SUL, either with random inputs (Line~\ref{algline:test:rand_in})
or with inputs leading to the target (Line~\ref{algline:test:target_in}), which are computed using schedulers. 
If we observe an output which is not possible in the hypothesis, we return a counterexample (Line~\ref{algline:test:return_cex}), 
alternatively we may stop with probability $\pstop$ (Line~\ref{algline:test:return_none}). If we reach the target or 
it becomes unreachable, we simply choose a new target state (Line~\ref{algline:test:choose_new_target}).

For each equivalence query, we repeat Algorithm~\ref{alg:state_coverage_testing} up to $\ntest$ times and report
the first counterexample we find. In case we find a counterexample $c$, we resample it up to $\nretest$ times or until $\complete(c)$, 
to get more accurate information about it. 

\paragraph{Checking Conformance to $\mathcal{S}$.}
 For each sequence $t \cdot i \in \TestSeq$ with $i \in \In$ such that $\complete(t \cdot i)$, we check
 for consistency between the information stored in $\samples$ and the hypothesis $\mathcal{H}$ by evaluating two conditions:
 \begin{enumerate}
  \item Is $t$ observable in $\mathcal{H}$? If it is not, then we determine the longest observable prefix $t'$
  of $t$ such that $t' \cdot i' \cdot v = t$, where $i'$ is a single input, and return $t' \cdot i'$ as counterexample from $\eq(\mathcal{H})$.
  \item Otherwise we determine $q = \langle o, \row(r)\rangle$ reached by $t$ in $\mathcal{H}$, where $r \in R$, and return $t \cdot i$ as counterexample 
  if $\diff_\fq(t \cdot i, r \cdot i)$ is true. 
  This statistical check approximates the comparison $M(t \cdot i) \neq M(r \cdot i)$, to check if $t \not\equiv_M r$. 
  Therefore, it checks implicitly $M(t \cdot i) \neq H(t\cdot i)$, as $t \equiv_H r$.
%   Hence, we check statistically if the output distribution in state $q$ is correct for input $i$ based on data from frequency queries. 
 \end{enumerate}

\end{full}

\begin{conference}
\paragraph{Convergence.}
We have examined convergence of the sampling-based \LstarMdp{} in the limit with respect to the following setup. We configure equivalence testing such that each input is chosen uniformly at random and the length of each test is geometrically distributed. This resembles the sampling regime assumed for \textsc{IoAlergia}~\cite{DBLP:journals/ml/MaoCJNLN16}. Likewise, we consider a data-dependent $\alpha_n = \frac{1}{n^r}$ with $r > 2$, where $n$ is the number of samples collected so far. Finally, we consider \LstarMdp{} without trimming of observation tables. Informally, letting the number of rounds and thus the sample size $n$ approach infinity, we eventually learn the correct~\gls*{MDP}.

\begin{theorem}
\label{theorem:convergence_non_exact}
\LstarMdp{} as configured above creates hypotheses $\mathcal{H}_n$ that are minimal in the number of states and output-distribution equivalent to the \SUL{} in the limit (see Theorem~4 and its proof in~\cite{lstar_mdp_tech_report}). %converges to as alternative , except for finitely many $n$.
\end{theorem}

\end{conference}

\begin{full}

\subsection{Convergence of \LstarMdp}
 
\label{sec:convergence}

In the following, we will show that the sampling-based \LstarMdp{} learns the correct \gls*{MDP}. Based on the notion of language identification in grammar inference~\cite{delaHiguera:2010:GIL:1830440}, we describe our goal as producing an \gls*{MDP} isomorphic to the canonical \gls*{MDP} modelling the \SUL{} with probability one in the limit. 
% It is common to assume that sampling distributions do not change during grammar inference~\cite{delaHiguera:2010:GIL:1830440}.
To show identification in the limit, we introduce slight simplifications. Firstly, we disable trimming of the observation table (see \cref{sec:learner_impl}), i.e. we do not remove rows. Second, we set $\prand = 1$ for equivalence testing and we do not stop at the first detected difference between \SUL{} and hypothesis, but solely based on a $\pstop < 1$; i.e. all input choices are uniformly randomly and the length of each test is geometrically distributed with $\pstop$. 
This is motivated by the common assumption that sampling distributions do not change during learning~\cite{delaHiguera:2010:GIL:1830440}. 
Third, we change the function $\mathrm{rank}$ in \cref{alg:create_comp_classes} to assign ranks based on a lexicographic ordering of traces instead of a rank based on observed frequencies, such that the trace consisting only of the initial \SUL{} output has the largest rank. We actually implemented both types of
$\mathrm{rank}$ functions and found that the frequency-based function led to better accuracy, but would require more complex proofs. 
We let the number of samples for learning approach infinity, therefore we do not use a stopping criterion. Finally, we concretely instantiate $\complete$ by setting $\ncomplete = 1$, since $\ncomplete$ is only relevant for applications in practice. 
% As our compatibility criterion between rows is based on Hoeffding bounds like the compatibility criterion applied by Carrasco and Oncina~\cite{DBLP:journals/ita/CarrascoO99}, we use similar arguments for convergence and we rely on the convergence of the exact learning algorithm \LstarMdpE{} discussed in \cref{sec:exact_convergence}.

\subsubsection{Proof Structure.}

We show convergence in two major steps: (1) we show that the hypothesis structure derived from a sampling-based observation table converges to the
hypothesis structure derived from the corresponding observation table with exact information.
(2) Then, we show that if counterexamples exist, we will eventually
find them. Through that, we eventually arrive at a hypothesis with the same structure 
as the canonical \gls*{MDP} $\can(M)$, where $M$ is the \SUL{} semantics. 
Given a hypothesis with correct structure, it follows 
by the law of large numbers that the estimated transition probabilities converge to true probabilities, thus the hypotheses converge to an \gls*{MDP} isomorphic to $\can(M)$.

A key point of the proofs concerns 
the convergence of statistical test applied by $\diff_f$, which is based on Hoeffding bounds~\cite{10.2307/2282952}. With regard
to that, we apply similar arguments as Carrasco and Oncina~\cite[p.11-13 \& Appendix]{DBLP:journals/ita/CarrascoO99}. 
Given convergence of $\diff_f$, we also rely on the convergence of the exact learning algorithm \LstarMdpE{} discussed in \cref{sec:exact_convergence}.
Another important point is that the shortest traces in each equivalence class of $S/\!\!\equiv_M$ do not form loops in $\can(M)$. Hence, there are finitely many 
such traces. Furthermore, for a given $\can(M)$ and some hypothesis \gls*{MDP}, the shortest counterexample has bounded length, therefore it suffices to check
finitely many test sequences to check for overall equivalence.

\subsubsection{Auxiliary Definitions \& Notation.}
We show convergence in the limit of the number of sampled system traces $n$. We take $n$ into account through a data-dependent $\alpha_n$ for the Hoeffding bounds used by $\diff_f$ defined in Def.~\ref{def:different}. More concretely, let $\alpha_n = \frac{1}{n^r}$ for $r > 2$ as used by Mao et al.~\cite{DBLP:journals/ml/MaoCJNLN16}, which implies $\sum_n \alpha_n n < \infty$. For the remainder of this section, let $\langle S_n, E_n, \Ta_n \rangle$ be the closed and consistent observation table containing the first $n$ samples stored by the teacher in the multiset $\samples_n$. Furthermore, let $\mathcal{H}_n$ be the hypothesis $\mathrm{hyp}( S_n, E_n, \Ta_n)$, let the semantics of the \SUL{} be $M$ and let $\mathcal{M}$ be the canonical \gls*{MDP} $\can(M)$. We say that two \glspl*{MDP} have the same structure, if their underlying graphs are isomorphic, i.e. exact transition probabilities may be different. 

\begin{theorem}[Convergence]
\label{theorem:convergence_non_exact}
Given a data-dependent $\alpha_n = \frac{1}{n^r}$ for $r > 2$, such that $\sum_n \alpha_n n < \infty$, 
then with probability one, the hypothesis $\mathcal{H}_n$ is isomorphic to $\mathcal{M}$, except for finitely many $n$.
\end{theorem}
Hence, we learn an \gls*{MDP} that is minimal with respect to the number of states and output-distribution equivalent to the \SUL{}.

\subsubsection{Hoeffding-Bound-Based Difference Check.}
First, we briefly discuss the Hoeffding-bound-based test applied by $\diff_f$. Recall, that for two test sequences $s$ and $s'$, we test for each $o \in \Out$ if the probability $p$ for observing $o$ after $s$ is different than the probability $p'$ for observing $o$ after $s'$. This is implemented through:
$$
\exists o \in \Out: \left| \frac{f(s)(o)}{n_1} - \frac{f(s')(o)}{n_2} \right| > \left(\sqrt{\frac{1}{n_1}}+\sqrt{\frac{1}{n_2}}\right) \cdot \sqrt{\frac{1}{2} \ln \frac{2}{\alpha}} = \epsilon_\alpha(n_1,n_2)
$$

As pointed out by Carrasco and Oncina~\cite[p.11-13 \& Appendix]{DBLP:journals/ita/CarrascoO99}, this test works with confidence level above $(1-\alpha)^2$ and for large enough $n_1$ and $n_2$ it tests for difference and equivalence of $p$ and $p'$. More concretely, for convergence, $n_1$ and $n_2$ must be such 
that $2\epsilon_\alpha(n_1,n_2)$ is smaller than the smallest absolute difference between any two different $p$ and $p'$. As our data-dependent 
$\alpha_n$ decreases only polynomially, $\epsilon_\alpha(n_1,n_2)$ tends to zero for increasing $n_1$ and $n_2$. Hence, the test implemented by
$\diff_f$ converges to an exact comparison between $p$ and $p'$.

In the remainder of the paper, we ignore Condition 2.a for $\diff_f$, which checks if the sampled distributions have the same support. By applying a 
data-dependent $\alpha_n$, as defined above, Condition 2.b converges to an exact comparison between output distributions, thus 2.a is a consequence of 2.b in the limit. Therefore, we only consider the Hoeffding-based tests of Condition 2.b.

\subsubsection{Access Sequences.}
The exact learning algorithm \LstarMdpE{} presented in \cref{sec:method_exact} iteratively updates an observation table. Upon termination it arrives at an observation table $\langle S,E,T\rangle$ producing a hypothesis $\mathcal{H} = \langle Q_\mathrm{h},\In, \Out,{q_0}_\mathrm{h}, \delta_\mathrm{h}, L_\mathrm{h}\rangle = \mathrm{hyp}( S,E,T)$. Let $S_\mathrm{acc} \subseteq S$ be the set of shortest access sequences leading to states in $Q$ given by $S_\mathrm{acc} = \{s| s \in S, \nexists s' \in S: s' \ll s \land s' \neq s \land \delta^*_\mathrm{h}({q_0}_\mathrm{h},s) = \delta^*_\mathrm{h}({q_0}_\mathrm{h},s')\}$ (the shortest traces in each equivalence class of $S/\!\!\equiv_M$). By this definition, $S_\mathrm{acc}$ forms a directed spanning tree in the structure of $\mathcal{H}$. There are finitely many different spanning trees for a given hypothesis, therefore there are finitely many different $S_\mathrm{acc}$. Hypothesis models learned by \LstarMdpE{} are isomorphic to $\mathcal{M}$, thus there are finitely many possible final hypotheses. Let $\overline{S}$ be the finite union of all access sequence sets $S_\mathrm{acc}$ forming spanning trees in all valid final hypotheses. Let $\overline{L} = \{s\cdot i \cdot o | s \in \overline{S}, i \in \In, o \in \Out,M(s\cdot i)(o) > 0 \}$ be one-step extensions of $\overline{S}$ with non-zero probability. Observe that for the correct construction of correct hypotheses in \LstarMdpE{}, it is sufficient for $\equivrow_E$ to approximate $M$-equivalence (see \cref{def:m_equiv}) for traces in $\overline{L}$. Consequently, the approximation of $\equivrow_E$ via $\compatible_E$ needs to hold only for traces in $\overline{L}$.

\subsubsection{Hypothesis Construction.}

\begin{theorem}[Compatibility Convergence]
\label{theorem:compatible_convergence}
 Given $\alpha_n$ such that $\sum_n \alpha_n n < \infty$, then with probability one: $\compatible_E(s,s') \Leftrightarrow \equivrow_E(s,s')$ for all traces $s, s'$ in $\overline{L}$, except for finitely many $n$.
\end{theorem} 
% \todo[inline]{Now we talk only about Hoeffding tests in $\diff$, but the implementation also contains another check for \emph{``same support''} of empirical output distributions. 
% Either we show that this check follows from the Hoeffing test in the limit, or we simply introduce another simplification for that by ignoring it. The additional check improves accuracy in experiments.
% }
\begin{proof}
Let $A_n$ be the event that $\compatible_E(s,s') \not\Leftrightarrow \equivrow_E(s,s')$ and $p(A_n)$ be the probability of this event. In the following, we derive a bound for $p(A_n)$ based on the confidence level of applied tests in \cref{def:different} which is above $(1-\alpha_n)^2$~\cite{DBLP:journals/ita/CarrascoO99}. An observation table stores $|S\cup \LongTr(S)| \cdot |E|$ cells, which gives us an upper bound on the number of tests performed for computing $\compatible_E(s,s')$ for two traces $s$ and $s'$. However, note that cells do not store unique information; multiple cells may correspond to the same test sequence in $\TestSeq$, therefore it is simpler to reason about the number of tests in calls to $\diff_{\Ta}(c,c') = \diff_{\fq}(c,c')$ with respect to $\samples_n$. A single call to $\diff_\fq$ involves either $0$ or $|\Out|$  tests. We apply  tests only if we have observed both $c$ and $c'$ at least once, therefore we perform at most $2 \cdot |\Out| \cdot n$ different  tests for all pairs of observed test sequences. The event $A_n$ may occur if any test produces an incorrect result, i.e. it yields a Boolean result different from the comparison between the true output distributions induced by $c$ and $c'$. This leads to $p(A_n) \leq 2 \cdot |\Out| \cdot n \cdot (1 - (1-\alpha_n)^2)$, %
which implies $p(A_n) \leq 4 \cdot |\Out| \cdot n \cdot \alpha_n$. By choosing $\alpha_n$ such that $\sum_n \alpha_n n < \infty$, we have $\sum_n p(A_n) < \infty$ and we can apply the Borel-Cantelli lemma like Carrasco and Oncina~\cite{DBLP:journals/ita/CarrascoO99}, which states 
$A_n$ happens only finitely often. Hence, there is an $N_\mathrm{comp}$ such that for $n > N_\mathrm{comp}$, we have $\compatible_E(s,s') \Leftrightarrow \equivrow_E(s,s')$ with respect to $\samples_n$. %\todo{we arrived at this results without incorporating trace generation, so there may be issues; and I added $\overline{L}$ below only because \cite{DBLP:journals/ita/CarrascoO99} did something similar.}
%As noted above $\compatible_E(s,s') \Leftrightarrow \equivrow_E(s,s')$ needs to hold only for $s$ and $s'$ in $\overline{L}$, which is a finite set. 
\end{proof}

\begin{lemma}
\label{lem:sample_l_bar}
 Under the assumed uniformly randomised equivalence testing strategy, for every $s \cdot i \cdot o \in \overline{L}:\samples_n(s\cdot i \cdot o) > 0$ after finitely many $n$.
\end{lemma}

\begin{proof}
 Informally, we will eventually sample all traces $l \in \overline{L}$. The probability $p_L$ of sampling $l = o_0 \cdot i_1 \cdot o_1 \cdots o_n \cdot i \cdot o$ during a test, where $l[\ll k]$ is the prefix of $l$ of length $k$, is given by (note that we may sample $l$ as a prefix of another sequence):
 $$ 
p_L = \frac{1}{|\In|^{n+1}} M(l[\ll 1])(o_1) \cdots M(l[\ll n])(o_n) \cdot M(t[\ll n+1])(o)(1-\pstop)^{n}
$$

Since every $l \in \overline{L}$ is observable, we have $M(l[\ll 1])(o_1) \cdots M(l[\ll n])(o_n) \cdot M(t[\ll n+1])(o) > 0$, thus $p_L > 0$. Hence, there is a finite $N_L$ such that for all $s\cdot i \cdot o \in \overline{L}:\samples_n(s\cdot i \cdot o) > 0$ for $n > N_L$. 
\end{proof}

\begin{lemma}
\label{lem:prefix_closed_r}
If $\compatible_E(s,s') \Leftrightarrow \equivrow_E(s,s')$, then the set of representatives $R$ computed by \cref{alg:create_comp_classes} for the closed and consistent observation table $\langle S_n,E_n, \Ta_n\rangle$ is prefix-closed. 
\end{lemma}
\begin{proof}
Recall that we assume the function $\mathrm{rank}$ to impose a lexicographic ordering on traces. This simplifies showing prefix-closedness of $R$, which we do 
by contradiction. Assume that $R$ is not prefix-closed.  In that case, there is a trace $r$ of length $n$ in $R$ with a prefix $r_p$ of length $n-1$
that is not in $R$. As $r_p\notin R$, we have $r_p \neq \repr{r_p}$ and $\mathrm{rank}(r_p) < \mathrm{rank}(\repr{r_p})$, because the representative $\repr{r_p}$ has the largest rank in its class $\cg(r_p)$.  Since $S_n$ is prefix-closed and $R \subseteq S_n$, $r_p \in S_n$. Let $i\in \In$ and $o \in \Out$ such that $r_p \cdot i \cdot o = r$. \cref{alg:create_comp_classes} enforces $\compatible_E(r_p,\repr{r_p})$ and due to consistency, we have that $\compatible_E(r_p \cdot i \cdot o, \repr{r_p} \cdot i \cdot o) = \compatible_E(r, \repr{r_p} \cdot i \cdot o)$. Since $r$ is a representative in $R$,  $\repr{r_p} \cdot i \cdot o \in \cg(r)$.  Representatives $r$ have the largest rank in their compatibility class $\cg(r)$ and $r \neq \repr{r_p} \cdot i \cdot o$, thus $\mathrm{rank}(r) > \mathrm{rank}(\repr{r_p} \cdot i \cdot o)$. 

In combination we have $\mathrm{rank}(r_p) < \mathrm{rank}(\repr{r_p})$ and $\mathrm{rank}(r_p \cdot i \cdot o) > \mathrm{rank}(\repr{r_p} \cdot i \cdot o)$ 
which is a contradiction given the lexicographic ordering on traces imposed by $\mathrm{rank}$. Consequently, $R$ must be prefix-closed under the premises of \cref{lem:prefix_closed_r}.

\end{proof}

\begin{lemma}
\label{lem:exact_t_same_support}
 Let $\langle S_n, E_n, T_n \rangle$ be the \emph{exact} observation table corresponding to the sampling-based observation table $\langle S_n, E_n, \Ta_n \rangle$, i.e. with $T_n(s) = \odq(s)$ for $s\in (S_n \cup \LongTr(S_n))\cdot E$. Then, $T_n(r \cdot i)(o) > 0 \Leftrightarrow \Ta_n(r \cdot i)(o) > 0$ for $r \in R, i \in \In, o \in \Out$ after finitely many $n$. 
\end{lemma}

\begin{proof}
 First, we will show for prefix-closed $R$ (\cref{lem:prefix_closed_r}) that $R \subseteq \overline{S}$, \linebreak if $\compatible_E(s,s') \Leftrightarrow \equivrow_E(s,s')$. $\overline{S}$ contains all traces corresponding to simple paths of $\can(M)$, 
 therefore we show by contradiction that no $r \in R$ forms a cycle in $\can(M)$. 
 
 Assume that $r$ forms a cycle in $\can(M)$, i.e. it visits states multiple times. We can split $r$ into three parts $r = r_p \cdot r_c \cdot r_s$, where $r_p\in \Traces$ such that $r_p$ and $r_p \cdot r_c$ reach the same state, and $r_s \in (\In \times \Out)^*$ is the longest suffix such that $r_s$ visits every state of $\can(M)$ at most once. As $R$ is prefix-closed, $R$ includes $r_p$ and $r_p \cdot r_c$ as well. The traces $r_p$ and $r_p \cdot r_c$ reach the same state in $\can(M)$, thus %, otherwise we could increase the length of the suffix $r_s$. 
 we have $r_p  \equiv_M r_p \cdot r_c$ which implies $\equivrow_E(r_p,r_p \cdot r_c)$ and $\compatible_E(r_p ,r_p \cdot r_c)$. By \cref{alg:create_comp_classes} all $r \in R$ are pairwise not compatible with respect to $\compatible_E$ leading to a contradiction, thus no $r$ visits a state of $\can(M)$ more than once and we have $R \subseteq \overline{S}$.

Hence, every observable $r_l = r\cdot i \cdot o$ for $r\in R, i \in \In$ and $o \in \Out$ is in $\overline{L}$, as $\overline{L}$
includes all observable extensions of $\overline{S}$.
By \cref{lem:sample_l_bar}, we will sample $r_l$ eventually, i.e. $\Ta_n(r \cdot i)(o) > 0$ and therefore $T_n(r \cdot i)(o) > 0 \Leftrightarrow \Ta_n(r \cdot i)(o) > 0$ after finitely many $n$.

\end{proof}

\begin{lemma}
\label{lem:chaos}
The chaos state $q_\mathrm{chaos}$ is not reachable in $\mathcal{H}_n$, except for finitely many $n$.
\end{lemma}

\begin{proof}
 We add a transition from state $q = \langle \lastOut(r), \row(r)\rangle$ with input $i$ to $q_\mathrm{chaos}$ if $\complete(r \cdot i) = \false$. As we consider $\ncomplete = 1$, $\complete(r \cdot i) = \true$ if there is an $o$ such that $\Ta_n(r \cdot i)(o) > 0$. \cref{lem:exact_t_same_support} states  
 that $\Ta_n(r \cdot i)(o) > 0$ for any observable $r\cdot i \cdot o$ after finitely many $n$. Thus, \cref{lem:exact_t_same_support} implies $\complete(r \cdot i) = \true$ for all $r \in R$ and $i \in \In$, therefore the chaos is unreachable in $\mathcal{H}_n$, except for finitely many $n$. 
\end{proof}

Combining \cref{theorem:compatible_convergence}, \cref{lem:exact_t_same_support} and \cref{lem:chaos}, it follows that, after finitely many $n$, hypotheses created in the sampling-based setting have the same structure as in the exact setting. 

\begin{corollary}
\label{cor:structure}
  Let $\langle S_n, E_n, T_n \rangle$ be the exact observation table corresponding to the sampling-based observation table $\langle S_n, E_n, \Ta_n \rangle$, i.e. $T_n(s) = \odq(s)$ for $s\in (S_n \cup \LongTr(S_n))\cdot E$. Then there exists a finite $N_\mathrm{struct}$ such that the exact hypothesis $\mathrm{hyp}( S_n, E_n, T_n)$ has the same structure as $\mathcal{H}_n$ for $n > N_\mathrm{struct}$. 
\end{corollary}

\subsubsection{Equivalence Queries.}
\begin{theorem}[Convergence of Equivalence Queries]
\label{theorem:conv_eq_queries}
 Given $\alpha_n$ such that $\sum_n \alpha_n n < \infty$, an observation table  $\langle S_n, E_n, \Ta_n \rangle$ and a hypothesis $\mathcal{H}_n$, then with probability one, $\mathcal{H}_n$ has the same structure as $\mathcal{M}$ or we find a counterexample to equivalence, except for finitely many $n$. 
\end{theorem}

According to \cref{cor:structure}, there is an $N_\mathrm{struct}$ such that $\mathcal{H}_n$ has the same structure as in the exact setting and $\compatible_E(s,s') \Leftrightarrow \equivrow_E(s,s')$ for $n > N_\mathrm{struct}$. Therefore, we assume $n > N_\mathrm{struct}$ for the following discussion of counterexample search through the implemented equivalence queries $\eq$. Let $H_n$ be the semantics of $\mathcal{H}_n$. 
Recall that we apply two strategies for checking equivalence:
\begin{enumerate}
 \item Random testing with a uniformly randomised scheduler ($\prand = 1$): this form of testing of testing can find traces $s \cdot o$, with $s \in\TestSeq$ and $o \in \Out$, such that $H(s)(o) = 0$ and $M(s)(o) > 0$. While this form of search is coarse, we store all sampled traces in $\samples_n$ that is used by our second counterexample search strategy performing a fine-grained analysis. 
 \item Checking conformance with $\samples_n$: for all observed test sequences, we statistically check for differences between output distributions in $\mathcal{H}_n$ and distributions estimated from $\samples_n$ through applying $\diff_\fq$. Applying that strategy finds counterexample sequences $s \in \TestSeq$ such that $M(s) \neq \bot$ (as $s$ must have been observed) and approximately $M(s) \neq H(s)$.
 \end{enumerate}
 
\paragraph{Case 1.} 
If $\mathcal{H}_n$ and $\mathcal{M}$ have the same structure and $n > N_\mathrm{struct}$, such that $\equivrow_E(s,s') \Leftrightarrow \compatible_E(s,s')$, we may still find counterexamples that are spurious due to inaccuracies. Therefore, we will show that adding a prefix-closed set of traces to the set of short traces $S_n$ does not change the hypothesis structure, as this is performed by \cref{alg:overview} in response to counterexamples returned by $\eq$. 

\begin{lemma}
\label{lem:non_changing_structure}
If $\mathcal{H}_n$ has the same structure as $\mathcal{M}$ and $n > N_\mathrm{struct}$, then adding a prefix-closed set of observable traces $S_t$ to $S_n$ will neither introduce  closedness-violations nor inconsistencies, i.e. $\langle S_n\cup S_t,E_n,\Ta_n \rangle$ is closed and consistent. Consequently, the hypothesis structure does not change, i.e. $\mathcal{H}_n$ and $\mathrm{hyp}( S_n\cup S_t,E_n,\Ta_n)$ have the same structure.
\end{lemma}
\begin{proof}

Let $t$ be a trace in $S_t$ and $q_t = \delta^*_\mathrm{h}(t)$ be the hypothesis state reached by $t$, which exists because $\mathcal{H}_n$ has the same structure as $\mathcal{M}$. Let $t_s \in S_n$ be a short trace also reaching $q_t$. Since $\mathcal{M}$ and $\mathcal{H}_n$ have the same structure, $t$ and $t_s$ also reach the same state of $\mathcal{M}$, therefore $t \equiv_M t_s$ (by reaching the same state both traces lead to the same future behaviour), implying $\equivrow_E(t,t_s)$. With $n > N_\mathrm{struct}$, we have $\compatible_E(t,t_s)$.
By the same reasoning, we have $\compatible_E(t \cdot i \cdot o,t_s\cdot i \cdot o)$ for any $i \in \In$, $o \in \Out$ with $M(t\cdot i)(o) > 0$; which is the condition for consistency of observation tables, i.e. adding $t$ to $S_n$ leaves the observation tables consistent. 

Furthermore because $\langle S_n,E_n,\Ta_n \rangle$ is closed, there exists a $t_s' \in S_n$, with $\compatible_E(t_s\cdot i \cdot o,t_s')$. Since $\compatible_E(t \cdot i \cdot o,t_s\cdot i \cdot o)$ and because $\compatible_E$ is transitive for $n > N_\mathrm{struct}$, we have $\compatible_E(t \cdot i \cdot o,t_s')$. Hence, adding $t$ as to $S_n$ does not violate closedness, because for each observable extensions of $t$, there exists a compatible short trace $t_s'$. 

\end{proof}

\paragraph{Case 2.} If the hypothesis $\mathcal{H}_n$ does not have the same structure as $\mathcal{M}$ and $n > N_\mathrm{struct}$, then $\mathcal{H}_n$ has fewer states than $\mathcal{M}$ (following \cref{lem:minimal} given that $\mathcal{H}$ is consistent with $\Ta_n$ and $\compatible_E(s,s') \Leftrightarrow \equivrow_E(s,s')$). Since $\mathcal{M}$ is minimal with respect to the number of states, $\mathcal{H}_n$ and $\mathcal{M}$ are not equivalent, thus a counterexample to observation equivalence exists and we are guaranteed to find any such counterexample after finitely many samples. 

\begin{lemma}
\label{lem:non_exact_minimal_mdp}
 If $\compatible_E(s,s') \Leftrightarrow \equivrow_E(s,s')$  for traces $s$ and $s'$ in $S_n$, then the hypothesis $\mathcal{H}_n$ derived from $\langle S_n,E_n,\Ta_n \rangle$ is the smallest \gls*{MDP} consistent with $\Ta_n$.\end{lemma}
 \begin{proof}
Recall that for a given observation table $\langle S, E,T\rangle$, the exact learning algorithm \LstarMdpE{} derives the smallest hypothesis consistent with $T$. By \cref{cor:structure}, $\mathcal{H}_n$ is the smallest \gls*{MDP} consistent $T$. As $\diff_{\Ta_n}$ does not produce spurious results for $n > N_\mathrm{struct}$ (\cref{theorem:compatible_convergence}), $\mathcal{H}_n$ is also the smallest
\gls*{MDP} consistent with $\Ta_n$ with respect to $\diff_{\Ta_n}$. 

%$\mathcal{M}$ is the canonical \gls*{MDP} for the \SUL{} and consistent with $\Ta$, with respect to $\diff_\Ta$. Hence, $\mathcal{H}_n$ is not equivalent to the \SUL{} and there must exist a counterexample demonstrating that. The following lemma provides us with a set $C$ of counterexample sufficient to demonstrate equivalence. Consequently, given a hypothesis $\mathcal{H}_n$ not equivalent to the \SUL{}, we can find a counterexample in $C$. 
 \end{proof}

\begin{lemma}
\label{lem:sufficient_cex}
Let $n_q$ be the number of states of $\mathcal{M}$, $C = \bigcup_{i = 0}^{n_q^2+1} (\Out \times 
\In)^{i}$ and $C^\mathrm{obs} = \{c | c \in C: M(c) \neq \bot\}$. For any other \gls*{MDP} 
$\mathcal{M}'$ with at most $n_q$ states and semantics $M'$, iff $\forall c \in C^\mathrm{obs}: M(c) 
= M'(c)$, then $\mathcal{M} \equiv_\mathrm{od} \mathcal{M}'$. 
\end{lemma} %\todo{I am sure that the length of sequences could be bounded by $n_q$}

Hence, there is a finite set $C^\mathrm{obs}$ of sequences with lengths bounded by $n_q^2+1$ such that we if we test all sequence in $C^\mathrm{obs}$, we can check equivalence with certainty. 
\begin{proof}
Let $\mathcal{M}$ and $\mathcal{M}'$ with states $Q$ and $Q'$ as defined above, i.e. $|Q| = n_q$ and $|Q'|\leq n_q$, and let $\mathrm{reachQSeq}(t)\in (Q\times Q')^*$ be the sequence of state-pairs visited along a trace $t$ by $\mathcal{M}$ and $\mathcal{M}'$, respectively. $\mathcal{M} \equiv_\mathrm{od} \mathcal{M}'$ iff for all $t \in \Traces$ and $i\in \In$, we have $M(t \cdot i) = M'(t \cdot i)$. If the length of $t \cdot i$ is at most $n_q^2 + 1$, then $t \cdot i \in C$. Otherwise, $\mathrm{reachQSeq}(t)$ contains duplicated state pairs, because $|Q \times Q'|\leq n_q^2$. For $t$ longer than $n_q^2$, we can remove loops on $Q\times Q'$ from $t$ to determine a trace $t'$ of length at most $n_q^2$ such that $\mathrm{reachQSeq}(t)[|t|] = \mathrm{reachQSeq}(t')[|t'|]$, i.e  such that $t$ and $t'$ reach the same state pair. Since $t$ reaches the same state as $t'$ in $\mathcal{M}$ and in $\mathcal{M}'$, we have $M(t \cdot i) = M(t' \cdot i)$ and $M'(t \cdot i) = M'(t' \cdot i)$, thus $M(t \cdot i) = M'(t \cdot i) \Leftrightarrow M(t' \cdot i) = M'(t' \cdot i)$. Consequently for all $t \cdot i \in \Traces \cdot \In$: either $t\cdot i \in C$, or there is a $t'\cdot i \in C$ leading to the same check between $\mathcal{M}$ and $\mathcal{M}'$.

We further restrict $C$ to $C^\mathrm{obs}$, by considering only observable test sequences in $C$. 
% Note that for sequence for a sequence $t \cdot i$, consisting of a trace $t$ and input $i$, $M(t\cdot i) \neq \bot$ iff $t$ is observable on $\mathcal{M}$.
% We can further refine $C$ to $C^\mathrm{obs}$ by considering only sequences corresponding to observable traces. 
This restriction is justified by \cref{rem:minimal_cex}. In summary:
\begin{align*}
\mathcal{M} \equiv_\mathrm{od} \mathcal{M}' &\Leftrightarrow \forall c \in \TestSeq :M(c) = M'(c) \\
&\Leftrightarrow \forall c \in C :M(c) = M'(c) \\
&\Leftrightarrow \forall c \in C^\mathrm{obs} :M(c) = M'(c) \\
\end{align*}

% Hence, it holds that $\forall c \in C: M(c) = M'(c) \Leftrightarrow \forall c' \in C^\mathrm{obs}: M(c') = M'(c')$. 
\end{proof}

% \begin{lemma}
% Let $C$, $\mathcal{M}$ and $\mathcal{M}'$ be defined as in the previous lemma. Let $C^\mathrm{obs} = \{c | c \in C: M(c) \neq \bot\}$, iff $\forall c \in C^\mathrm{obs}: M(c) = M'(c)$, then $\mathcal{M} \equiv_\mathrm{od} \mathcal{M}'$.
% \end{lemma} 

%  \begin{proof}
% We can show this by contradiction. Assume that $\forall c \in C^\mathrm{obs}: M(c) = M'(c)$ holds but not $\forall c \in C: M(c) = M'(c)$. From this assumption it follows that there is a $cex \in C \setminus C^\mathrm{obs}:  M(cex) \neq M'(cex)$. Since $cex \in C \setminus C^\mathrm{obs}$ we have $M(cex) = \bot$ and $M'(cex) \neq \bot$. Let $cex'$ be the the longest prefix of $cex$ such that $M(cex') = \bot$, thus $cex'$ is of the form $cex' = cex'' \cdot o \cdot i$ with $M(cex'')(o) = 0$. Due to prefix-closedness of $\semdom(M')$, $M'(cex) \neq \bot$ implies $M'(cex') \neq \bot$, therefore $M'(cex'')(o) > 0$ and we have $M(cex'') = M'(cex'')$ for a $cex'' \in C^\mathrm{obs}$, as $M(cex'') \neq \bot$. This contradicts our assumption, hence  $\forall c \in C: M(c) = M'(c) \Leftrightarrow \forall c' \in C^\mathrm{obs}: M(c') = M'(c')$.
%  \end{proof}

\begin{lemma}
\label{lem:non_zero_prob_c}
 Under the randomised testing strategy with $\prand = 1$ and $\pstop < 1$, all $c$ in $C^\mathrm{obs}$ have non-zero probability to be observed. 
\end{lemma}

\begin{proof}
Due to $\prand = 1$ and $\pstop < 1$ we apply uniformly randomised inputs during testing and each test has a length that is distributed dependent on $\pstop$. Let $c = o_0 i_1 o_1 \cdots o_{n-1} \cdot i_n$ be a sequence in $C^\mathrm{obs}$ with $c[\ll k]$ being its prefix of length $k$, then the probability $p_c$ of observing $c$ is (note that we may observe $c$ as a prefix of another sequence):
$$ 
p_c = \frac{1}{|\In|^n} M(c[\ll 1])(o_1) \cdot M(c[\ll 2])(o_2) \cdots M(c[\ll n-1])(o_{n-1}) \cdot (1-\pstop)^{n-1}
$$
By definition of $C^\mathrm{obs}$, we have $M(c[\ll j])(o_j) > 0$ for all indexes $j$ and $c$ in $C^\mathrm{obs}$, therefore $p_c > 0$. 
\end{proof}

In every round of \LstarMdp{}, we check for conformance between $\samples_n$ and the hypothesis $\mathcal{H}_n$ and return a counterexample if we detect a difference via $\diff_\fq$. Since we apply $\diff_\fq$, we follow a similar reasoning as for the convergence of hypothesis construction. Here, we approximate $M(c) \neq H(c)$ for $c \in \TestSeq$ by $\diff_\fq(t \cdot i, r \cdot i)$, where $c = t \cdot i$ for a trace $t$, input $i$ and the hypothesis state $\langle \lastOut(r), \row(r)\rangle$ reached by $t$, where $r\in R$ is the corresponding representative short trace. 

\begin{lemma}
\label{lem:find_cex}
  Given $\alpha_n$ such that $\sum_n \alpha_n n < \infty$, then with probability one \linebreak
  $M(c) \neq H(c) \Leftrightarrow \diff_\fq(t \cdot i, r \cdot i)$ for 
  $c= t \cdot i \in C^\mathrm{obs}$ and $r$ as defined above, except for finitely many $n$.
\end{lemma}

\begin{proof}
We use the identity $H(t \cdot i) = H(r \cdot i)$ for traces $t$ and $r$ and inputs $i$, which holds because $t$ and $r$ reach the same state in the hypothesis $\mathcal{H}$. Applying that, we test for $M(t \cdot i) \neq H(t \cdot i)$ by testing $M(t \cdot i) \neq H(r \cdot i)$ via $\diff_\fq(t \cdot i, r \cdot i)$. We perform $|\Out|$ tests for each unique observed sequence $c$, therefore we apply at most $n \cdot |\Out|$ tests. Let $B_n$ be the event that any of these tests is wrong, that is, $M(t \cdot i) \neq H(r \cdot i) \not\Leftrightarrow \diff_\fq(t \cdot i, r \cdot i)$ for at least one observed $c = t\cdot i$. Due to the confidence level greater than $(1-\alpha_n)^2$ of the tests, the probability $p(B_n)$ of $B_n$ is bounded by $p(B_n) \leq n \cdot |\Out| \cdot (1-(1-\alpha_n)^2) \leq 2 \cdot n\cdot |\Out|\cdot \alpha_n$. By choosing $\alpha_n$ such that $\sum_n \alpha_n n < \infty$, we can apply the Borel-Cantelli lemma as above. Hence, $B_n$ only happens finitely often, thus there is an $N_1$ such that for all $n > N_1$ we have $M(t \cdot i) \neq H(r \cdot i) \Leftrightarrow \diff_\fq(t \cdot i, r \cdot i)$ for all observed $c = t\cdot i$. Furthermore, the probability of observing any $c$ of the finite set $C^\mathrm{obs}$ during testing is greater than zero (\cref{lem:non_zero_prob_c}), thus there is a finite $N_2$ such that $\samples_n$ contains all $c \in C^\mathrm{obs}$ for $n > N_2$. Consequently, there is an $N_\mathrm{cex}$, %the maximum of $N_1$ and $N_2$, 
such that \cref{lem:find_cex} holds for all $n > N_\mathrm{cex}$.
\end{proof}

\cref{lem:non_exact_minimal_mdp} states that hypotheses $\mathcal{H}_n$ are minimal after finitely many $n$ and thus all potential counterexamples are in $C^\mathrm{obs}$ (\cref{lem:sufficient_cex}). From \cref{lem:find_cex}, it follows that we will identify a counterexample in $C^\mathrm{obs}$ if one exists. 
Combining that with \cref{lem:non_changing_structure} concludes the proof of \cref{theorem:conv_eq_queries}.

\subsubsection{Putting Everything Together.}
% \begin{theorem}
%  Given $\alpha_n$ such that $\sum_n \alpha_n n < \infty$, then with probability one \LstarMdp creates a $\mathcal{H}_n$ such that $\mathcal{H}_n \equiv_\mathrm{od} \mathcal{M}$ except for finitely many $n$.
% \end{theorem} 

We have established that after finitely many $n$, the sampling-based hypothesis $\mathcal{H}_n$ has the same structure as in the exact
setting (\cref{cor:structure}). Therefore, certain properties of the exact learning algorithm \LstarMdpE{} hold for the 
sampling-based \LstarMdp{} as well. The derived hypotheses are therefore minimal, i.e. they have at most as many states as $\mathcal{M}$. 
As with \LstarMdpE{}, adding a non-spurious counterexample to the trace set $S_n$ introduces at least one state in the derived hypotheses. 
Furthermore, we have shown that equivalence queries return non-spurious counterexamples, except for finitely many $n$ (\cref{theorem:conv_eq_queries}). 
Consequently, after finite $n$ we arrive at a hypothesis $\mathcal{H}_n$ with the same structure as $\mathcal{M}$. We derive transition probabilities by 
computing empirical means, thus by the law of large numbers these estimated probabilities converge to the true probabilities. 
Hence, we learn a hypothesis $\mathcal{H}_n$ isomorphic to the canonical \gls*{MDP} $\mathcal{M}$ in the limit as stated by \cref{theorem:convergence_non_exact}.

\subsubsection{More efficient parameters.}
So far, we discussed a particular parametrisation of \LstarMdp{}.
Among others, we used uniformly random input choices for equivalence testing with $\prand = 1$,
and instantiated $\complete$ to accept samples as complete after only $\ncomplete = 1$ observation.
This simplified the proof, but is inefficient in practical experiments.
However, the arguments based on $\ncomplete = 1$, such as \autoref{lem:exact_t_same_support} and \autoref{lem:chaos}, are easily extended to small constant values of $\ncomplete$:
Since the samples are collected independently, any observation that occurs at least once after a finite number of steps also occurs at least $\ncomplete$ times after a finite number of steps.
% \todo{I also expect that all arguments extend to $\prand < 1$, as long as $\prand > 0$ (?).}

\end{full}

\section{Experiments}
\label{sec:eval}
\begin{full}
In active automata learning, our goal is generally to learn an \gls*{MDP} which is equivalent to the true 
\gls*{MDP} modelling the \SUL. This changes in the stochastic setting, where we want 
to learn a model close to true model, as equivalence can hardly be achieved. 
% In the following, we report results from two kinds of experiments with four different systems: two gridworlds, a consensus protocol model and model of a slot machine.
Note that we perform experiments with known models, which we treat as a black boxes during learning. 
As a reference, we also learn models and perform the same measurements with \textsc{IoAlergia}. 
Our experiments aim to measure the similarity between the learned models and the true model:
\begin{enumerate}
 \item We compute the discounted bisimilarity distance between the true models and the learned \glspl*{MDP}~\cite{DBLP:conf/mfcs/BacciBLM13,DBLP:conf/qest/BacciBLM13}.
 We adapted the distance measure from \glspl*{MDP} with rewards to labelled \glspl*{MDP} by defining a distance of $1$ between states with different labels.
 \item Additionally, we perform probabilistic model-checking. We
 compute and compare maximal probabilities of manually defined temporal properties 
 with all models. The computation is done via \textsc{Prism}~\cite{DBLP:conf/cav/KwiatkowskaNP11}.
\end{enumerate}
Experimental results and the implementation can be found in the evaluation material~\cite{eval_material}.

\end{full}
\begin{conference}
We evaluate the sampling-based \LstarMdp{} and compare it to the passive \textsc{IoAlergia}~\cite{DBLP:journals/ml/MaoCJNLN16} by learning a gridworld model with both techniques. 
Experimental results and the implementation can be found in the evaluation material~\cite{eval_material}.
We treat the known true \gls*{MDP} model $\mathcal{M}$ as a black box for learning and measure similarity to this model using two criteria:
% \begin{compactenum}
 (1) the discounted bisimilarity distance~\cite{DBLP:conf/mfcs/BacciBLM13,DBLP:conf/qest/BacciBLM13} between $\mathcal{M}$ and the learned \glspl*{MDP} and 
%  We adapted the distance measure from \glspl*{MDP} with rewards to labelled \glspl*{MDP} by defining a distance of $1$ between states with different labels.
 (2) the difference between probabilistic model-checking results for $\mathcal{M}$ and learned \glspl*{MDP}. We
 compute maximal probabilities of manually defined temporal properties 
 with all models using \textsc{Prism} 4.4~\cite{DBLP:conf/cav/KwiatkowskaNP11}.
% \end{compactenum}
\end{conference}

\paragraph{Measurement Setup.}

\begin{full}
As in~\cite{DBLP:journals/ml/MaoCJNLN16}, we configure \textsc{IoAlergia} 
with a data-dependent significance parameter for the compatibility check, by setting $\epsilon_N = \frac{10000}{N}$,
where $N$ is the total combined length of all traces used for learning. This parameter serves a role analogous to the 
$\alpha$ parameter for the Hoeffding bounds used by \LstarMdp{}. In contrast to \textsc{IoAlergia}, we observed that \LstarMdp{} shows better performance with 
non-data-dependent $\alpha$, therefore we set $\alpha = 0.05$ for all experiments. 
Motivated by convergence guarantees given in~\cite{DBLP:journals/ml/MaoCJNLN16},
we collect traces for \textsc{IoAlergia} by sampling with a scheduler that selects inputs according to a uniform distribution. 
The length of these traces is geometrically distributed with a parameter $p_l$ and the number of traces is chosen such that
\textsc{IoAlergia} and \LstarMdp{} learn from approximately the same amount of data. 
\end{full}
\begin{conference}
As in~\cite{DBLP:journals/ml/MaoCJNLN16}, we use a data-dependent $\epsilon_N = \frac{10000}{N}$ for \textsc{IoAlergia},
where $N$ is the combined length of all learning traces. This parameter serves a role analogous to the $\alpha$ parameter
of \LstarMdp{}. In contrast, we observed that \LstarMdp{} performs better with a fixed $\alpha = 0.05$. 
We sample traces for \textsc{IoAlergia} with a length geometrically distributed with parameter $p_l$ 
and inputs chosen uniformly at random, also as in~\cite{DBLP:journals/ml/MaoCJNLN16}.
The number of traces is chosen such that \textsc{IoAlergia} and \LstarMdp{} learn from approximately the same amount of data. 
\end{conference}

We implemented \LstarMdp{} and \textsc{IoAlergia} in Java.
\begin{full}
In addition 
to our Java implementations, we use \textsc{Prism}  4.4~\cite{DBLP:conf/cav/KwiatkowskaNP11} for probabilistic model-checking. 
and an adaptation of the \textsc{MDPDist} library available at~\cite{web_bisimdist}
for computing bisimilarity distances.
\end{full}
\begin{conference}
 Additionally, we use the \textsc{MDPDist} library~\cite{web_bisimdist} for bisimilarity distances, adapted
 to labelled \glspl*{MDP}.
\end{conference}
We performed the experiments with a Lenovo Thinkpad T450 with 16 GB RAM, an Intel 
Core i7-5600U CPU with $2.6$ GHz and running Xubuntu Linux 18.04.

\begin{full}
\subsection{First Gridworld}
\begin{wrapfigure}[10]{t}{5cm}
% \begin{figure}[t]
 \begin{center}
\begin{tikzpicture}[scale=1, transform shape, font=\scriptsize, inner sep=0pt,outer sep = 0pt]

\node[fill=black, minimum width = 0.2cm, minimum height = 0.2cm] at ( -1.1,+1.6) {};
\node[fill=black, minimum width = 0.5cm, minimum height = 0.2cm] at ( -0.75,+1.6) {};
\node[fill=black, minimum width = 0.5cm, minimum height = 0.2cm] at ( -0.25,+1.6) {};
\node[fill=black, minimum width = 0.5cm, minimum height = 0.2cm] at (  0.25,+1.6) {};
\node[fill=black, minimum width = 0.5cm, minimum height = 0.2cm] at (  0.75,+1.6) {};
\node[fill=black, minimum width = 0.5cm, minimum height = 0.2cm] at (  1.25,+1.6) {};
\node[fill=black, minimum width = 0.2cm, minimum height = 0.2cm] at (  1.6,+1.6) {};

\node[fill=black, minimum width = 0.2cm, minimum height = 0.2cm] at ( -1.1,-1.1) {};
\node[fill=black, minimum width = 0.5cm, minimum height = 0.2cm] at ( -0.75,-1.1) {};
\node[fill=black, minimum width = 0.5cm, minimum height = 0.2cm] at ( -0.25,-1.1) {};
\node[fill=black, minimum width = 0.5cm, minimum height = 0.2cm] at (  0.25,-1.1) {};
\node[fill=black, minimum width = 0.5cm, minimum height = 0.2cm] at (  0.75,-1.1) {};
\node[fill=black, minimum width = 0.5cm, minimum height = 0.2cm] at (  1.25,-1.1) {};
\node[fill=black, minimum width = 0.2cm, minimum height = 0.2cm] at (  1.6,-1.1) {};

\node[fill=black, minimum width = 0.2cm, minimum height = 0.5cm] at ( -1.1,-0.75) {};
\node[fill=black, minimum width = 0.2cm, minimum height = 0.5cm] at ( -1.1,-0.25) {};
\node[fill=black, minimum width = 0.2cm, minimum height = 0.5cm] at ( -1.1, 0.25) {};
\node[fill=black, minimum width = 0.2cm, minimum height = 0.5cm] at ( -1.1, 0.75) {};
\node[fill=black, minimum width = 0.2cm, minimum height = 0.5cm] at ( -1.1, 1.25) {};

\node[fill=black, minimum width = 0.2cm, minimum height = 0.5cm] at (  1.6,-0.75) {};
\node[fill=black, minimum width = 0.2cm, minimum height = 0.5cm] at (  1.6,-0.25) {};
\node[fill=black, minimum width = 0.2cm, minimum height = 0.5cm] at (  1.6, 0.25) {};
\node[fill=black, minimum width = 0.2cm, minimum height = 0.5cm] at (  1.6, 0.75) {};
\node[fill=black, minimum width = 0.2cm, minimum height = 0.5cm] at (  1.6, 1.25) {};

\node at (-0.75,+1.25) {C};
\node[circle,minimum size=0.4cm,draw,thick] at (-0.75,+1.25) {};
\node at (-0.25,+1.25) {C};
\node at ( 0.25,+1.25) {C};
\node[fill=black!40, minimum width = 0.5cm, minimum height = 0.5cm] at ( 0.75,+1.25) {M};
\node[fill=black, minimum width = 0.5cm, minimum height = 0.5cm] at ( 1.25,+1.25) {};

\node[fill=black, minimum width = 0.5cm, minimum height = 0.5cm] at (-0.75,+0.75) {};
\node[fill=black, minimum width = 0.5cm, minimum height = 0.5cm] at (-0.25,+0.75) {};
\node[fill=black, minimum width = 0.5cm, minimum height = 0.5cm] at ( 0.25,+0.75) {};
\node at ( 0.75,+0.75) {C};
\node[fill=black!40, minimum width = 0.5cm, minimum height = 0.5cm] at ( 1.25,+0.75) {M};

\node[fill=black!20, minimum width = 0.5cm, minimum height = 0.5cm] at (-0.75,+0.25) {S};
\node[fill=black!40, minimum width = 0.5cm, minimum height = 0.5cm] at (-0.25,+0.25) {M};
\node[fill=black!30, minimum width = 0.5cm, minimum height = 0.5cm] at ( 0.25,+0.25) {G};
\node at ( 0.75,+0.25) {C};
\node[fill=black!30, minimum width = 0.5cm, minimum height = 0.5cm] at ( 1.25,+0.25) {G};

\node[fill=black!40, minimum width = 0.5cm, minimum height = 0.5cm] at (-0.75,-0.25) {M};
\node[fill=black!30, minimum width = 0.5cm, minimum height = 0.5cm] at (-0.25,-0.25) {G};
\node at ( 0.25,-0.25) {C};
\node[fill=black!40, minimum width = 0.5cm, minimum height = 0.5cm] at ( 0.75,-0.25) {M};
\node[fill=black, minimum width = 0.5cm, minimum height = 0.5cm] at ( 1.25,-0.25) {};

\node[circle,minimum size=0.4cm,draw,thick] at (-0.75,-0.75) {};
\node[circle,minimum size=0.3cm,draw,thick] at (-0.75,-0.75) {};

\node[fill=black!30, minimum width = 0.5cm, minimum height = 0.5cm] at (-0.75,-0.75) {G};

\node[circle,minimum size=0.4cm,draw,thick] at (-0.75,-0.75) {};
\node[circle,minimum size=0.25cm,draw,thick] at (-0.75,-0.75) {};
\node[fill=black!20, minimum width = 0.5cm, minimum height = 0.5cm] at (-0.25,-0.75) {S};
\node[fill=black!40, minimum width = 0.5cm, minimum height = 0.5cm] at ( 0.25,-0.75) {M};
\node[fill=black!30, minimum width = 0.5cm, minimum height = 0.5cm] at ( 0.75,-0.75) {G};
\node[fill=black, minimum width = 0.5cm, minimum height = 0.5cm] at ( 1.25,-0.75) {};
\node at (1,+1.25){};
\node at (1,-1.25){};
\draw[step=0.5cm,color=gray] (-1,-1) grid (1.5,1.5);

\end{tikzpicture}
\vspace{-0.2cm}
\begin{full}
\caption{The first gridworld}
\label{fig:gridworld_1}
\end{full}
\begin{conference}
\caption{The evaluation gridworld}
\label{fig:gridworld_1}
\end{conference}
\end{center}
\end{wrapfigure}

Models similar to our gridworlds have, e.g., been considered in the context of learning control strategies~\cite{DBLP:conf/rss/FuT14}.
Basically, a robot moves around in a world of tiles of different terrains. It may make errors 
in movement, e.g. move south west instead of south with an error probability depending on the target terrain. 
Our aim is to learn an environment model, i.e. a map. 
Figure \ref{fig:gridworld_1} shows the first gridworld used for evaluation.
Black tiles are walls and other 
terrains are represented by different shades of grey and letters (Sand, Mud, Grass \& Concrete).
A circle marks the initial location and a double circle marks a goal location.
Four inputs enable movement in four directions. Observable
outputs include the different terrains, walls, and a label indicating the goal. 
The true model of this gridworld has $35$ different states.

\end{full}
\begin{conference}
\paragraph{First Gridworld.}

Models similar to our gridworld have, e.g., been considered in the context of learning control strategies~\cite{DBLP:conf/rss/FuT14}. 
Basically, a robot moves around in a world of tiles of different terrains. It may make errors 
in movement, e.g. move south west instead of south with an error probability depending on the target terrain. 
Our aim is to learn an environment model, i.e. a map. 
\cref{fig:gridworld_1} shows our gridworld.
Black tiles are walls and other 
terrains are represented by different shades of grey and letters (Sand, Mud, Grass \& Concrete).
A circle marks the initial location and a double circle marks a goal location.
Four inputs enable movement in four directions. Observable
outputs include the different terrains, walls, and a label indicating the goal. 
The true model of this gridworld has $35$ different states. All terrains except Concrete have a distinct positive error probability.

\end{conference}

\begin{full}
We set the sampling parameters to
$\nbatch = \nretest = 300$, $\ntest = 50$, $\pstop = 0.25$ and $\prand = 0.25$.
As stopping parameter served $t_\mathrm{unamb} = 0.99$, $\rmin =500$ and $\rmax=4000$. 
Finally, the parameter $p_l$ for \textsc{IoAlergia}'s  geometric trace length distribution 
was set to $0.125$.
\end{full}
\begin{conference} % \nretest is not explained in conference version
We configured sampling by
$\nbatch = 300$, $\ntest = 50$, $\pstop = 0.25$ and $\prand = 0.25$, and 
stopping by $t_\mathrm{unamb} = 0.99$, $\rmin =500$ and $\rmax=4000$. 
Finally, we set $p_l=0.25$ for \textsc{IoAlergia}.
\end{conference}

\begin{table}[t]
\centering
\begin{conference}
\caption{Results for learning the gridworld example.}
 \label{tab:results_gridworld_1}
\end{conference}
\begin{full}
\caption{Results for learning the first gridworld example.}
 \label{tab:results_gridworld_1}
\end{full}
\begin{tabular}{r|c|c|c}
 & true model & \LstarMdp & \textsc{IoAlergia} \\\hline
\# outputs     & - & $\n{3101959}$ & $\n{3103607}$ \\\hline
\# traces      & - & $\n{391530}$  & $\n{387746}$ \\\hline
time [s]       & - &  $\n{118.377}$ & $\n{21.442}$ \\\hline
\# states      & $\n{35}$ &  $\n{35}$ & $\n{21}$ \\\hline \hline
$\delta_{0.9}$ & - & $\n{0.144173}$ & $\n{0.524063}$ \\\hline \hline
$\mathbb{P}_{\max}(F^{\leq 11} (\mathrm{goal}))$                    & $\n{0.96217534}$ & $\n{0.9651200582879222}$  & $\n{0.230586749545231}$ \\\hline
$\mathbb{P}_{\max}(\lnot \mathrm{G}\ U^{\leq14} (\mathrm{goal}))$   & $\n{0.6499274956800001}$ & $\n{0.6461047439226908}$  & $\n{0.15766325837412196}$\\\hline
$\mathbb{P}_{\max}(\lnot \mathrm{S}\ U^{\leq16} (\mathrm{goal}))$   & $\n{0.6911765746880001}$ &  $\n{0.6768341773434474}$ & $\n{0.1800497272051816}$\\
\end{tabular}
\end{table}

\paragraph{Results.} \cref{tab:results_gridworld_1} shows the measurement results for learning the 
\begin{full} first \end{full} gridworld. 
Our active learning stopped after $1147$ rounds, sampling $\n{391530}$ traces (Row~$2$) with a combined number of outputs
of $\n{3101959}$ (Row $1$).
The bisimilarity distance discounted with $\lambda=0.9$ to the true model is $0.144$ for \LstarMdp{} and $0.524$ 
for \textsc{IoAlergia} (Row $5$); thus it can be assumed that model checking the \LstarMdp{} model produces more accurate results. 
This is indeed true for our three evaluation queries in the last three rows. These model-checking queries ask for the
maximum probability (quantified over all schedulers) of reaching the $\mathrm{goal}$
within a varying number of steps. The first query does not restrict the terrain visited before the $\mathrm{goal}$,
but the second and third require to avoid $\mathrm{G}$ and $\mathrm{S}$, respectively.
The absolute difference to the true values is at most $0.015$ for \LstarMdp{}, but the results for \textsc{IoAlergia} 
differ greatly from the true values. One reason is that the \textsc{IoAlergia} model
with $21$ states is significantly smaller than the minimal true model, while the \LstarMdp{} model has as many states as the true model. 
\textsc{IoAlergia} is faster than \LstarMdp{}, which
applies time-consuming computations during equivalence queries. 
However, the runtime of learning-specific computations is often negligible in practical applications, such as learning of protocol models~\cite{DBLP:conf/icst/TapplerAB17,DBLP:conf/uss/RuiterP15}, 
as the communication with the \SUL{} usually dominates the overall runtime.
Given the smaller bisimilarity distance and the lower difference to the true probabilities computed with \textsc{Prism}, we conclude 
that the \LstarMdp{} model is more accurate.

\begin{full}

\subsection{Second Gridworld}
 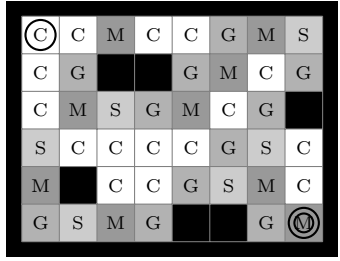
\begin{figure}[t] 
 \begin{center}
\begin{tikzpicture}[scale=1, transform shape, font=\scriptsize, inner sep=0pt,outer sep = 0pt]

\node[fill=black, minimum width = 0.5cm, minimum height = 0.2cm] at ( -1.75,+1.6) {};
\node[fill=black, minimum width = 0.5cm, minimum height = 0.2cm] at ( -1.25,+1.6) {};
\node[fill=black, minimum width = 0.5cm, minimum height = 0.2cm] at ( -0.75,+1.6) {};
\node[fill=black, minimum width = 0.5cm, minimum height = 0.2cm] at ( -0.25,+1.6) {};
\node[fill=black, minimum width = 0.5cm, minimum height = 0.2cm] at (  0.25,+1.6) {};
\node[fill=black, minimum width = 0.5cm, minimum height = 0.2cm] at (  0.75,+1.6) {};
\node[fill=black, minimum width = 0.5cm, minimum height = 0.2cm] at (  1.25,+1.6) {};
\node[fill=black, minimum width = 0.5cm, minimum height = 0.2cm] at (  1.75,+1.6) {};

\node[fill=black, minimum width = 0.5cm, minimum height = 0.2cm] at ( -1.75,-1.6) {};
\node[fill=black, minimum width = 0.5cm, minimum height = 0.2cm] at ( -1.25,-1.6) {};
\node[fill=black, minimum width = 0.5cm, minimum height = 0.2cm] at ( -0.75,-1.6) {};
\node[fill=black, minimum width = 0.5cm, minimum height = 0.2cm] at ( -0.25,-1.6) {};
\node[fill=black, minimum width = 0.5cm, minimum height = 0.2cm] at (  0.25,-1.6) {};
\node[fill=black, minimum width = 0.5cm, minimum height = 0.2cm] at (  0.75,-1.6) {};
\node[fill=black, minimum width = 0.5cm, minimum height = 0.2cm] at (  1.25,-1.6) {};
\node[fill=black, minimum width = 0.5cm, minimum height = 0.2cm] at (  1.75,-1.6) {};

\node[fill=black, minimum width = 0.2cm, minimum height = 0.5cm] at (  2.1,-1.25) {};
\node[fill=black, minimum width = 0.2cm, minimum height = 0.5cm] at (  2.1,-0.75) {};
\node[fill=black, minimum width = 0.2cm, minimum height = 0.5cm] at (  2.1,-0.25) {};
\node[fill=black, minimum width = 0.2cm, minimum height = 0.5cm] at (  2.1, 0.25) {};
\node[fill=black, minimum width = 0.2cm, minimum height = 0.5cm] at (  2.1, 0.75) {};
\node[fill=black, minimum width = 0.2cm, minimum height = 0.5cm] at (  2.1, 1.25) {};

\node[fill=black, minimum width = 0.2cm, minimum height = 0.5cm] at (  -2.1,-1.25) {};
\node[fill=black, minimum width = 0.2cm, minimum height = 0.5cm] at (  -2.1,-0.75) {};
\node[fill=black, minimum width = 0.2cm, minimum height = 0.5cm] at (  -2.1,-0.25) {};
\node[fill=black, minimum width = 0.2cm, minimum height = 0.5cm] at (  -2.1, 0.25) {};
\node[fill=black, minimum width = 0.2cm, minimum height = 0.5cm] at (  -2.1, 0.75) {};
\node[fill=black, minimum width = 0.2cm, minimum height = 0.5cm] at (  -2.1, 1.25) {};

%corners
\node[fill=black, minimum width = 0.2cm, minimum height = 0.2cm] at ( -2.1,+1.6) {};
\node[fill=black, minimum width = 0.2cm, minimum height = 0.2cm] at ( 2.1,+1.6) {};
\node[fill=black, minimum width = 0.2cm, minimum height = 0.2cm] at ( -2.1,-1.6) {};
\node[fill=black, minimum width = 0.2cm, minimum height = 0.2cm] at ( 2.1,-1.6) {};

\node[fill=white   , minimum width = 0.5cm, minimum height = 0.5cm] at (-1.75,  1.25) {C};
\node[fill=white   , minimum width = 0.5cm, minimum height = 0.5cm] at (-1.25,  1.25) {C};
\node[fill=black!40, minimum width = 0.5cm, minimum height = 0.5cm] at (-0.75,  1.25) {M};
\node[fill=white   , minimum width = 0.5cm, minimum height = 0.5cm] at (-0.25,  1.25) {C};
\node[fill=white   , minimum width = 0.5cm, minimum height = 0.5cm] at ( 0.25,  1.25) {C};
\node[fill=black!30, minimum width = 0.5cm, minimum height = 0.5cm] at ( 0.75,  1.25) {G};
\node[fill=black!40, minimum width = 0.5cm, minimum height = 0.5cm] at ( 1.25,  1.25) {M};
\node[fill=black!20, minimum width = 0.5cm, minimum height = 0.5cm] at ( 1.75,  1.25) {S};
%initial
\node[circle,minimum size=0.4cm,draw,thick] at (-1.75,  1.25) {};

\node[fill=white   , minimum width = 0.5cm, minimum height = 0.5cm] at (-1.75,  0.75) {C};
\node[fill=black!30, minimum width = 0.5cm, minimum height = 0.5cm] at (-1.25,  0.75) {G};
\node[fill=black,    minimum width = 0.5cm, minimum height = 0.5cm] at (-0.75,  0.75) { };
\node[fill=black   , minimum width = 0.5cm, minimum height = 0.5cm] at (-0.25,  0.75) { };
\node[fill=black!30, minimum width = 0.5cm, minimum height = 0.5cm] at ( 0.25,  0.75) {G};
\node[fill=black!40, minimum width = 0.5cm, minimum height = 0.5cm] at ( 0.75,  0.75) {M};
\node[fill=black!0 , minimum width = 0.5cm, minimum height = 0.5cm] at ( 1.25,  0.75) {C};
\node[fill=black!30, minimum width = 0.5cm, minimum height = 0.5cm] at ( 1.75,  0.75) {G};

\node[fill=white   , minimum width = 0.5cm, minimum height = 0.5cm] at (-1.75,  0.25) {C};
\node[fill=black!40, minimum width = 0.5cm, minimum height = 0.5cm] at (-1.25,  0.25) {M};
\node[fill=black!20, minimum width = 0.5cm, minimum height = 0.5cm] at (-0.75,  0.25) {S};
\node[fill=black!30, minimum width = 0.5cm, minimum height = 0.5cm] at (-0.25,  0.25) {G};
\node[fill=black!40, minimum width = 0.5cm, minimum height = 0.5cm] at ( 0.25,  0.25) {M};
\node[fill=black!0 , minimum width = 0.5cm, minimum height = 0.5cm] at ( 0.75,  0.25) {C};
\node[fill=black!30, minimum width = 0.5cm, minimum height = 0.5cm] at ( 1.25,  0.25) {G};
\node[fill=black   , minimum width = 0.5cm, minimum height = 0.5cm] at ( 1.75,  0.25) { };

\node[fill=black!20, minimum width = 0.5cm, minimum height = 0.5cm] at (-1.75,  -0.25) {S};
\node[fill=white   , minimum width = 0.5cm, minimum height = 0.5cm] at (-1.25,  -0.25) {C};
\node[fill=white   , minimum width = 0.5cm, minimum height = 0.5cm] at (-0.75,  -0.25) {C};
\node[fill=white   , minimum width = 0.5cm, minimum height = 0.5cm] at (-0.25,  -0.25) {C};
\node[fill=white   , minimum width = 0.5cm, minimum height = 0.5cm] at ( 0.25,  -0.25) {C};
\node[fill=black!30, minimum width = 0.5cm, minimum height = 0.5cm] at ( 0.75,  -0.25) {G};
\node[fill=black!20, minimum width = 0.5cm, minimum height = 0.5cm] at ( 1.25,  -0.25) {S};
\node[fill=black!0 , minimum width = 0.5cm, minimum height = 0.5cm] at ( 1.75,  -0.25) {C};

\node[fill=black!40, minimum width = 0.5cm, minimum height = 0.5cm] at (-1.75,  -0.75) {M};
\node[fill=black ,   minimum width = 0.5cm, minimum height = 0.5cm] at (-1.25,  -0.75) { };
\node[fill=black!0 , minimum width = 0.5cm, minimum height = 0.5cm] at (-0.75,  -0.75) {C};
\node[fill=white   , minimum width = 0.5cm, minimum height = 0.5cm] at (-0.25,  -0.75) {C};
\node[fill=black!30, minimum width = 0.5cm, minimum height = 0.5cm] at ( 0.25,  -0.75) {G};
\node[fill=black!20, minimum width = 0.5cm, minimum height = 0.5cm] at ( 0.75,  -0.75) {S};
\node[fill=black!40, minimum width = 0.5cm, minimum height = 0.5cm] at ( 1.25,  -0.75) {M};
\node[fill=black!0 , minimum width = 0.5cm, minimum height = 0.5cm] at ( 1.75,  -0.75) {C};

\node[fill=black!30, minimum width = 0.5cm, minimum height = 0.5cm] at (-1.75,  -1.25) {G};
\node[fill=black!20, minimum width = 0.5cm, minimum height = 0.5cm] at (-1.25,  -1.25) {S};
\node[fill=black!40, minimum width = 0.5cm, minimum height = 0.5cm] at (-0.75,  -1.25) {M};
\node[fill=black!30, minimum width = 0.5cm, minimum height = 0.5cm] at (-0.25,  -1.25) {G};
\node[fill=black   , minimum width = 0.5cm, minimum height = 0.5cm] at ( 0.25,  -1.25) { };
\node[fill=black   , minimum width = 0.5cm, minimum height = 0.5cm] at ( 0.75,  -1.25) { };
\node[fill=black!30, minimum width = 0.5cm, minimum height = 0.5cm] at ( 1.25,  -1.25) {G};
\node[fill=black!40, minimum width = 0.5cm, minimum height = 0.5cm] at ( 1.75,  -1.25) {M};

% goal
\node[circle,minimum size=0.4cm,draw,thick] at (1.75,-1.25) {};
\node[circle,minimum size=0.25cm,draw,thick] at (1.75,-1.25) {};

\draw[step=0.5cm,color=gray] (-2,-1.5) grid (2.0,1.5);

\end{tikzpicture}
\vspace{-0.2cm}
\caption{The second gridworld}
\label{fig:gridworld_2}
\end{center}
 \end{figure}
 \cref{fig:gridworld_2} shows the second gridworld used in our evaluation. As before, the robot starts in the initial location in the top left corner 
 and can only observe the different terrains. The goal location is in the bottom right corner in this example. The true \gls*{MDP} representing this gridworld has $72$ states. We configured learning as for the first gridworld, but collect more samples per round by setting $\nretest = \nbatch = \n{1000}$. \cref{tab:results_gridworld_2} shows the measurement results for learning.

\begin{table}[t]
\centering

\caption{Results for learning the second gridworld example.}
 \label{tab:results_gridworld_2}
\begin{tabular}{r|c|c|c}
 & true model & \LstarMdp & \textsc{IoAlergia} \\\hline
\# outputs     & - & $\n{3663415}$ & $\n{3665746}$ \\\hline
\# traces      & - & $\n{515950}$  & $\n{457927}$ \\\hline
time [s]       & - &  $\n{166.855}$ & $\n{15.136}$ \\\hline
\# states      & $72$ &  $\n{72}$ & $\n{31}$ \\\hline \hline
$\delta_{0.9}$ & - & $\n{0.112141}$ & $\n{0.576325}$ \\\hline \hline
$\mathbb{P}_{\max}(F^{\leq 14} (\mathrm{goal}))$                    & $\n{0.93480795088125}$ & $\n{0.9404068223872026}$  & $\n{0.020750175172571578}$ \\\hline
$\mathbb{P}_{\max}(F^{\leq 12} (\mathrm{goal}))$                    & $\n{0.67119477}$ & $\n{0.6796094609865879}$  & $\n{0.01723651388939612}$ \\\hline
$\mathbb{P}_{\max}(\lnot \mathrm{M}\ U^{\leq18} (\mathrm{goal}))$   & $\n{0.9742903305241056}$ & $\n{0.9750155740662912}$  & $\n{0.01956042468224018}$ \\\hline
$\mathbb{P}_{\max}(\lnot \mathrm{S}\ U^{\leq20} (\mathrm{goal}))$   & $\n{0.14244219329051103}$ & $\n{0.16442424175837944}$  & $\n{0.024040339974465294}$ \\
\end{tabular}
\end{table} 
We sampled $\n{515950}$ traces with a combined number of outputs of $\n{3663415}$, i.e. the 
combined length of all traces is in a similar range as before, although we sampled more traces in a single round. This is the case because
learning stopped already after $\n{500}$ rounds. We used similar model-checking queries as in the previous example and we can again see that the difference between the true model and the \LstarMdp{} model is much smaller than for \textsc{IoAlergia}. However, compared to the previous example, the absolute difference between \LstarMdp{} and the true model with respect to model-checking has slightly increased. 

\end{full}

\begin{full}
\subsection{Shared Coin Consensus}

This example is a randomised consensus protocol by Aspnes and Herlihy~\cite{DBLP:journals/jal/AspnesH90}.
In particular, we used a model of the protocol distributed with the PRISM model checker~\cite{DBLP:conf/cav/KwiatkowskaNP11}
as a basis for our experiments.\footnote{A thorough discussion of the model and related experiments can be found 
at~\url{http://www.prismmodelchecker.org/casestudies/consensus_prism.php}. Accessed: June 28, 2019}
We generally performed only minor adaptions such as adding action labels for inputs, but 
we also slightly changed the functionality by doing that. 
For the purpose of this evaluation these changes are immaterial, though.

We consider only the configuration with the smallest state space of size $272$ with two processes and constant $K$ set to $2$. 
Basically, the \SUL{} has two inputs $\mathrm{go}_1$ and $\mathrm{go}_2$, one for each process, where executing input $\mathrm{go}_i$ causes process $\mathrm{p}_i$ 
to perform exactly one step. The outputs of the \SUL{} comprise the counter state, the processes' coin states,
as well as additional propositions, e.g., denoting that the protocol finished. Note that we need to make 
the coin states visible, to be able to model the \SUL{} with deterministic \glspl*{MDP}. In this experiment, we basically learn
the state machine underlying the protocol, which we cannot observe directly. 

We set the learning parameters to
$\nbatch = \nretest = \ntest = 50$, $\pstop = 0.25$ and $\prand = 0.25$.
We controlled stopping with $t_\mathrm{unamb} = 0.99$, $\rmin =500$ and $\rmax=4000$. 
Finally, we set $p_l=0.125$ for \textsc{IoAlergia}.

\begin{table}[t]
\centering

\caption{Results for learning the shared coin consensus protocol.}
 \label{tab:results_shared_coin}
\begin{tabular}{r|c|c|c}
 & true & \LstarMdp & \textsc{IoAlergia} \\\hline
\# outputs     & - & $\n{537665}$ & $\n{537885}$ \\\hline
\# traces      & - & $\n{98064}$ & $\n{67208}$ \\\hline
 time [s]    & - & $\n{3188.851}$ & $\n{3.548}$ \\\hline
\# states      & $\n{272}$ & $\n{163}$ & $\n{94}$ \\\hline \hline
$\delta_{0.9}$ & - & $\n{0.114241}$ & $\n{0.448197}$ \\\hline \hline
$\mathbb{P}_{\max}(F(\mathrm{finished} \land \mathrm{p_1\_heads} \land \mathrm{p_2\_tails}))$                        
& $\n{0.10694382182657244}$ & $\n{0}$ & $\n{0}$ \\\hline
$\mathbb{P}_{\max}(F(\mathrm{finished} \land \mathrm{p_1\_tails} \land \mathrm{p_2\_tails}))$                        
& $\n{0.5555528623795738}$ & $\n{0.6764506727525488}$ & $\n{0.6594174686344907}$ \\\hline
$\mathbb{P}_{\max}(\mathrm{counter} \neq 5 \ U\ \mathrm{finished})$                        
& $\n{0.3333324384052837}$ & $\n{0.3899122694547719}$ & $\n{0.5356486744317088}$ \\\hline
$\mathbb{P}_{\max}(\mathrm{counter} \neq 4 \ U\ \mathrm{finished})$                        
& $\n{0.42857002816478273}$ & $\n{0.5191215602621915}$ & $\n{0.6682162346767}$ \\\hline
$\mathbb{P}_{\max}(F^{<40}(\mathrm{finished} \land \mathrm{p_1\_heads} \land \mathrm{p_2\_tails}))$                        
& $\n{0.001708984375}$ & $\n{0}$ & $\n{0}$ \\\hline
$\mathbb{P}_{\max}(F^{<40}(\mathrm{finished} \land \mathrm{p_1\_tails} \land \mathrm{p_2\_tails}))$                        
& $\n{0.266845703125}$ & $\n{0.3065539035780639}$ & $\n{0.2694137327781993}$ \\\hline
$\mathbb{P}_{\max}(\mathrm{counter} \neq 5 \ U^{<40}\ \mathrm{finished})$                        
& $\n{0.244384765625}$ & $\n{0.29278361623970706}$ & $\n{0.4459686405804827}$ \\\hline
$\mathbb{P}_{\max}(\mathrm{counter} \neq 4 \ U^{<40}\ \mathrm{finished})$                        
& $\n{0.263427734375}$ & $\n{0.32459228546948615}$ & $\n{0.5050016333012939}$ \\
% $\mathbb{P}_{\max}(\lnot "\mathrm{G}" U^{\leq14} ("\mathrm{goal}"))$   & $\n{529379}$ & $\n{529379}$ & $\n{529379}$ \\\hline
% $\mathbb{P}_{\max}(\lnot "\mathrm{S}" U^{\leq16} ("\mathrm{goal}"))$   & $\n{529379}$ & $\n{529379}$ & $\n{529379}$ \\
\end{tabular}
\end{table}

Table~\ref{tab:results_shared_coin} shows the measurement results for learning a model of the shared coin consensus 
protocol. Compared to the previous example, we need a significantly lower sample size of $\n{98064}$ traces containing $\n{537665}$ outputs, 
although the models 
are much larger. A reason for this is that there is a relatively large number of outputs in this example, such
that states are easier to distinguish from each other. The bisimilarity distance is in a similar range as before for 
\LstarMdp{}, which is again significantly smaller than \textsc{IoAlergia}'s bisimilarity distance. 
The \LstarMdp{} model is again larger than the \textsc{IoAlergia} model, but in this example it is smaller than the true model. This happens
because many states are never reached during learning, as reaching them within a bounded number of steps has a very low probability -- see e.g. the fifth 
model-checking query determining the maximum probability of finishing the protocol within less than $40$ steps, but without consensus, as $\mathrm{p_1}$ chooses $\mathrm{heads}$ and $\mathrm{p_2}$ chooses $\mathrm{tails}$. 
Here, we also see that the model-checking results computed with the \textsc{IoAlergia} model are more accurate in some cases, but
\LstarMdp{}  produces more accurate results overall. The absolute difference from the true values averaged over all model-checking results is about 
about $0.066$ for \LstarMdp{}, approximately half of \textsc{IoAlergia}'s average absolute difference of $0.138$. 
We see an increase in runtime compared to the gridworld examples, which is caused by the larger state space, since the precomputation
time for equivalence testing grows with the state space.

\subsection{Slot machine}
The slot machine originally served as an example in \cite{DBLP:journals/corr/abs-1212-3873,DBLP:journals/ml/MaoCJNLN16}, 
as an adaptation from another model, and  we used it subsequently in \cite{Aichernig2019} as well.
It has three reels, each of them controlled by a separate input. Initially they are blank, but after a reel is spun,
it may either show \emph{apple} or \emph{bar}. A play generally spans $m$ rounds (spins) and after that a prize 
is awarded. It is \emph{Pr10}, if all reels show \emph{bar}, it is \emph{Pr2}, if two reels show \emph{bar}, 
and otherwise it is \emph{Pr0}. The probability of \emph{bar} decreases with decreasing number of remaining rounds. 
Finally, there is also a fourth input \emph{stop}, which with equal probability either stops the game or grants two extra rounds, 
but the remaining rounds cannot exceed $m$.

% For the experiments, we use two different configurations. 
% In the first, 
For our experiments, we configured the slot machine with $m=3$. In this configuration, the true minimal model has $\n{109}$ states. We configured sampling for \textsc{IoAlergia} with $p_l=0.125$ and we set the following parameters for \LstarMdp{}: $\nbatch = \nretest = \ntest = 300$, 
$\pstop = 0.25$, $\prand = 0.25$, $\rmin =500$ and $\rmax=20000$. To demonstrate the influence of the parameter $t_\mathrm{unamb}$, we performed 
experiments with $t_\mathrm{unamb} = 0.9$ and $t_\mathrm{unamb} = 0.99$. 
% state space computation: 3*3*3 (reel configurations) * (m + 1) (rounds + one "pseudo"-round)
% in the pseudo round, a spin just leads to a prize being produced, this is required due to input-enabledness
% +3 (prizes) + 1 (end label) -3 (initial all blank configuration only possible in initial states but not in rounds after that)

\begin{table}[t]
\centering
\caption{Results for learning the slot machine with $t_\mathrm{unamb} = 0.9$.}
 \label{tab:results_slot_machine_0_9}
\begin{tabular}{r|c|c|c}
 & true &  \LstarMdp{} & \textsc{IoAlergia}  \\\hline
\# outputs     & - & $\n{4752687}$ & $\n{4752691}$ \\\hline
\# traces      & - & $\n{1567487}$ & $\n{594086}$\\\hline
 time [s]      & - & $\n{3380.961}$&$\n{60.348}$ \\\hline
\# states      & $\n{109}$ &$\n{109}$ &$\n{86}$ \\\hline \hline
$\delta_{0.9}$ & - &$\n{0.163167}$ &$\n{0.298270}$ \\\hline \hline
$\mathbb{P}_{\max}(F(\mathrm{Pr10}))$                            & $\n{0.36370225843259507}$ & $\n{0.37687890748679354}$ & $\n{0.41693891229269564}$ \\\hline
$\mathbb{P}_{\max}(F(\mathrm{Pr2}))$                             & $\n{0.6441884770185449}$ & $\n{0.6697236608027858}$ & $\n{0.6944979474080994}$\\\hline
$\mathbb{P}_{\max}(F(\mathrm{Pr0}))$                             & $\n{1.0}$ & $\n{1.0}$ & $\n{1.0}$\\\hline
$\mathbb{P}_{\max}(X (X (\textit{bar-bar-blank})))$              & $\n{0.16000000000000006}$ & $\n{0.16146174231990834}$ & $\n{0.1638908294146348}$\\\hline
$\mathbb{P}_{\max}(X(X (X (\textit{apple-bar-bar}))))$           & $\n{0.28622222222222227}$ & $\n{0.28649738373581296}$ & $\n{0.2776027193946863}$\\\hline
$\mathbb{P}_{\max}(\lnot (F^{<10}(\mathrm{end})))$               & $\n{0.25}$ & $\n{0.3012954865683398}$ & $\n{0.32832440181268996}$\\\hline
$\mathbb{P}_{\max}(X (X (X (\textit{apple-apple-apple})))\land (F (\mathrm{Pr0})) )$   
                                                                 & $\n{0.02563900225373582}$ & $\n{0.026179011777341537}$ & $\n{0.010698529301420779}$\\
\end{tabular}
\end{table}

\cref{tab:results_slot_machine_0_9} and \cref{tab:results_slot_machine_0_99} show the results for $t_\mathrm{unamb} = 0.9$ and $t_\mathrm{unamb} = 0.99$, respectively. Configured with $t_\mathrm{unamb} = 0.9$, \LstarMdp{} stopped after $\n{2988}$ rounds and it stopped after $\n{12879}$ rounds, 
if configured with $t_\mathrm{unamb} = 0.99$. We see here that learning an accurate model of the slot machine requires a large amount of samples; 
in the case of $t_\mathrm{unamb} = 0.99$, we sampled $\n{7542332}$ traces containing $\n{24290643}$ outputs. These are almost $10$ times as many 
outputs as for the gridworld examples. However, we also see that sampling more traces clearly pays off. The \LstarMdp{} results shown in \cref{tab:results_slot_machine_0_99} are much better than those shown in \cref{tab:results_slot_machine_0_9}. Notably the state space stayed the same way. 
Thus, the model learned with fewer traces presumably includes some incorrect transitions. This is exactly what our stopping heuristic aims to avoid; 
it aims to avoid ambiguous membership of traces in compatibility classes to reduce the uncertainty in creating transitions. 

We also see in both settings that \LstarMdp{} models are more accurate than \textsc{IoAlergia} models, with respect to bisimilarity distance 
and with respect to model-checking results. While the experiment with $t_\mathrm{unamb} = 0.99$ required the most samples among all experiments, 
it also led to the lowest bisimilarity distance. It is also noteworthy that model-checking results for the \LstarMdp{} model are within a low 
range of approximately $0.01$ of the true results. A drawback of  \LstarMdp{} compared to \textsc{IoAlergia} is again the learning runtime, as 
\LstarMdp{} required about $5$ hours while learning with \textsc{IoAlergia}  took only about $8.7$ minutes. However, in a non-simulated environment, 
the sampling time would be much larger than $5$ hours, such that the learning runtime becomes negligible. Consider for instance a scenario where 
sampling a single traces takes $20$ milliseconds. The sampling time of \LstarMdp{} is about $\n{42}$ hours in that scenario, i.e. about $8.4$ 
times the learning runtime. 

\begin{table}[t]
\centering
\caption{Results for learning the slot machine with $t_\mathrm{unamb} = 0.99$.}
 \label{tab:results_slot_machine_0_99}
\begin{tabular}{r|c|c|c}
 & true &  \LstarMdp{} & \textsc{IoAlergia}  \\\hline
\# outputs     & - & $\n{24290643}$ & $\n{24282985}$ \\\hline
\# traces      & - & $\n{7542332}$ & $\n{3036332}$\\\hline
 time [s]      & - & $\n{18047.961}$&$\n{518.852}$ \\\hline
\# states      & $\n{109}$ &$\n{109}$ &$\n{97}$ \\\hline \hline
$\delta_{0.9}$ & - &$\n{0.048562}$ &$\n{0.251811}$ \\\hline \hline
$\mathbb{P}_{\max}(F(\mathrm{Pr10}))$                            & $\n{0.36370225843259507}$ & $\n{0.37221646080937276}$ & $\n{0.39910173705241514}$ \\\hline
$\mathbb{P}_{\max}(F(\mathrm{Pr2}))$                             & $\n{0.6441884770185449}$ & $\n{0.6552158270947213}$ & $\n{0.699695823269987}$\\\hline
$\mathbb{P}_{\max}(F(\mathrm{Pr0}))$                             & $\n{1.0}$ & $\n{1.0}$ & $\n{1.0}$\\\hline
$\mathbb{P}_{\max}(X (X (\textit{bar-bar-blank})))$              & $\n{0.16000000000000006}$ & $\n{0.16073824552550947}$ & $\n{0.15972326474987406}$\\\hline
$\mathbb{P}_{\max}(X(X (X (\textit{apple-bar-bar}))))$           & $\n{0.28622222222222227}$ & $\n{0.2865792773645907}$ & $\n{0.28509746711603917}$\\\hline
$\mathbb{P}_{\max}(\lnot (F^{<10}(\mathrm{end})))$               & $\n{0.25}$ & $\n{0.26064997446353955}$ & $\n{0.40000608796148596}$\\\hline
$\mathbb{P}_{\max}(X (X (X (\textit{apple-apple-apple})))\land (F (\mathrm{Pr0})) )$   
                                                                 & $\n{0.02563900225373582}$ & $\n{0.026350661212454343}$ & $\n{0.012760875333368903}$\\
\end{tabular}
\end{table}

\subsection{Discussion \& Threats to Validity}
Our case studies demonstrated that \LstarMdp{} is able to achieve better accuracy than \textsc{IoAlergia}. The bisimilarity distances of \LstarMdp{} models to the true models were generally lower and the model checking results were more accurate. These observations will be investigated in further case studies. It should be noted though that the considered systems have different characteristics. The gridworld has small state-space, but is strongly connected and the different terrains lead to different probabilistic decisions, e.g. if we try to enter \emph{mud} there is a probability of $0.4$ of entering one of the neighbouring tiles, whereas entering \emph{concrete} is generally successful (the probability of entering other tiles instead is $0$). The consensus protocol has a large state space with many different outputs and finishing the protocol takes at least $14$ steps. The slot machines requires states to be distinguished based on subtle differences in probabilities, as the probability of seeing \emph{bar} decreases in each round. 

\LstarMdp{} has several parameters that affect performance and accuracy. We plan to investigate the influence of parameters in further experiments. For the present experiments, we fixed most of the parameters except for $\nretest$, $\ntest$ and $\nbatch$ and we observed that results are robust with respect to these parameters. We, e.g., increased $\nbatch$ from $\n{300}$ for the first gridworld to $\n{1000}$ for the second gridworld. Both settings led to approximately the same results, as learning simply performed fewer rounds with $\nbatch = \n{1000}$. Hence, further experiments will examine if the fixed parameters are indeed appropriately chosen and if guidelines for choosing other parameters can be provided. 
% Note that the $\cq$ query with the $\ncomplete$ parameter follows a similar motivation as the applied statistical tests based on Hoeffding bounds. These tests check for difference and assume compatibility otherwise. In case of low information, i.e. if $\cq$ returns false, we assume compatibility. 

\LstarMdp{} and \textsc{IoAlergia} learn from different traces, thus the trace selection may actually be the main reason for the better accuracy of \LstarMdp{}. We examined if this is the case, by learning \textsc{IoAlergia} models from two types of traces: traces with uniform input selection and traces sampled during learning with \LstarMdp{}. We noticed that models learned from \LstarMdp{} traces altogether led to less accurate results, especially in terms of bisimilarity distance, and therefore we reported only results for models learned from traces with uniformly distributed inputs.  

\end{full}

\begin{conference}
Due to space constraints, we only present the intuitive gridworld experiment. 
The full technical report includes further experiments with a larger
gridworld (72 states), a consensus protocol (272 states) and a slot machine model (109 states)~\cite{lstar_mdp_tech_report}.
They also confirm the favourable accuracy of \LstarMdp{}. 
\vspace{0.6cm}
\end{conference}

\section{Related Work}
\label{sec:rel_work}

In the following, we discuss techniques for learning both model structure and transition probabilities in case of probabilistic systems. There are many learning approaches for models with a given structure, e.g., for learning control strategies~\cite{DBLP:conf/rss/FuT14}. Covering these approaches is beyond the scope of this paper. 

We build upon Angluin's $L^*$~\cite{DBLP:journals/iandc/Angluin87}, thus our work shares 
similarities with other $L^*$-based work like active learning of Mealy machines~\cite{DBLP:conf/fm/ShahbazG09}. 
Interpreting \glspl*{MDP} as functions from test sequences to output distributions is
similar to the interpretation of Mealy machines as functions from input sequences to outputs~\cite{DBLP:conf/sfm/SteffenHM11}.

Volpato and Tretmans presented an $L^*$-based technique for non-deterministic input-output transition systems~\cite{DBLP:journals/eceasst/VolpatoT15}.
They simultaneously learn an over- and an under-approximation of the \SUL{} with respect to the \gls*{ioco} relation~\cite{DBLP:journals/stp/Tretmans96}. %a popular conformance relation in testing. 
\begin{full}
Inspired by that, \LstarMdp{} uses completeness queries and we add transitions to a chaos state in case
we have low information. 
\end{full}
\begin{conference}
Inspired by that, we apply completeness queries and we add transitions to a chaos state in case of incomplete information. 
\end{conference}
Beyond that, we consider systems to behave stochastically rather than non-deterministically.
\begin{full}
While \cite{DBLP:journals/eceasst/VolpatoT15} leaves the concrete implementation of queries unspecified, \LstarMdp{}'s implementation closely follows \cref{sec:method_non_exact}.
\end{full}
Early work on \gls*{ioco}-based learning for non-deterministic systems has been presented by Willemse~\cite{DBLP:conf/fmics/Willemse06}.
Khalili and Tacchella~\cite{DBLP:conf/icgi/KhaliliT14} addressed non-determinism by presenting an $L^*$-based algorithm for non-deterministic Mealy machines. 
\begin{full}
Like Volpato and Tretmans~\cite{DBLP:journals/eceasst/VolpatoT15}, they assume to be able to observe all possible outputs in response to input sequences applied
during learning. 
Our implementation does not require this assumption by checking for compatibility, i.e. approximate equivalence, between output distributions. Both these approaches assume a testing context, as we do. 
\end{full}

\begin{full}
Most sampling-based learning algorithms for stochastic systems are passive, i.e. they assume preexisting samples of system traces.
Their roots can be found in grammar inference techniques like \textsc{Alergia}~\cite{Carrasco_Oncina_1994} and \texttt{rlips}~\cite{DBLP:journals/ita/CarrascoO99}, which identify stochastic regular languages. We share with these techniques that we also apply Hoeffing bounds~\cite{10.2307/2282952} for testing for difference between probability distributions. 
\textsc{Alergia} has been extended to \glspl*{MDP} by Mao et al.~\cite{DBLP:journals/corr/abs-1212-3873,DBLP:journals/ml/MaoCJNLN16}. The extension is called \textsc{IoAlergia} and basically creates a tree-based representation of the sampled system traces and repeatedly merges compatible nodes to create an automaton. Finally, transition probabilities
are estimated from observed output frequencies. Like \LstarMdp{}, \textsc{IoAlergia} converges in 
the limit, but showed worse accuracy in \cref{sec:eval}. It was adapted to an active setting by Chen 
and Nielsen~\cite{DBLP:conf/icmla/ChenN12}. They proposed to generate new samples to reduce 
uncertainty in the data. In contrast to this, we base our sampling not only on the data collected so 
far (refine queries), but also on the current observation table and the derived hypothesis 
\glspl*{MDP} (refine \& equivalence queries), i.e. we take information about the \SUL{}'s structure 
into account. In previous work, we presented a different approach to apply \textsc{IoAlergia} in an 
active setting which takes reachability objectives into account with the aim of maximising the 
probability of reaching desired events~\cite{Aichernig2019}. 
\end{full}
\begin{conference}
 Most sampling-based learning algorithms for stochastic systems are passive.
Notable early works are \textsc{Alergia}~\cite{Carrasco_Oncina_1994} and \texttt{rlips}~\cite{DBLP:journals/ita/CarrascoO99}, which identify stochastic regular languages. Both also apply Hoeffing bounds for testing for difference between probability distributions. 
We compare \LstarMdp{} to \textsc{IoAlergia}, an extension of \textsc{Alergia} by Mao et al.~\cite{DBLP:journals/corr/abs-1212-3873,DBLP:journals/ml/MaoCJNLN16}. It basically creates a tree-based representation of given system traces and repeatedly merges compatible nodes, creating an automaton. Normalised observed output frequencies estimate transition probabilities. \textsc{IoAlergia} also converges in the limit. 
Chen and Nielsen applied it in an active setting~\cite{DBLP:conf/icmla/ChenN12}, by sampling new 
traces to reduce uncertainty in the data. In contrast to this, we base our sampling not only on data 
collected so far (refine queries), but also on observation tables and derived hypothesis 
\glspl*{MDP} (refine \& equivalence queries), taking information about the \SUL{}'s structure into 
account. In previous work, we presented a different approach to active learning via 
\textsc{IoAlergia} which takes reachability objectives into account with the aim at maximising the 
probability of reaching desired events~\cite{Aichernig2019}. 
\end{conference}

\begin{full}
$L^*$-based learning for probabilistic systems has also been presented by Feng et al.~\cite{DBLP:conf/atva/FengHKP11}. They learn assumptions in the form of probabilistic finite automata for compositional verification of probabilistic systems. Their learning algorithm requires queries returning exact probabilities, hence it is not directly applicable in a sampling-based setting. The learning algorithm shares similarities with an $L^*$-based algorithm for learning multiplicity automata~\cite{DBLP:journals/siamcomp/BergadanoV96}, a generalisation of deterministic automata. Further query-based learning in a probabilistic setting has been described by Tzeng~\cite{DBLP:journals/ml/Tzeng92}. He presented a query-based algorithm for learning probabilistic automata and described an adaptation of Angluin's $L^*$ for learning Markov chains. In contrast to our exact learning algorithm \LstarMdpE{}, which relies on output distribution queries, Tzeng's algorithm for Markov chains queries the generating probabilities of strings.
Castro and Gavald{\`a} review passive learning techniques for probabilistic automata with a focus on convergence guarantees and present them in a query framework~\cite{Castro2016}. Unlike \glspl*{MDP}, the learned automata cannot be controlled by inputs. 
\end{full}

\begin{conference}
 Feng et al.~\cite{DBLP:conf/atva/FengHKP11} learn assumptions for compositional verification in the form of probabilistic finite automata with an $L^*$-style method. Their method requires queries returning exact probabilities, hence it is not applicable in a sampling-based setting. It shares similarities with an $L^*$-based algorithm for learning multiplicity automata~\cite{DBLP:journals/siamcomp/BergadanoV96}, a generalisation of deterministic automata. Further query-based learning in a probabilistic setting has been described by Tzeng~\cite{DBLP:journals/ml/Tzeng92}. He presented a query-based algorithm for learning probabilistic automata and an adaptation of Angluin's $L^*$ for learning Markov chains. Castro and Gavald{\`a} review passive learning techniques for probabilistic automata with a focus on convergence guarantees and present them in a query framework~\cite{Castro2016}. Unlike \glspl*{MDP}, the learned automata cannot be controlled by inputs. 
 
 %In contrast to our exact algorithm \LstarMdpE{}, his algorithm for Markov chains queries the generating probabilities of strings rather than output distributions.
\end{conference}

\section{Conclusion}
\label{sec:concl}

We presented $L^*$-based learning of \glspl*{MDP}. For our exact learning algorithm \LstarMdpE{}, we assumed an ideal setting  that allows to query information about the \SUL{}  with exact precision. 
%\LstarMdpE{} resembles Angluin's original $L^*$ algorithm~\cite{DBLP:journals/iandc/Angluin87} and its adaptation for Mealy machines~\cite{DBLP:conf/fm/ShahbazG09,DBLP:conf/sfm/SteffenHM11}. 
Subsequently, we relaxed our assumptions, by approximating exact queries through sampling \SUL{} traces via directed testing. 
These traces serve to infer the structure of hypothesis \glspl*{MDP}, to estimate transition probabilities
and to check for equivalence between \SUL{} and learned hypotheses. 
The resulting sampling-based \LstarMdp{} iteratively learns approximate \glspl*{MDP} which converge to the correct \gls*{MDP} in the large sample limit. We implemented \LstarMdp{} and compared it to \textsc{IoAlergia}~\cite{DBLP:journals/ml/MaoCJNLN16}, a state-of-the-art passive learning algorithm for \glspl*{MDP}. The evaluation showed that \LstarMdp{} is able to produce more accurate models. To the best of our knowledge, \LstarMdp{} is the first $L^*$-based algorithm for \glspl*{MDP} that can be implemented via testing. \begin{conference}                                                                                                                                                                                                                                                        Further details regarding the implementation, convergence proofs and extended experiments can be found in the technical report~\cite{lstar_mdp_tech_report} and the evaluation material~\cite{eval_material}.                                                                                                                                                                                                                                                    \end{conference}
\begin{full}
 Experimental results and the implementation can be found in the evaluation material~\cite{eval_material}.
\end{full}

The evaluation showed promising results, therefore we believe that our technique can greatly aid the black-box analysis of reactive systems such as communication protocols. While deterministic active automata learning has successfully been applied in this area~\cite{DBLP:conf/cav/Fiterau-Brostean16,DBLP:conf/icst/TapplerAB17}, networked environments are prone to be affected by uncertain behaviour that can be captured by \glspl*{MDP}. \LstarMdp{} converges in the limit, therefore a potential direction for future work is an analysis with respect to \gls*{PAC} learnability~\cite{DBLP:journals/cacm/Valiant84,Castro2016}. 
A challenge towards this goal will be the identification of a distance measure suited to verification~\cite{DBLP:journals/ml/MaoCJNLN16}.
Furthermore, \LstarMdp{} provides room for experimentation, e.g. different testing techniques could be applied in equivalence queries. 

% \begin{appendix}
% \end{appendix}
\subsubsection*{Acknowledgment.}
The work of B.\,Aichernig, M. Eichlseder and M.\,Tappler has been carried out as part of the TU Graz LEAD project ``Dependable Internet of Things in Adverse Environments''.
The work of K.\,Larsen and G.\,Bacci has been supported by the Advanced ERC Grant nr. 867096 (LASSO).

% \renewcommand{\doi}[1]{} % do not print url in references because doi command prints them anyway
\bibliographystyle{splncs04}
\bibliography{references}

\end{document}